\documentclass[Afour,sageh,times]{sagej}

\usepackage{moreverb,url}

\usepackage[colorlinks,bookmarksopen,bookmarksnumbered,citecolor=blue,urlcolor=red]{hyperref}

\usepackage{times}
\usepackage{soul}
\usepackage{url}
\usepackage[utf8]{inputenc}
\usepackage{graphicx}
\usepackage{amsmath,amssymb,amsfonts}
\usepackage{amsthm}
\usepackage{booktabs}
\usepackage{algorithm}
\usepackage{algorithmic}
\usepackage[switch]{lineno}
\usepackage{verbatim} 
\usepackage{enumerate}
\usepackage{rotating}
\usepackage{enumitem}%\setlist{nolistsep}
\usepackage{color}
\usepackage{mathrsfs}
\usepackage{threeparttable}
\usepackage{mathtools}

\usepackage{subcaption}
\usepackage{textcomp}
\usepackage[table]{xcolor}
\usepackage{siunitx} \usepackage{flushend}
\usepackage{tikz}
\usepackage{svg}
\usepackage{multirow}
\usetikzlibrary{positioning, arrows.meta,automata,trees}
\usetikzlibrary{calc}

\definecolor{algLine}{RGB}{190,190,190}   % subtle separator lines
\definecolor{algNote}{RGB}{55,55,55}      % comment text (almost black)
\definecolor{algHL}{RGB}{30,80,160}       % highlight blue (nice-looking)
\newcommand{\algsep}{\STATE \textcolor{algLine}{\rule{\linewidth}{0.4pt}}}
\newcommand{\algcomment}[1]{\STATE \textbf{\textcolor{algHL}{// #1}}}

\newcommand\BibTeX{{\rmfamily B\kern-.05em \textsc{i\kern-.025em b}\kern-.08em
T\kern-.1667em\lower.7ex\hbox{E}\kern-.125emX}}

\def \RC{\mathcal{R}}
\def \IC{\mathcal{I}}
\def \PC{\mathcal{P}}
\def \PP{\mathbb{P}}
\def \TT{\mathbb{T}}

\def \s{\mathbf{s}}
\def \x{\mathbf{x}}
\def \z{\mathbf{z}}

\def \a{\mathbf{a}}
\def \R{\mathbb{R}}
\def \G{\mathrm{G}}
\def \F{\mathrm{F}}
\def \U{\mathrm{U}}

\newcommand{\bs}{\boldsymbol}
\newcommand{\mr}{\mathrm}

\newtheorem{example}{Example}
\newtheorem{theorem}{Theorem}
\newtheorem{remark}{Remark}

\newtheorem{problem}{Problem}
\newtheorem{lemma}{Lemma}

\makeatletter
\def\ps@title{%
\let\@oddhead\@empty
\let\@evenhead\@empty
\def\@oddfoot{\hfill\thepage\hfill}%
\let\@evenfoot\@oddfoot}

\def\@maketitle{%
\vspace*{-34pt}%
\null%
\begin{center}
\if@PCfour
\begin{rm}
\else
\begin{sf}
\fi
\begin{minipage}[t]{\textwidth}
  \vskip 12.5pt%
  {\raggedright\titlesize\textbf{\@title}\par}%
  \vskip 1.5em%
  \vskip 12.5mm%
\end{minipage}
{\par\large%
  \lineskip .5em%
  {\raggedright\textbf{\@author}\par}}
\vskip 40pt%
{\noindent\usebox\absbox\par}
{\vspace{20pt}%
  {\noindent\normalsize\@keywords}\par}
\if@PCfour
\end{rm}
\else
\end{sf}
\fi
\end{center}
\vspace{22pt}
\par}
\makeatother

\begin{document}

% \runninghead{Liu and Hou \textit{et~al.}}

\title{DAG-STL: A Hierarchical Framework for Zero-Shot Trajectory Planning under Signal Temporal Logic Specifications}

\author{Ruijia Liu\affilnum{1},  Ancheng Hou\affilnum{1},  Xiao Yu\affilnum{2},    and Xiang Yin\affilnum{1}}

\affiliation{\affilnum{1}School of Automation and Intelligent Sensing, Shanghai Jiao Tong University, Shanghai, China.\\
\affilnum{2}Institute of Artificial Intelligence, Xiamen University, Xiamen, China.}

\corrauth{Xiang Yin, School of Automation and Intelligent Sensing, Shanghai Jiao Tong University, 800 Dongchuan RD. Minhang District, Shanghai 200240, China.}

\email{yinxiang@sjtu.edu.cn}

\begin{abstract}
Signal Temporal Logic (STL) is a powerful language for specifying temporally structured robotic tasks. Planning executable trajectories under STL constraints, however, remains difficult when the system dynamics and environment structure are not analytically available. Existing methods typically either assume explicit system models or learn task-specific behaviors, which limits their ability to generalize zero-shot to previously unseen STL tasks. In this work, we study offline STL planning under unknown dynamics using only task-agnostic trajectory data.
Our central design philosophy is to separate logical reasoning from trajectory realization. We instantiate this idea in DAG-STL, a hierarchical framework that converts long-horizon STL planning into three stages. It first decomposes an STL formula into reachability and invariance progress conditions linked by shared timing constraints. It then allocates timed waypoints using learned reachability-time estimates. Finally, it synthesizes trajectories between these waypoints with a diffusion-based generator. This decomposition--allocation--generation pipeline reduces a difficult global planning problem to a set of shorter and better-supported subproblems.
To bridge the remaining gap between planning-level correctness and execution-level feasibility, we further introduce a rollout-free dynamic consistency metric, an anytime refinement search   procedure for improving multiple allocation hypotheses under finite budgets, and a hierarchical online replanning mechanism for execution-time recovery. Experiments in Maze2D, OGBench AntMaze and Cube domain show that DAG-STL substantially outperforms direct robustness-guided diffusion on complex long-horizon STL tasks and generalizes across both navigation and manipulation settings. In a custom-built environment with an optimization-based reference, DAG-STL recovers most model-solvable tasks while retaining a clear computational advantage  over direct optimization based on the explicit system model.
\end{abstract}


\maketitle
\pagestyle{plain}

\section{Introduction}
 
Task and motion planning (TAMP) is a fundamental problem in robotics, as it enables autonomous systems to accomplish high-level objectives while reasoning about low-level feasibility constraints~\cite{lavalle2006planning}. With the growing demand for advanced robotic autonomy, there has been increasing interest in planning for complex missions specified using formal logics~\cite{kress2018synthesis,luckcuck2019formal,yin2024formal}. Among many formalisms, Signal Temporal Logic (STL) has emerged as a powerful specification language for describing temporal behaviors of real-valued and real-time signals. Owing to its expressive syntax and the availability of both Boolean and quantitative semantics, STL has been widely adopted for defining high-level robotic tasks and has been applied successfully in diverse domains such as autonomous drones~\cite{silano2021power,yang2025stlgame}, robotic manipulation~\cite{puranic2021learning,sewlia2022cooperative,yuasa2026neuro}, and bipedal locomotion~\cite{gu2025robust}.

Controlling robots under STL task constraints is highly challenging, since one must simultaneously ensure task satisfaction and compatibility with the underlying system dynamics. In principle, when the environment and system dynamics are fully known, STL planning can be formulated as a hybrid optimization problem over both continuous and discrete variables, where system dynamics and STL constraints are encoded explicitly~\cite{raman2014model,belta2019formal,kurtz2020trajectory}. Although such optimization-based approaches provide a principled framework for STL planning, they often incur substantial computational burden and are generally intractable because of the combinatorial complexity of the resulting hybrid problem. Beyond computational difficulty, a more fundamental limitation is their reliance on accurate knowledge of the system model, including both robot dynamics and environment interactions.

These limitations motivate a data-driven alternative, especially in settings where accurate dynamics models are unavailable but trajectory data can be collected from simulations or prior operations. In many practical systems, the true dynamics are unknown or difficult to model explicitly.
For instance, in many commercially packaged robotic systems, low-level parameters and dynamics are proprietary and thus inaccessible to end users, leaving only limited interfaces available for control and planning \cite{salunkhe2025kinematic}. Likewise, for platforms such as soft robots, constructing faithful analytical models remains notoriously difficult~\cite{armanini2023soft}. More broadly, for systems with high-dimensional, strongly nonlinear, or contact-rich dynamics \cite{suh2025dexterous}, obtaining an accurate analytical model is often prohibitively challenging due to the complexity of the underlying physical processes. In such cases, collecting trajectory data is often substantially more practical and feasible than deriving an explicit mathematical model from first principles.

For STL planning without explicit dynamics models, a major line of work has focused on reinforcement learning (RL); see, e.g.,~\cite{aksaray2016q,balakrishnan2019structured,kalagarla2021model,venkataraman2020tractable,ikemoto2022deep,wang2024synthesis}. However, RL methods are typically designed and trained for a specific task, require carefully engineered task-dependent rewards, and often demand a large number of training episodes. As a result, they generally exhibit limited generalization across different STL tasks. More recently, generative methods such as diffusion models have been explored for trajectory planning with improved test-time flexibility and generalization capability; see, e.g.,~\cite{janner2022planning,ajay2022conditional,chi2023diffusion}. Nevertheless, their current use is still largely concentrated on relatively simple point-to-point or goal-conditioned planning problems, and they do not readily handle long-horizon temporal constraints of the kind induced by STL. Achieving zero-shot STL planning from task-agnostic trajectory data alone, especially for complex long-horizon specifications, therefore remains highly challenging.

In this paper, we address the problem of generating feasible trajectories for complex, long-horizon STL tasks when the underlying system dynamics are unknown. We assume access only to task-agnostic trajectory data collected from prior operations, without any task-specific labels. Our goal is to leverage this data to construct a planning framework that can satisfy previously unseen STL tasks while remaining compatible with the unknown system dynamics, thereby enabling zero-shot task generalization.

The main contributions of this paper are summarized as follows.
\begin{itemize}[leftmargin=*]
    \item  
    We develop a hierarchical framework, DAG-STL, that integrates semantics-based task \textbf{\underline{D}}ecomposition, system-aware progress \textbf{\underline{A}}llocation, and diffusion-based trajectory \textbf{\underline{G}}eneration for offline STL planning under unknown dynamics.
    Specifically, STL tasks are decomposed into reachability and invariance progress conditions coupled through shared timing constraints; a search-based allocation module then constructs a timed waypoint skeleton; and a diffusion-based completion module generates trajectories that realize the allocated waypoints while enforcing the induced invariance conditions. This decomposition--allocation--generation architecture provides a sound symbolic-level planning framework while remaining fully compatible with task-agnostic offline data. 
    \item 
    We further augment the basic planner with a rollout-free \emph{Dynamic Consistency Metric}, an \emph{Anytime Refinement Search} (ARS) procedure, and a hierarchical \emph{Online Replanning} mechanism. The consistency metric evaluates whether planned states and transitions are locally supported by the offline data manifold; ARS upgrades single-solution allocation into a multi-hypothesis, score-guided search with cost-guided backjumping; and online replanning supports both local segment repair and history-consistent global recovery during execution. Collectively, these components extend the framework beyond basic STL-correct open-loop planning toward a more execution-aware and robust planning architecture.
    \item 
    We validate the proposed framework through case studies and systematic experiments on Maze2D, OGBench AntMaze, and the Cube domain, together with controlled comparisons and supplementary stress tests in a custom-built environment. The results demonstrate that DAG-STL significantly outperforms a direct robustness-guided diffusion baseline on complex long-horizon STL tasks, that ARS and online replanning yield clear gains in execution reliability on challenging tasks, and that the overall framework generalizes across both navigation and manipulation domains while maintaining favorable computational scalability.
\end{itemize}

\section{Related Works}
We next  briefly review the relevant lines of research related to our work, including model-based STL planning, reinforcement learning for STL tasks, diffusion-based planning, and STL decomposition methods.
\subsection{Model-Based STL Planning}
For systems with known dynamics and explicit environment constraints, a standard approach is to encode STL satisfaction using binary variables and formulate trajectory planning as an optimization problem constrained by both system dynamics and task semantics. This line of work was initiated in~\cite{raman2014model} within a model predictive control framework, which enables receding-horizon adaptation to environmental disturbances. However, the number of binary variables typically grows rapidly with the planning horizon, leading to substantial computational burden. To improve scalability, \cite{sun2022multi} introduced timed waypoints in a multi-agent setting, while \cite{kurtz2022mixed,cardona2025stl} proposed more compact encodings with fewer binary variables.

Another major challenge for optimization-based STL planning is the non-smoothness of the resulting objective. To address this, several works~\cite{pant2017smooth,gilpin2020smooth,dawson2022robust,kapoor2025stlcg++} proposed smooth approximations of STL robustness, enabling faster gradient-based optimization. Sample-based alternatives have also been explored~\cite{ho2023sampling,ilyes2023stochastic} to avoid solving mixed-integer programs directly. Despite these advances, such methods still fundamentally rely on accurate models of system dynamics and environment interactions, which are often unavailable in the settings considered in this paper.

\subsection{Reinforcement Learning for STL Tasks}
To address unknown dynamics, a large body of work has studied reinforcement learning (RL) for STL tasks~\cite{aksaray2016q,balakrishnan2019structured,kalagarla2021model,venkataraman2020tractable,ikemoto2022deep,wang2024synthesis,wang2025multi}. A common strategy is to design reward functions that approximate STL satisfaction. For example, \cite{aksaray2016q} introduced one of the first model-free approaches by augmenting the state with a fixed-length trajectory window induced by the STL formula, thereby enabling Q-learning with instance-based rewards. Later works~\cite{venkataraman2020tractable,wang2024synthesis} proposed more scalable progress encodings that avoid storing the entire trajectory, while \cite{balakrishnan2019structured} explored locally shaped rewards over partial signal traces. Although these methods improve tractability, they are still typically task-specific and often struggle with long-horizon STL specifications.

More recent work has attempted to improve generalization within the RL paradigm. One direction is model-based RL, where a learned dynamics model is combined with planning or optimization, as in~\cite{kapoor2020model}. However, the reliance on optimization still limits such approaches to relatively short-horizon tasks. Another direction is skill-based or goal-conditioned RL, where reusable goal-conditioned policies are composed to satisfy temporal tasks; for example, \cite{he2024scalable} trains a library of skills and sequences them through search. While such methods provide some task reuse, they still depend on predefined goal abstractions and do not naturally support zero-shot generalization to arbitrary unseen STL specifications. In contrast, our setting assumes only task-agnostic trajectory data and seeks zero-shot planning for previously unseen STL tasks without task-specific training.

\subsection{Planning with Diffusion Models}
More recently, generative models, especially diffusion models~\cite{ho2020denoising}, have emerged as a promising approach to trajectory generation under unknown dynamics~\cite{janner2022planning,ajay2022conditional,chi2023diffusion,carvalho2023motion,huang2025diffusion}. Compared with many model-based RL approaches, diffusion models offer stronger long-horizon generation capability and greater test-time flexibility~\cite{janner2022planning}, which has made them attractive for complex planning problems. Beyond standard goal-conditioned planning, diffusion models have also been explored for temporally specified tasks. For example, for finite Linear Temporal Logic ($\text{LTL}_f$) tasks, \cite{feng2024ltldog} introduced classifier-guided diffusion sampling, while \cite{feng2025diffusion} proposed a hierarchical framework that decomposes co-safe LTL tasks into sub-tasks and combines diffusion generation with hierarchical control.

In the STL setting, diffusion-based planning has also begun to attract attention. For example, \cite{zhong2023guided} proposed a robustness-guided diffusion approach that uses STL robustness gradients to steer trajectory sampling, thereby generating vehicle trajectories that satisfy traffic-rule specifications. Building on this, \cite{meng2024diverse} introduced data augmentation to improve diversity and rule satisfaction rates. However, these methods  operate largely by directly optimizing STL robustness in trajectory space, which becomes increasingly difficult for complex long-horizon specifications. Moreover, directly maximizing robustness often creates a tension between task satisfaction and trajectory feasibility under the data distribution~\cite{li2024derivative}. In contrast, our method introduces an explicit intermediate planning structure, namely semantics-based decomposition and progress allocation, before invoking diffusion-based trajectory completion. This separation between logic-level planning and data-driven trajectory synthesis enables our framework to handle more complex STL tasks while remaining computationally efficient. In the experiments, we compare DAG-STL against~\cite{zhong2023guided} and show clear advantages on complex long-horizon tasks.

\subsection{STL Decomposition Techniques}
To alleviate the complexity induced by nested and long-horizon STL specifications, several decomposition-based synthesis methods have been proposed in the literature; see, e.g.,~\cite{charitidou2021signal,leahy2023rewrite,yu2023model,kapoor2024safe,zhang2025decomposition}. Among these, the work most closely related to ours is~\cite{kapoor2024safe}, which decomposes STL tasks into spatio-temporal subtasks with time-variable constraints, simplifies these constraints, partitions the task into intervals, and then solves the resulting atomic subtasks sequentially.

Our approach is inspired by this decomposition perspective but differs in two important aspects. First, we decompose STL tasks into reachability and invariance progress conditions coupled through shared time-variable constraints, and resolve them through a search-based allocation procedure with dynamic on-the-fly constraint maintenance. This allows the planner to backtrack and revise earlier allocation decisions while preserving a globally consistent timing context. Second, our decomposition is designed to interface naturally with learned transition-time prediction, diffusion-based trajectory completion, and the refinement mechanisms introduced later in the paper. As a result, unlike prior decomposition-based methods that primarily target model-based synthesis, our framework is tailored to the offline unknown-dynamics setting considered here.

\subsection{Comparison with Preliminary Version}
A preliminary version of this work was presented in \cite{liu2025zeroshot}, where the basic decomposition--allocation--generation backbone was introduced. Compared with that earlier version, the present paper makes three substantive extensions. First, while retaining the original DAG backbone, we refine several key components, including the transition-time prediction and constrained trajectory generation modules. More importantly, we identify and explicitly address a major practical limitation of the original framework: even when a nominal plan is STL-correct, it may still lie in a region only weakly supported by the offline data manifold, making it difficult to track reliably during execution. Second,   to further close this gap at planning time, we introduce a rollout-free Dynamic Consistency Metric together with an Anytime Refinement Search (ARS) framework. Rather than serving merely as a post-hoc trajectory smoothing procedure, ARS strictly generalizes the original allocator by upgrading single-hypothesis progress allocation into a multi-hypothesis, score-guided search procedure with cost-guided backjumping. Third, to handle execution-time deviations that cannot be fully resolved offline, we develop a hierarchical online replanning mechanism that combines local segment repair with history-consistent global re-allocation, thereby extending the original open-loop planner into a more execution-aware and robust framework.
\section{Preliminaries}
\subsection{System Model}
We consider a discrete-time system with unknown dynamics
\begin{equation}\label{formula:system}
\x_{t+1}=f(\x_t,\a_t),
\end{equation}
where $\x_t \in \R^n$ and $\a_t\in \R^m$ denote the state and the action applied at time step $t$, respectively.
Given an initial state $\x_0$ and an action sequence
$\a_0\a_1\dots\a_{T-1}$,
the resulting \emph{trajectory} of the system is
$\bs{\tau}=\x_{0} \a_{0} \x_{1} \a_{1}\dots\a_{T-1}\x_{T}$,
where $T$ is the planning horizon.
The corresponding \emph{signal} is the state sequence
$\s=\x_0\x_1\dots\x_T$,
and we denote by $\s_t=\x_t\x_{t+1}\dots\x_T$ the sub-signal starting from time step $t$.
In this work, STL specifications are imposed on the state signal $\s$, and the planner primarily reasons over state sequences rather than action sequences. The planned state sequence is therefore interpreted as a reference trajectory, while action synthesis and trajectory tracking are handled at the execution stage.

\begin{figure*}[t]
    \centering
    \includegraphics[width=0.98\textwidth]{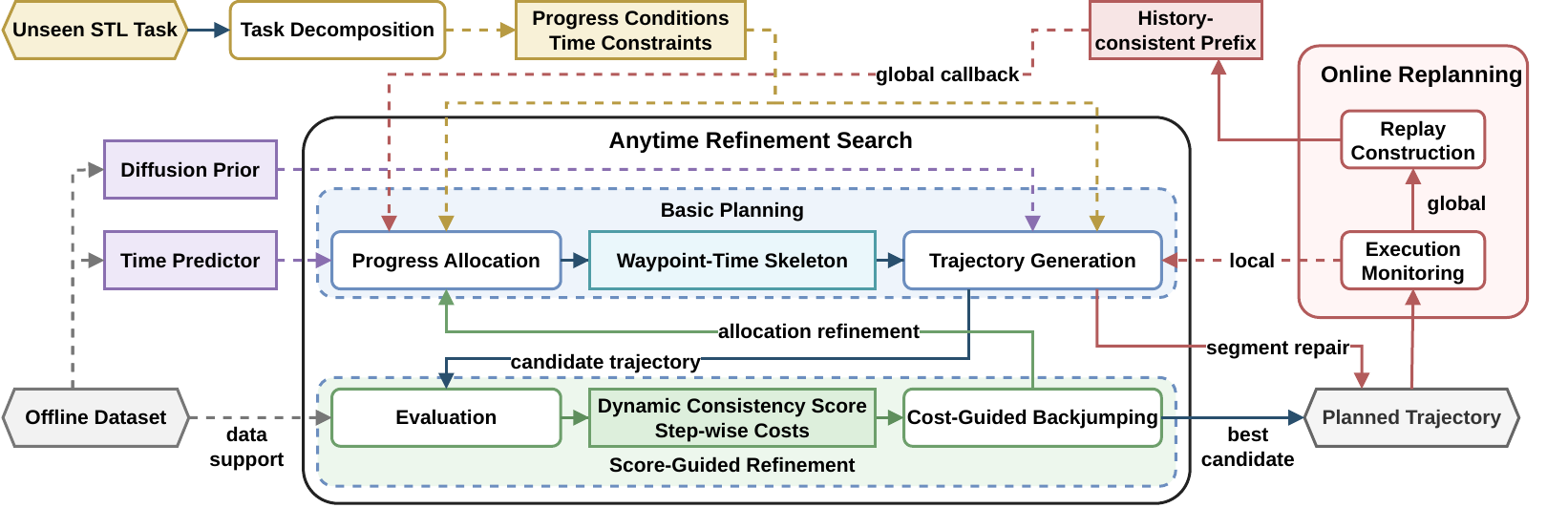}
    \caption{The overall framework of DAG-STL.}
    \label{fig:structure}
\end{figure*}

\subsection{Signal Temporal Logic}\label{sec:STL}
We use signal temporal logic (STL) to specify formal tasks over generated state signals~\cite{maler2004monitoring}. Specifically, we consider STL formulae in \emph{positive normal form} (PNF)~\cite{sadraddini2015robust}, whose syntax is given by
\begin{align}
\varphi ::= \;
&\top \mid \mu \mid \varphi_1 \wedge \varphi_2 \mid \varphi_1 \vee \varphi_2 \mid \F_{[a,b]}\varphi \mid \G_{[a,b]}\varphi \nonumber\\
&\mid \varphi_1 \U_{[a,b]} \varphi_2 ,
\end{align}
where $\top$ is the true predicate, and $\mu$ is an atomic predicate associated with an evaluation function $h_{\mu}:\R^n\to \R$, i.e., $\mu$ is satisfied at state $\x_t$ iff $h_{\mu}(\x_t)\ge 0$.
Furthermore, $\wedge$ and $\vee$ denote the logical operators \emph{conjunction} and \emph{disjunction}, respectively; $\U_{[a,b]}$, $\F_{[a,b]}$, and $\G_{[a,b]}$ denote the temporal operators \emph{until}, \emph{eventually}, and \emph{always}, respectively; and $[a,b]$ is a time interval with $a,b\in\mathbb{Z}$ and $0\le a\le b<\infty$.

Note that negation does not appear in PNF. However, as shown in~\cite{sadraddini2015robust}, this does not result in any loss of generality, since one can redefine atomic predicates to absorb negations and thereby express any STL formula in PNF. For readability, when presenting task formulae and case studies, we still write expressions such as \(\neg\mu\) as shorthand for the corresponding redefined atomic predicate whose sign convention has already been absorbed into its evaluation function. In this work, we impose an additional restriction on the class of PNF formulae considered here: for any subformula of the form $\varphi_1 \U_{[a,b]} \varphi_2$, the prefix formula $\varphi_1$ is allowed to involve only the temporal operator ``always''. This restriction is introduced so that our later decomposition of until formulae remains well structured and tractable. The precise technical role of this restriction will be clarified in Section~\ref{sec:STLdecomposition}.

For any signal $\s=\x_0\x_1\dots\x_T$, we write $\s_t\vDash \varphi$ if $\s$ satisfies $\varphi$ at time step $t$, and $\s\vDash \varphi$ if $\s_0\vDash \varphi$. This is formally defined by the Boolean semantics of STL~\cite{bartocci2018specification}:
\begin{align}
\s_t\vDash \mu
&\Leftrightarrow h_\mu(\x_t) \geq 0, \\
\s_t\vDash \varphi_1 \wedge \varphi_2
&\Leftrightarrow (\s_t \vDash \varphi_1)\ \wedge\ (\s_t \vDash \varphi_2), \\
\s_t \vDash \varphi_1 \vee \varphi_2
&\Leftrightarrow (\s_t \vDash \varphi_1)\ \vee\ (\s_t \vDash \varphi_2), \\
\s_t \vDash \F_{[a,b]} \varphi
&\Leftrightarrow \exists t' \in [t+a, t+b] \text{ s.t. } \s_{t'} \vDash \varphi, \label{formula:sem_F}\\
\s_t \vDash \G_{[a,b]} \varphi
&\Leftrightarrow \forall t' \in [t+a, t+b],\ \s_{t'} \vDash \varphi, \label{formula:sem_G}\\
\s_t \vDash \varphi_1 \U_{[a,b]} \varphi_2
&\Leftrightarrow \exists t' \in [t+a, t+b] \text{ s.t. } \s_{t'} \vDash \varphi_2 \nonumber\\
&\quad \text{and } \forall t'' \in [t, t'],\ \s_{t''} \vDash \varphi_1. \label{formula:sem_U}
\end{align}
For notational convenience, we slightly abuse notation and write $\x_t \vDash \mu$ to denote $h_\mu(\x_t)\ge 0$.

Besides the Boolean semantics, STL also admits a standard \emph{quantitative semantics}, often referred to as the \emph{robustness value}~\cite{donze2010robust,bartocci2018specification}, which measures the degree to which a signal satisfies or violates a specification. For a signal $\s_t$ and a formula $\varphi$, we denote the corresponding robustness by $\rho(\s_t,\varphi)$. Nonnegative robustness indicates satisfaction, while negative robustness indicates violation. For the STL fragment considered in this paper, the quantitative semantics is defined recursively as
\begin{equation}
\begin{aligned}
\rho(\s_t,\mu) &= h_\mu(\x_t),\\
\rho(\s_t,\varphi_1\wedge\varphi_2) &= \min\bigl(\rho(\s_t,\varphi_1),\,\rho(\s_t,\varphi_2)\bigr),\\
\rho(\s_t,\varphi_1\vee\varphi_2) &= \max\bigl(\rho(\s_t,\varphi_1),\,\rho(\s_t,\varphi_2)\bigr),\\
\rho(\s_t,\F_{[a,b]}\varphi) &= \max_{t'\in[t+a,t+b]} \rho(\s_{t'},\varphi),\\
\rho(\s_t,\G_{[a,b]}\varphi) &= \min_{t'\in[t+a,t+b]} \rho(\s_{t'},\varphi),\\
\rho\left(\s_t, \varphi_1 \U_{[a, b]} \varphi_2\right) & =\max _{t^{\prime} \in[t+a, t+b]} \min \{\rho\left(\s_{t^{\prime}}, \varphi_2\right), \\ 
&\quad \quad \quad \quad \quad \quad \min _{\tau \in\left[t, t^{\prime}\right]} \rho\left(\s_\tau, \varphi_1\right)\}.
\end{aligned}
\end{equation}

It is well known that $\rho(\s_t,\varphi)\ge 0$ implies $\s_t\vDash\varphi$, while $\rho(\s_t,\varphi)<0$ implies violation. In this paper, our planning framework is developed primarily from the Boolean semantics above; the robustness value is used later only as a post-hoc evaluation quantity in the experiments.

\subsection{Offline Planning with Unknown Dynamics and Implicit Environment Structure}
In STL planning, the goal is to construct a trajectory whose induced state signal satisfies a given STL specification. When the system dynamics are known, this problem can often be addressed by model-based optimization. In contrast, we consider a purely offline setting in which the dynamics mapping
$f:\R^n \times \R^m \to \R^n$
in~\eqref{formula:system} is unavailable, and no explicit analytical representation of the environment constraints is provided, while a dataset of historical trajectories generated by the underlying system is given. These trajectories are collected from previous task-agnostic operations and may have different lengths.

Importantly, although the dynamics model and the environment geometry are not explicitly available, the task predicates that appear in the STL specification are assumed to be queryable. That is, for each atomic predicate $\mu$, its evaluation function $h_\mu(\x)$ is given, so that whether a state satisfies $\mu$ can be determined.

In this setting, satisfying the STL formula alone is not sufficient. Since no explicit dynamics model or environment model is available at test time, a planned trajectory should also remain consistent with the dynamical and geometric patterns supported by the offline dataset, so that it is more likely to be executable by the underlying system and trackable by a downstream controller. Our goal is therefore to leverage the offline dataset to synthesize new state trajectories that not only satisfy previously unseen STL tasks, but also stay close to the data-supported behavior manifold induced jointly by the unknown dynamics and the implicit environment structure.

\begin{problem}
Given a set of trajectories generated by the unknown system~\eqref{formula:system}, without access to an explicit dynamics model or an explicit environment model, and given an STL formula $\varphi$ in the desired form together with queryable predicate evaluators for its atomic predicates, construct a state trajectory (equivalently, a reference signal) $\s$ such that $\s \vDash \varphi$, using only the offline dataset, while promoting consistency between the planned trajectory and the dynamical and geometric support contained in the dataset.
\end{problem}

\section{Overall Framework of DAG-STL}\label{sec:framework}
Before presenting the technical details, we first summarize the overall design of DAG-STL (Figure~\ref{fig:structure}). We consider STL planning in a purely offline setting, where only an offline trajectory dataset is available. The transition dynamics are unknown, and no explicit map or analytical environment model is provided. Thus, executability-relevant dynamical and environmental structure must be inferred implicitly from the offline data. In this setting, satisfying the STL specification at the planning level is necessary but not sufficient: the resulting trajectory must also remain consistent with offline data support so that it is more likely to be executable by the underlying system and trackable by a downstream controller during execution.

To this end, DAG-STL is organized as a three-stage planning-and-repair framework, in which each stage plays a complementary role:
\begin{itemize}[leftmargin=*]
    \item \textbf{Stage I: Basic Planning Backbone.}  
    This stage forms the backbone of the framework by semantically decomposing the STL task, allocating a temporally consistent waypoint skeleton, and generating a nominal plan that satisfies the specification.
    \item
    \textbf{Stage II: Offline Evaluation and Anytime Refinement.}  
    This stage lifts the basic backbone into an anytime refinement framework. 
    Instead of accepting the first feasible allocation, it uses a rollout-free dynamic consistency metric to assess the generated plan and feeds the resulting signal back to the allocation layer, thereby transforming one-shot planning into a multi-hypothesis, score-guided search over temporally consistent waypoint skeletons. In this way, the refinement stage is not merely a post-processing module, but a strict generalization of the basic planner.
    \item \textbf{Stage III: Online Replanning.}  
    This stage performs online replanning during execution to recover from tracking drift, disturbances, and model mismatch, while preserving the temporal semantics and accumulated progress of the original task.
\end{itemize}

\subsection{Stage I: Basic Planning Backbone}
The basic planner forms the foundation of our
\underline{D}ecomposition,
\underline{A}llocation, and
\underline{G}eneration framework for STL planning.

\textbf{Semantics-Based Task Decomposition.}
Given an STL formula, we decompose it into a set of spatiotemporal progress conditions
$\mathbb{P} = \mathbb{P}^\RC \dot{\cup} \mathbb{P}^\IC$,
where $\mathbb{P}^\RC$ contains \emph{reachability progress conditions} and $\mathbb{P}^\IC$ contains \emph{invariance progress conditions}. These conditions are coupled through time constraints $\mathbb{T}$ over shared time variables $\Lambda$.

\textbf{System-Aware Progress Allocation.}
We then allocate a sequence of time-stamped waypoints
\[
(\tilde{\x}_0,t_0), (\tilde{\x}_1,t_1), \dots, (\tilde{\x}_n,t_n),
\]
such that each waypoint satisfies its assigned reachability progress condition and the associated timestamps satisfy the temporal constraints.

\textbf{Diffusion-Based Trajectory Generation.}
Conditioned on the allocated waypoint-time skeleton, we generate trajectory segments between consecutive waypoints while enforcing the invariance conditions active on the corresponding intervals, and concatenate these segments into a full state trajectory.

\subsection{Stage II: Offline Evaluation and Anytime Refinement}
Although Stage~I produces a nominal STL-satisfying plan, a feasible plan may still contain locally inconsistent transitions that are weakly supported by the offline data manifold. We therefore augment the basic planner with an offline evaluation-and-refinement stage.

\textbf{Dynamic Consistency Metric.}
We define a rollout-free dynamic consistency score $S(\bs{\tau})$ to evaluate the step-wise state and transition support of a planned trajectory under the offline dataset.

\textbf{Anytime Refinement Search (ARS).}
Using this score as feedback, we refine the progress allocation itself by exploring multiple admissible state-time hypotheses and revising upstream decisions responsible for poorly supported local transitions. When the branching factor is set to one and backjumping is disabled, ARS reduces to the Stage~I basic planning pipeline.

\subsection{Stage III: Online Replanning}
To handle tracking drift and disturbances during execution, we further introduce an online replanning mechanism.

\textbf{Online Replanning.}
When necessary, the planner performs either local segment repair or global history-consistent re-allocation, together with a fallback strategy when replanning budgets are exhausted.

The following sections detail these three stages and their corresponding algorithms.

\section{STL Planning via Decomposition, Allocation and Generation}\label{sec:basic_planning}
In this section, we present the technical details of the basic planning module of our framework involving task decomposition, progress allocation and diffusion-based trajectory generation. 
\subsection{Decompositions of STL Formulae}\label{sec:STLdecomposition}
\subsubsection{Eliminating Disjunctions}
Given an STL task $\varphi$, our first step is to transform it into a \emph{conservative disjunctive normal form}, written as
\begin{equation}
    \tilde{\varphi}=\varphi_1 \vee \varphi_2 \vee \cdots \vee \varphi_n,
\end{equation}
where each subformula $\varphi_i$ contains no disjunction operator. This transformation is obtained recursively using the following rewriting rules:
\begin{itemize}[leftmargin=*]
    \item replace
    $\F_{[a,b]}(\varphi_1\vee\varphi_2)$
    by
    $\F_{[a,b]}\varphi_1 \vee \F_{[a,b]}\varphi_2$;
    \item replace
    $\G_{[a,b]}(\varphi_1 \vee \varphi_2)$
    by
    $\G_{[a,b]}\varphi_1 \vee \G_{[a,b]}\varphi_2$;
    \item replace
    $(\phi_1\vee\phi_2)\U_{[a,b]}(\varphi_1\vee\varphi_2)$
    by
    $\bigvee_{i,j\in\{1,2\}} \phi_i\U_{[a,b]}\varphi_j$.
\end{itemize}

The first rewriting rule is logically equivalent, whereas the second and third are conservative strengthenings in general. Therefore, the transformed formula $\tilde{\varphi}$ is generally stronger than the original formula $\varphi$: any trajectory satisfying $\tilde{\varphi}$ also satisfies $\varphi$, but the converse may not hold. This conservative normalization simplifies the subsequent allocation procedure by removing disjunctions from each branch, at the cost of potentially excluding some feasible trajectories of the original task.

Since satisfying $\tilde{\varphi}$ only requires satisfying one of its disjunction-free branches, the planning problem can be solved for each subformula $\varphi_i$ individually. In the full framework, we compute a candidate plan for each branch and then select the best executable one according to the downstream evaluation criterion. For ease of exposition, however, we assume in the remainder of this paper that the formula under consideration is already a single disjunction-free branch, still denoted by $\varphi$. Combined with the preceding positive normal form transformation, this means that the formula contains neither negations applied to temporal operators nor disjunctions. Such preprocessing-based simplifications are common in STL planning and control synthesis when one seeks a tractable intermediate representation for subsequent synthesis or search~\cite{yu2024continuous}. Extending the framework to retain disjunctive choices directly during planning, rather than resolving them through preprocessing, is an interesting direction for future work.

\subsubsection{Progress Conditions and Time Constraints}
We next decompose the STL task $\varphi$ into a finite collection of \emph{progress conditions} coupled through \emph{time constraints}. Specifically, for a signal fragment $\s_t=\x_t\x_{t+1}\dots\x_T$, we consider the following two types of progress conditions:
\begin{itemize}[leftmargin=*]
    \item
    \textbf{Reachability Progress Condition:}
    $\RC(a_{\Lambda}+t,b_{\Lambda}+t,\mu)$ denotes that
    $\exists t' \in [a_{\Lambda}+t,b_{\Lambda}+t]\ \text{such that}\ \x_{t'} \vDash \mu.$
    \item
    \textbf{Invariance Progress Condition:}
    $\IC(a_{\Lambda}+t,b_{\Lambda}+t,\mu)$ denotes that
    $\forall t' \in [a_{\Lambda}+t,b_{\Lambda}+t],\ \x_{t'} \vDash \mu.$
\end{itemize}
Unless otherwise specified, we take $t=0$ and omit it from the notation.

The interval endpoints $a_{\Lambda}$ and $b_{\Lambda}$ may depend on a set of time variables $\Lambda$. Accordingly, the decomposition of $\varphi$ produces a triple $(\PP_{\varphi},\TT_{\varphi},\Lambda_{\varphi})$,
where $\PP_{\varphi}$ is the set of progress conditions, $\TT_{\varphi}$ is the set of time constraints, and $\Lambda_{\varphi}$ is the associated set of time variables.

We use a vector
$\bs{\lambda}=[\lambda_1,\lambda_2,\ldots,\lambda_{|\Lambda|}] \in \mathbb{Z}_{+}^{|\Lambda|}$
to denote an assignment of all time variables in $\Lambda$. Under such an assignment, the symbolic interval endpoints $a_{\Lambda}$ and $b_{\Lambda}$ evaluate to concrete integers, i.e., $a_{\Lambda}(\bs{\lambda}),b_{\Lambda}(\bs{\lambda})\in \mathbb{Z}_{+}$. In our decomposition, these endpoints are always 0--1 linear combinations of time variables, and all constraints in $\TT_{\varphi}$ are simple interval constraints of the form $\lambda_i\in[a_i,b_i]$ involving only a single time variable.

We denote by
\[
\mathcal{F}_{\varphi}
=
\left\{
\bs{\lambda}\in\mathbb{Z}_{+}^{|\Lambda_{\varphi}|}
\;\middle|\;
\bs{\lambda}\ \text{satisfies all constraints in}\ \TT_{\varphi}
\right\}
\]
the set of all feasible assignments of time variables.

With a slight abuse of notation, for $\PC\in\{\RC,\IC\}$, we write
$\s_0 \vDash \PC(a_{\Lambda}(\bs{\lambda}),b_{\Lambda}(\bs{\lambda}),\mu)$,
or simply $\s_0 \vDash \PC(\bs{\lambda})$, to denote the satisfaction of a progress condition under the time-variable assignment $\bs{\lambda}$. Formally,
\begin{align}
\s_0 \vDash \RC(a_{\Lambda}(\bs{\lambda}), b_{\Lambda}(\bs{\lambda}), \mu)
&\Leftrightarrow
\exists t \in [a_{\Lambda}(\bs{\lambda}), b_{\Lambda}(\bs{\lambda})]
\ \text{s.t.}\ 
\x_t \vDash \mu,
\label{formula:reachability}\\
\s_0 \vDash \IC(a_{\Lambda}(\bs{\lambda}), b_{\Lambda}(\bs{\lambda}), \mu)
&\Leftrightarrow
\forall t \in [a_{\Lambda}(\bs{\lambda}), b_{\Lambda}(\bs{\lambda})],\ 
\x_t \vDash \mu.
\label{formula:invariance}
\end{align}

\subsubsection{Progress-Based Decomposition}
We now define recursively how a disjunction-free STL formula $\varphi$ is decomposed into a set of progress conditions $\PP_{\varphi}$, a set of time constraints $\TT_{\varphi}$, and a time-variable set $\Lambda_{\varphi}$.

\noindent\textbf{Base cases:} 
\begin{itemize}[leftmargin=15pt]
\item[1)] 
If $\varphi= \mu$, then we have 
$\mathbb{P}_\varphi= \{  \RC(0,0,\mu) \}$ and $\mathbb{T}_\varphi=\varnothing$, $\Lambda_\varphi=\varnothing$.  
\item[2)] 
If $\varphi= \F_{[a,b]}\mu$, then we have 
$\mathbb{P}_\varphi= \{  \RC(\lambda_i,\lambda_i,\mu)  \}$ and $\mathbb{T}_\varphi=\{ \lambda_i \in [a,b]\}$, $\Lambda_\varphi=\{\lambda_i\}$, where $\lambda_i$ is a newly introduced time variable.  
\item[3)] 
If $\varphi= \G_{[a,b]}\mu$, then we have 
$\mathbb{P}_\varphi= \{  \IC(a,b,\mu)  \}$ and $\mathbb{T}_\varphi=\varnothing,\Lambda_\varphi=\varnothing$.  
\item[4)] 
If $\varphi= \mu_1\U_{[a,b]}\mu_2$, then we have 
$\mathbb{P}_\varphi= \{  \IC(0,\lambda_i,\mu_1),$ $\RC(\lambda_i,\lambda_i,\mu_2)  \}$ and $\mathbb{T}_\varphi=\{\lambda_i\in [a,b]\},\Lambda_\varphi=\{\lambda_i\}$, where $\lambda_i$ is a newly introduced time variable.
\end{itemize}
\noindent\textbf{Recursive compositions:} 
\begin{itemize}[leftmargin=15pt]
    \item[1)]
    If $\varphi=\varphi_1\wedge\varphi_2$, then $\PP_\varphi=\PP_{\varphi_1}\uplus\PP_{\varphi_2},\TT_\varphi=\TT_{\varphi_1}\uplus\TT_{\varphi_2},\Lambda_\varphi=\Lambda_{\varphi_1}\uplus\Lambda_{\varphi_2},$
    where $\uplus$ emphasizes that the merged sets are disjoint.

    \item[2)]
    If $\varphi'=\F_{[a,b]}\varphi$, then we:
    (i) introduce a new time variable $\lambda_i$,
    (ii) add the constraint $\lambda_i\in[a,b]$, and
    (iii) shift every progress condition in $\PP_{\varphi}$ by $\lambda_i$. Formally, $\Lambda_{\varphi'}=\Lambda_{\varphi}\uplus\{\lambda_i\},\TT_{\varphi'}=\TT_{\varphi}\uplus\{\lambda_i\in[a,b]\}$, and
    $\PP_{\varphi'}
    =\left\{\PC(c_{\Lambda}+\lambda_i,d_{\Lambda}+\lambda_i,\mu)\;\middle|\;\PC(c_{\Lambda},d_{\Lambda},\mu)\in\PP_{\varphi}
    \right\}.$
    
    In the special case $\varphi'=\F_{[a,a]}\varphi$, no new time variable is introduced; instead, all progress intervals are shifted by the constant $a$.

    \item[3)]
    If $\varphi'=\G_{[a,b]}\varphi$, we decompose it as
    $
    \G_{[a,b]}\varphi \equiv \bigwedge_{k\in[a,b]}\F_{[k,k]}\varphi$
    and apply the previous rule to each copy independently. Let $\PP_{\varphi}^{(k)}$, $\TT_{\varphi}^{(k)}$, and $\Lambda_{\varphi}^{(k)}$ denote the corresponding copies for time shift $k$. Then
    \begin{align} \Lambda_{\varphi'}=&\biguplus_{k\in[a,b]}\Lambda_\varphi^{(k)}, \mathbb{T}_{\varphi'}=\biguplus_{k\in[a,b]}\mathbb{T}_\varphi^{(k)}\\
    \mathbb{P}_{\varphi'}= &\bigcup_{k\in[a,b]}\{   \PC(c_{\Lambda}\! +\! k,d_{\Lambda}\! +\! k,\mu) \mid    \PC(c_{\Lambda},d_{\Lambda},\mu)  \!\in \!  \mathbb{P}_{\varphi}^{(k)} \} \nonumber 
    \end{align}
    In particular, if an invariance progress condition $\IC(c,d,\mu)\in\PP_\varphi$ has constant endpoints (i.e., it contains no time variables), then the family
    $\{\IC(c+k,d+k,\mu)\}_{k\in[a,b]}$
    can be merged into the single condition
    $\IC(c+a,d+b,\mu)$,
    since these shifted intervals overlap consecutively and their union is exactly $[c+a,d+b]$. We do not apply this merge to invariance progress conditions whose endpoints depend on copy-specific time variables, because the resulting intervals may not form a single contiguous interval under feasible assignments.
    \item[4)]
    If $\varphi'=\phi\U_{[a,b]}\varphi$, then we:
    (i) introduce a new time variable $\lambda_i$,
    (ii) add the constraint $\lambda_i\in[a,b]$,
    (iii) shift every progress condition in $\PP_{\varphi}$ by $\lambda_i$, and
    (iv) extend every invariance progress condition in $\PP_{\phi}$ up to the chosen until time. Formally,
    $\Lambda_{\varphi'}=\Lambda_{\varphi}\uplus\{\lambda_i\},
    \TT_{\varphi'}=\TT_{\varphi}\uplus\{\lambda_i\in[a,b]\},
    \PP_{\varphi'}=
    \left\{ \PC(c_{\Lambda}+\lambda_i,d_{\Lambda}+\lambda_i,\mu)
    \;\middle|\; \PC(c_{\Lambda},d_{\Lambda},\mu)\in\PP_{\varphi}
    \right\}
    \cup
    \left\{
    \IC(c,d+\lambda_i,\mu)
    \;\middle|\;
    \IC(c,d,\mu)\in\PP_{\phi}
    \right\}.
    $
\end{itemize}

\begin{remark}
By the restriction introduced in Section~\ref{sec:STL}, for any subformula of the form $\varphi_1\U_{[a,b]}\varphi_2$, the prefix formula $\varphi_1$ does not contain the temporal operators $\F$ or $\U$. This restriction is essential for the decomposition above. After introducing the until-time variable $\lambda_i$, the prefix part must be enforced over a variable interval of the form $[0,\lambda_i]$. If $\varphi_1$ were allowed to contain eventually- or until-type subformulae, then its decomposition would include reachability-type progress conditions. Enforcing such conditions under a variable-horizon always operator would require expanding obligations over a time interval whose upper bound is itself symbolic, so the number and locations of the resulting shifted progress conditions could not be determined a priori. Under the present restriction, this difficulty does not arise: in the case $\phi\U_{[a,b]}\varphi$, the decomposition of $\phi$ introduces neither time variables nor reachability progress conditions. In particular, we have $\Lambda_\phi=\varnothing$, $\TT_\phi=\varnothing$, and $\PP_{\phi}$ consists only of invariance progress conditions.
\end{remark}

% \begin{remark}
% By the restriction introduced in Section~\ref{sec:STL}, for any subformula of the form $\varphi_1\U_{[a,b]}\varphi_2$, the prefix formula $\varphi_1$ does not contain the temporal operators $\F$ or $\U$. Therefore, in the case $\phi\U_{[a,b]}\varphi$, no time variables or reachability progress conditions are introduced by decomposing $\phi$. In particular, we have $\Lambda_\phi=\varnothing$, $\TT_\phi=\varnothing$, and $\PP_{\phi}$ consists only of invariance progress conditions.
% \end{remark}

We provide the following example to illustrate the above decomposition procedure.

\begin{example}\label{example:dec}
\begin{figure}[t]
    \centering
    \includegraphics[width=0.45\textwidth]{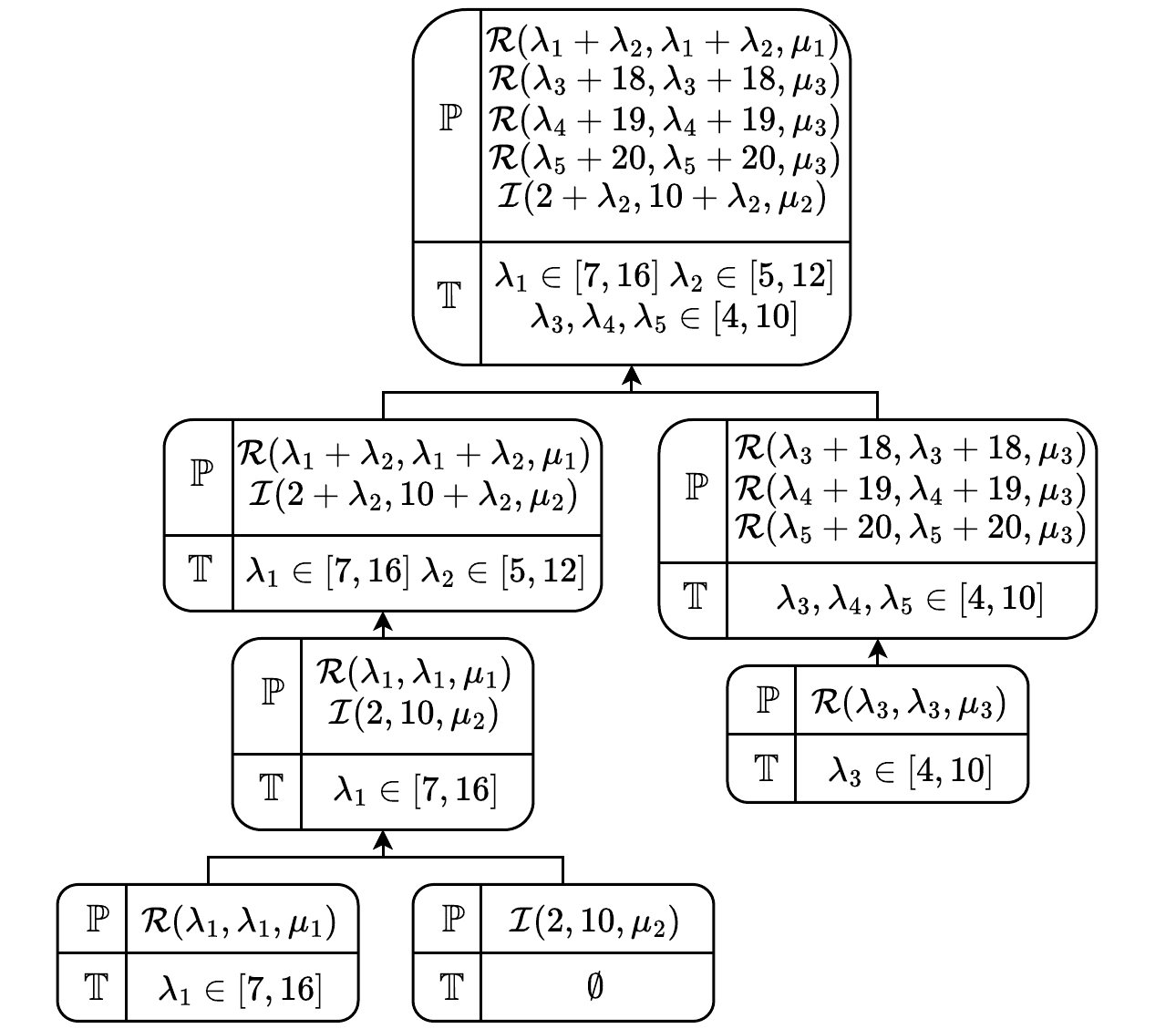}
    \caption{Decomposition process of the STL formula in~\eqref{formula:STL}.}
    \label{fig:decomposition}
\end{figure}
Let us consider the STL formula
\begin{equation}
\varphi=\F_{[5,12]}(\F_{[7,16]}\mu_1 \wedge \G_{[2,10]}\mu_2)\wedge \G_{[18,20]}\F_{[4,10]}\mu_3.
\label{formula:STL}
\end{equation}
The decomposition process is shown in Figure~\ref{fig:decomposition}, where progress conditions and time constraints are constructed incrementally from the leaves to the root. The top node represents the final decomposition
$(\PP_\varphi,\TT_\varphi,\Lambda_\varphi)$
of the STL formula $\varphi$.

One feasible assignment is $\bs{\lambda}=[11,5,4,4,4]$, under which the progress condition $\RC(\lambda_1+\lambda_2,\lambda_1+\lambda_2,\mu_1)$ is satisfied at $t=16$, the conditions
$\RC(\lambda_3+18,\lambda_3+18,\mu_3)$,
$\RC(\lambda_3+19,\lambda_3+19,\mu_3)$, and
$\RC(\lambda_3+20,\lambda_3+20,\mu_3)$
are satisfied at $t=22,23,24$, respectively, and the invariance condition
$\IC(2+\lambda_2,10+\lambda_2,\mu_2)$
is satisfied over the interval $t\in[7,15]$. According to this assignment, one can construct a signal satisfying the STL formula by enforcing
$\x_{16}\vDash \mu_1,
\x_{7:15}\vDash \mu_2,
\x_{22:24}\vDash \mu_3,$
while leaving all other states unconstrained.
\end{example}

The following lemma formalizes the soundness of the above semantics-based decomposition.

\begin{lemma}[(Soundness of the Decomposition)]\label{lemma:task_dec}
Let $\varphi$ be a disjunction-free STL formula in PNF, and let
$(\PP_{\varphi},\TT_{\varphi},\Lambda_{\varphi})$
be the result of the recursive decomposition above. Then, for every signal $\s_0$,
\[
\s_0 \vDash \varphi
 \Leftrightarrow 
\exists \bs{\lambda}\in\mathcal{F}_{\varphi}
\ \text{such that}\
\forall \PC\in\PP_{\varphi},\ 
\s_0 \vDash \PC(\bs{\lambda}).
\]
In particular, when $\Lambda_{\varphi}=\varnothing$, this reduces to
\[
\s_0 \vDash \varphi
 \Leftrightarrow 
\forall \PC\in\PP_{\varphi},\ 
\s_0 \vDash \PC.
\]
\end{lemma}
\begin{proof}
The proof is provided in Appendix~\ref{apx:proof}.
\end{proof}

This semantics-based decomposition turns STL planning into the search for a signal that satisfies a finite collection of progress conditions under a feasible time-variable assignment, rather than the direct search over the space of all possible signals.

\subsection{Progress Allocation}\label{sec:allocation}
By Lemma~\ref{lemma:task_dec}, STL satisfaction can be reduced, at the logical level, to satisfying all decomposed progress conditions under some feasible time-variable assignment. Accordingly, the role of the allocation module is to construct a temporally consistent timed waypoint skeleton that witnesses the reachability progress conditions while preserving the feasibility of the remaining invariance progress conditions.

In the unknown-dynamics setting considered here, however, such an allocation cannot be performed arbitrarily: even if a logical allocation exists, the resulting transitions between successive waypoints may still be dynamically implausible for the underlying system. To address this issue, we adopt a search-based allocation algorithm whose branching decisions are guided by a data-driven reachability-time predictor learned from offline trajectories.

\subsubsection{Preprocessing}
Before running the allocation search, we further rewrite each invariance progress condition
$\IC(a_{\Lambda},b_{\Lambda},\mu)\in\PP_\varphi$
into a triggering reachability condition
$\RC(a_{\Lambda},a_{\Lambda},\mu)$
and a residual invariance condition
$\IC(a_{\Lambda}+1,b_{\Lambda},\mu)$.
For simplicity, we denote the resulting progress-condition set by
$(\PP,\TT)$ without subscripts, where
$\PP=\PP^\RC \dot{\cup} \PP^\IC$.

For example, in Example~\ref{example:dec}, the condition
$\IC(2+\lambda_2,10+\lambda_2,\mu_2)$
is rewritten as
$\RC(2+\lambda_2,2+\lambda_2,\mu_2)$
and
$\IC(3+\lambda_2,10+\lambda_2,\mu_2)$.
After this rewriting, every invariance progress condition is associated with a unique triggering reachability condition that determines when the invariance requirement starts to take effect.

\begin{remark}
Intuitively, a reachability progress condition captures an instantaneous property achievement, whereas an invariance progress condition captures the maintenance of a property over an interval. The above rewriting makes the start of every invariance requirement explicit as a reachability event, so that the allocation procedure can reason in a unified way about when task-relevant property changes occur.
\end{remark}

\subsubsection{Main Allocation Algorithm}
\renewcommand{\algorithmicrequire}{\textbf{Input:}}
\renewcommand{\algorithmicensure}{\textbf{Output:}}

\begin{algorithm}[t]
\small
\caption{\texttt{MainAllocation}}\label{alg:allocation}
\begin{algorithmic}[1]
\REQUIRE Initial state $\x_0$, reachability progress conditions $\PP^{\RC}$, invariance progress conditions $\PP^{\IC}$, initial time-constraint set $\TT$
\ENSURE A feasible timed waypoint sequence $\tilde{\s}$ together with the final time-constraint set $\TT_f$, or \texttt{None} if no feasible allocation is found

\STATE \textbf{Initialize:}
\STATE current state $\x \leftarrow \x_0$; current time $t \leftarrow 0$
\STATE timed waypoint sequence $\tilde{\s} \leftarrow [(\x, t)]$
\STATE \textit{stack} $\leftarrow [(\x, t, \PP^{\RC}, \TT, \tilde{\s})]$

\WHILE{\textit{stack} is not empty}
    \STATE $(\x, t, \PP^{\RC}_{\mathrm{rem}}, \TT, \tilde{\s}) \leftarrow$ \textit{pop}(\textit{stack})

    \IF{$\PP^{\RC}_{\mathrm{rem}} = \emptyset$}
        \RETURN $\tilde{\s}, \TT$
    \ENDIF
    \STATE Collect the currently determined subset $\PP^{\IC}_{\mathrm{det}} \subseteq \PP^{\IC}$
    \FORALL{$\RC(a_\Lambda, b_\Lambda, \mu) \in \PP^{\RC}_{\mathrm{rem}}$}
        \STATE $(\x', t') \leftarrow \texttt{TimeAssign}\bigl(\RC(a_\Lambda, b_\Lambda, \mu), \x, t, \TT, \PP^{\IC}_{\mathrm{det}}\bigr)$
        \IF{$(\x', t') \neq \texttt{None}$}
            \STATE $\tilde{\s}' \leftarrow \tilde{\s} \cup \{(\x', t')\}$
            \STATE $(\PP^{\RC}_{\mathrm{rem}})' \leftarrow \PP^{\RC}_{\mathrm{rem}} \setminus \{\RC(a_\Lambda, b_\Lambda, \mu)\}$
            \STATE $\TT' \leftarrow \texttt{UpdateConstraint}(a_\Lambda, b_\Lambda, \TT, \x', t')$
            \IF{$\TT'$ remains feasible}
                \STATE \textit{push} $\bigl(\x', t', (\PP^{\RC}_{\mathrm{rem}})', \TT', \tilde{\s}'\bigr)$ onto \textit{stack}
            \ENDIF
        \ENDIF
    \ENDFOR
\ENDWHILE

\RETURN \texttt{None}
\end{algorithmic}
\end{algorithm}

The main allocation algorithm is shown in Algorithm~\ref{alg:allocation}. Its goal is to construct a sequence of timed waypoints such that each reachability progress condition is witnessed at an assigned time, while the evolving time-constraint set remains feasible and no already-activated invariance progress condition is violated at the waypoint level.

Formally, the algorithm performs a depth-first search (DFS) over candidate assignments for the reachability progress conditions in $\PP^\RC$. When the search succeeds, it returns a timed waypoint sequence of the form
\[
(\tilde{\x}_0,t_0),(\tilde{\x}_1,t_1),\dots,(\tilde{\x}_n,t_n),
\]
where each waypoint corresponds to the satisfaction of one reachability progress condition in $\PP^\RC$.
During the search, we maintain:
the current state $\x$,
the current time step $t$,
the set of remaining reachability progress conditions $\PP^\RC$,
the current time-constraint set $\TT$,
and the partial timed waypoint sequence $\tilde{\s}$.

At each step, the algorithm selects a remaining reachability progress condition
$\RC(a_{\Lambda}, b_{\Lambda}, \mu)\in\PP^\RC$
as the next target from the current state-time pair $(\x,t)$.
To realize this progress condition, the function \texttt{TimeAssign} proposes a candidate satisfaction time $t'$ together with a waypoint state $\x'$ such that $\x'\vDash \mu$.
The waypoint $(\x',t')$ is then appended to $\tilde{\s}$, the selected reachability progress condition is removed from $\PP^\RC$, and the current time-constraint set $\TT$ is updated by \texttt{UpdateConstraint} according to the assigned pair $(\x',t')$.
The algorithm then proceeds recursively.
If all reachability progress conditions have been assigned, the search terminates successfully.
If no feasible assignment exists at the current search node, the algorithm backtracks to explore alternative choices.

\begin{remark}\label{remark:more_candidate}
Algorithm~\ref{alg:allocation} is written to terminate immediately upon finding a feasible solution. In practice, one may instead record the current solution as a candidate and continue backtracking, thereby generating multiple feasible timed waypoint skeletons. These candidates can then be ranked according to additional criteria. This modification will be further developed in Section~\ref{sec:mc_planning}.
\end{remark}

\subsubsection{Dynamic Maintenance of Time Constraints}
During the allocation process, we maintain the current time-constraint set $\TT$ incrementally.
Whenever a new reachability progress condition is assigned a satisfaction time, additional constraints over the existing time variables are introduced.
This allows the planner to query, at any intermediate stage, the currently admissible value ranges of symbolic interval bounds such as $a_{\Lambda}$ and $b_{\Lambda}$ in a progress condition $\PC(a_{\Lambda},b_{\Lambda},\mu)$ by solving a small integer linear program (ILP), while remaining consistent with all previously imposed timing decisions.

For example, the minimum possible value of $a_{\Lambda}$ under the current constraint set $\TT$, denoted by $a_{\Lambda,\TT}^{\min}$, is obtained by solving
\[
\begin{aligned}
&\mathrm{minimize} \quad a_\Lambda(\bs{\lambda}) \\
&\text{subject to} \quad \bs{\lambda}\ \text{satisfies all constraints in}\ \TT, \\
&\phantom{\text{subject to}} \quad \bs{\lambda} \in \mathbb{Z}_{+}^{|\Lambda|}.
\end{aligned}
\]
The minimum and maximum feasible values of other symbolic bounds are computed analogously.

Importantly, the set of time variables $\Lambda$ remains fixed throughout the allocation process: the planner only tightens the current constraint set $\TT$, and does not introduce or remove time variables dynamically. As more constraints are added, the feasible assignment set shrinks monotonically. If at any point the resulting feasible set becomes empty, then the current partial allocation is infeasible and the search must backtrack.

\subsubsection{Constraint Updates}
We now describe the function \texttt{UpdateConstraint}.
Let
$\RC(a_{\Lambda}, b_{\Lambda}, \mu)$
be the selected reachability progress condition, and let
$(\x',t')$
be the assigned waypoint and satisfaction time.
To certify that this reachability progress condition can indeed be realized at time \(t'\), we add the constraints
\begin{equation}\label{eq:constraint-add}
   a_{\Lambda} \leq t',
   \qquad
   b_{\Lambda}\geq t' .
\end{equation}
These constraints ensure that the chosen time \(t'\) lies inside the admissible satisfaction window of the selected reachability progress condition.

Recall that, in the preprocessing step, every original invariance progress condition
$\IC(a_{\Lambda},b_{\Lambda},\mu)$
is rewritten as
$\RC(a_{\Lambda},a_{\Lambda},\mu)$
and
$\IC(a_{\Lambda}+1,b_{\Lambda},\mu)$.
Therefore, once constraints~\eqref{eq:constraint-add} have been added for the triggering reachability condition
$\RC(a_{\Lambda},a_{\Lambda},\mu)$,
the start time of the associated residual invariance condition is determined.
We denote by $\PP^\IC_{\mathrm{det}}$ the set of invariance progress conditions whose start times have already been determined in this way.

Next, consider any determined invariance progress condition
$\IC(c,d_\Lambda,\mu)\in \PP^\IC_{\mathrm{det}}$.
If the newly assigned waypoint violates its predicate, i.e., if $\x' \nvDash \mu$, then this invariance condition cannot remain active at time \(t'\).
Hence, we further add the constraint
\[
d_{\Lambda}< t'.
\]
This update truncates the feasible active interval of the invariance progress condition so as to avoid logical inconsistency with the newly assigned waypoint.

\subsubsection{Sampling Timed Waypoints}
\renewcommand{\algorithmicrequire}{\textbf{Input:}}
\renewcommand{\algorithmicensure}{\textbf{Output:}}

\begin{algorithm}[t]
\small
\caption{\texttt{TimeAssign}}\label{algorithm_TA}
\begin{algorithmic}[1]
\REQUIRE Reachability progress condition $\RC(a_\Lambda,b_\Lambda,\mu)$, current state $\x$, current time $t$, current time-constraint set $\TT$, determined invariance progress conditions $\PP^{\IC}_{\mathrm{det}}$
\ENSURE A locally admissible candidate pair $(\x', t_{\mathrm{new}})$, or \texttt{None} if no candidate is found

\STATE $t_{\min}\gets a_{\Lambda,\TT}^{\mathrm{min}}, \qquad t_{\max}\gets b_{\Lambda,\TT}^{\mathrm{max}}$
\FOR{up to $N_{\mathrm{max}}$ attempts}
    \STATE Sample state $\x'$ such that $\x'\vDash \mu$
    \STATE \textbf{Initialize:} conflict interval set $\mathcal{O} \leftarrow \emptyset$
    \FORALL{$\IC(c,d_\Lambda,\mu')\in \PP^{\IC}_{\mathrm{det}}$ with determined start time}
        \IF{$\x' \nvDash \mu'$}
            \STATE $\mathcal{O}\leftarrow \mathcal{O}\cup [c,d_{\Lambda,\TT}^{\mathrm{min}}]$
        \ENDIF
    \ENDFOR
    \STATE $t' \leftarrow t+\texttt{TimePredict}(\x,\x')$
    \IF{$t' > t_{\max}$ \textbf{or} $[\max\{t',t_{\min}\},t_{\max}]\setminus \mathcal{O}=\emptyset$}
        \STATE \textbf{continue} to the next sampling attempt
    \ENDIF
    \STATE $t_{\mathrm{new}}\leftarrow$ earliest time in $[\max\{t',t_{\min}\},t_{\max}] \setminus \mathcal{O}$
    \RETURN $(\x',t_{\mathrm{new}})$
\ENDFOR
\RETURN \texttt{None}
\end{algorithmic}
\end{algorithm}

When a reachability progress condition
$\RC(a_\Lambda, b_\Lambda, \mu)$
is selected, the function \texttt{TimeAssign} proposes a \emph{locally admissible} satisfaction time together with a corresponding waypoint state, while accounting for the currently available timing constraints and the currently determined invariance progress conditions.
The pseudocode is shown in Algorithm~\ref{algorithm_TA}.

The procedure first computes a heuristic admissible time window
\(
[t_{\min},t_{\max}]
\)
for the selected reachability progress condition under the current time-constraint set.
It then attempts, up to \(N_{\max}\) times, to draw a candidate state \(\x'\) from the satisfying set of \(\mu\) using a predicate-conditioned state sampler. The proposed framework does not depend on a specific implementation of this sampler; any mechanism capable of proposing states satisfying \(\mu\), including learned generative samplers, can be used. For each candidate, the algorithm computes a conflict interval set, denoted by \(\mathcal{O}\), that captures time intervals during which assigning \(\x'\) would necessarily conflict with currently determined invariance progress conditions.

Once \(\mathcal{O}\) is obtained, the function \texttt{TimePredict} estimates the transition time from the current state \(\x\) to the sampled state \(\x'\), thereby injecting data-driven dynamical information into the otherwise symbolic allocation process.
If
(i) the predicted arrival time exceeds \(t_{\max}\), or
(ii) the entire candidate interval
\(
[\max(t',t_{\min}),t_{\max}]
\)
is blocked by conflict intervals in \(\mathcal{O}\),
then the sampled state \(\x'\) is rejected and the algorithm proceeds to the next attempt.
Otherwise, the earliest conflict-free time in this interval is selected as the candidate satisfaction time, and the function returns this time together with the sampled state.

It is important to emphasize that \texttt{TimeAssign} is a heuristic candidate generator rather than a complete feasibility solver for the updated constraint system. The quantities
\(t_{\min}\), \(t_{\max}\), and \(\mathcal{O}\)
provide a fast local screening mechanism based on the currently available timing information and the learned transition-time prior. Therefore, the returned pair \((\x',t_{\mathrm{new}})\) should be understood as a locally admissible candidate under the current search state. Global consistency is enforced only after \texttt{UpdateConstraint} is applied and the resulting time-constraint set is checked for feasibility; otherwise, the branch is rejected and the search backtracks.

\begin{remark}
When computing the conflict interval set \(\mathcal{O}\), we only consider invariance progress conditions whose start times have already been determined. This is because, after preprocessing, every invariance progress condition has a unique triggering reachability progress condition. If its start time has not yet been determined, then its trigger has not yet been allocated, and the corresponding invariance requirement can only become active strictly later in the search.
\end{remark}

\subsubsection{Heuristic Ordering}
In the above DFS procedure, when selecting the next reachability progress condition from \(\PP^\RC\), we use the following heuristic ordering to accelerate search in practice: 
\begin{itemize}[leftmargin=*]
\item Reachability progress conditions with earlier potential deadlines \(b_{\Lambda,\TT}^{\min}\) are prioritized; 
\item If two progress conditions share the same earliest potential deadline, the one with the smaller earliest potential start time \(a_{\Lambda,\TT}^{\min}\) is preferred. 
\end{itemize}
This heuristic prioritizes more temporally urgent progress conditions while preserving slack for later ones, which empirically reduces the likelihood of downstream conflicts and speeds up the search for feasible allocations.

\subsubsection{Prediction of Reachability Time}\label{sec:time_predict}
In Algorithm~\ref{algorithm_TA}, the function \texttt{TimePredict} estimates the transition time (equivalently, the trajectory length) required to move from the current state \(\x\) to a candidate waypoint \(\x'\). We refer to the learned module behind this function as the \emph{time predictor}, and to its callable interface inside the planning algorithms as \texttt{TimePredict}. The time predictor is learned from the same offline trajectory dataset that is later used to train the diffusion-based trajectory generator.

In general, there may exist multiple feasible trajectories between two states with different lengths. To capture this non-uniqueness, we model the transition time as a conditional distribution and use a conditional diffusion model to learn
$p_{\theta}(l \mid \x_s, \x_g)$,
where \(\x_s\) and \(\x_g\) denote the start and goal states, and \(l\) is the trajectory length.

We denote by \(l_{\mathrm{norm}}\) the unguided prediction of the transition length, which reflects a typical length under the learned distribution. To estimate shorter or longer feasible transition times, denoted by \(l_{\min}\) and \(l_{\max}\), respectively, we further employ SVDD~\cite{li2024derivative} to guide the sampling process toward shorter or longer trajectory lengths. SVDD is particularly suitable for the time predictor because it preserves sample diversity, incurs no additional training cost, and can be implemented efficiently through parallel generation.

By selecting among \(l_{\mathrm{norm}}, l_{\min}\), and \(l_{\max}\), the planner can adjust the conservativeness of the timing prediction according to different planning objectives. In addition, for tasks involving strong invariance requirements, we introduce a scaling factor \(\gamma\) to enlarge the predicted transition time when needed, thereby providing an additional control knob over the conservativeness of the allocation procedure. More importantly, these multiple predicted lengths can also be viewed as alternative temporal hypotheses for the same transition. This property will later be exploited by the multi-hypothesis assignment mechanism described in Section~\ref{sec:mc_planning}, where different time predictions are retained as candidate allocations rather than collapsed into a single deterministic estimate.

\subsubsection{Soundness of the Progress Allocation Algorithm}
We now state the soundness property of the progress allocation module.

\begin{lemma}[(Soundness of Progress Allocation)]
\label{lemma:progress_allo}
Let $(\x_0,\PP^{\RC},\PP^{\IC},\TT)$ be the input to Algorithm~\ref{alg:allocation}, and assume that the initial time-constraint set \(\TT\) is feasible. Suppose Algorithm~\ref{alg:allocation} terminates successfully and returns the timed waypoint sequence
\[
\tilde{\s}=
(\tilde{\x}_0,t_0),(\tilde{\x}_1,t_1),\dots,(\tilde{\x}_n,t_n),
0=t_0\le t_1\le \dots \le t_n.
\]
Let the final time-constraint set be denoted by \(\TT_f\), and let
\[
\mathcal{F}
=
\left\{
\bs{\lambda}\in\mathbb{Z}_{+}^{|\Lambda|}
\;\middle|\;
\bs{\lambda}\ \text{satisfies all constraints in}\ \TT_f
\right\}.
\]
Then \(\mathcal{F}\neq\varnothing\), and for every signal
\(\s_0=\x_0\x_1\dots\x_T\) with \(T\ge t_n\) and
\(\x_{t_i}=\tilde{\x}_i\) for all \(0\le i\le n\),
there exists \(\bs{\lambda}\in\mathcal{F}\) such that
\[ 
\begin{aligned}
&\forall \RC(a_\Lambda, b_\Lambda, \mu) \!\in\! \PP^\RC,\ 
\s_0 \vDash \RC(a_\Lambda(\bs{\lambda}), b_\Lambda(\bs{\lambda}),\mu), \text{ and}\\
&\forall \IC(a_\Lambda, b_\Lambda, \mu) \!\in\! \PP^\IC,\ 
\forall t \!\in\! [a_\Lambda(\bs{\lambda}), b_\Lambda(\bs{\lambda})]\cap \{t_i\}_{i=0}^{n},\ 
\x_t \vDash \mu.
\end{aligned} 
\]
\end{lemma}
\begin{proof}
The proof is deferred to Appendix~\ref{apx:proof}.
\end{proof}

In other words, the returned timed waypoint sequence guarantees the existence of a feasible time-variable assignment under which every reachability progress condition is witnessed at the waypoint level, and no invariance progress condition is violated at any selected waypoint during its active interval.

\subsection{Diffusion-Based Trajectory Generation} \label{sec:traj_generation}

Lemma~\ref{lemma:progress_allo} guarantees that the timed waypoint sequence returned by the Progress Allocation module already witnesses all reachability progress conditions and remains waypoint-wise consistent with all active invariance progress conditions under a feasible time-variable assignment. The remaining task is therefore one of \emph{trajectory completion}: given the allocated timed waypoint skeleton, we must generate a full system trajectory that interpolates all waypoints at their assigned times and satisfies the required invariance progress conditions over the intervals between them.

Formally, suppose the Progress Allocation algorithm returns the timed waypoint sequence
\[
\tilde{\s} = (\tilde{\x}_0, t_0) (\tilde{\x}_1, t_1) \dots (\tilde{\x}_n, t_n),
0 = t_0 \leq t_1 \leq \dots \leq t_n.
\]
Our goal is to generate a state trajectory (equivalently, a reference signal)
\(\s_0=\x_0\x_1\dots\x_T\),
such that
\begin{equation}
\left\{
\begin{aligned}
& T \geq t_n, \\
& \x_{t_i} = \tilde{\x}_i \quad \text{for all } i \in \{0, \dots, n\}, \\
& \forall \IC \in \PP^\IC,\ \s_0 \vDash \IC(\bs{\lambda}^\star),
\end{aligned}
\right.
\end{equation}
for some selected feasible assignment \(\bs{\lambda}^\star \in \mathcal{F}\).
Since the true system dynamics are unknown, this completion step cannot rely on an explicit dynamics model. Instead, the generated state trajectory should remain close to the offline trajectory distribution induced by the unknown system, so that it is more likely to be executable by the downstream controller.

Note that the active time intervals of all invariance progress conditions are fully determined once \(\bs{\lambda}^\star\) is fixed. A special case arises when some invariance progress conditions remain active beyond the final waypoint time \(t_n\). In such cases, we append a terminal holding segment after \(t_n\) so that the trajectory remains at the final satisfying state in the planning space until all active invariance conditions expire. This is valid because Lemma~\ref{lemma:progress_allo} guarantees that \(\tilde{\x}_n\) satisfies every invariance predicate that is active at time \(t_n\).

\subsubsection{Diffusion Models for Trajectory Completion}
We formulate the above completion problem as a waypoint-conditioned, invariance-constrained generative synthesis problem over state trajectories, and instantiate it using a diffusion-based planner.
Unlike autoregressive approaches that generate states sequentially and may accumulate local errors over long horizons, diffusion models synthesize an entire trajectory segment holistically~\cite{janner2022planning,ajay2022conditional}. This trajectory-level generation paradigm is particularly suitable here, because it allows the boundary conditions, temporal horizon, and state-wise invariance constraints of each segment to be enforced jointly during sampling.

From the offline planning perspective, the diffusion model learns a distribution of data-supported state trajectories from historical rollouts of the unknown system, thereby turning trajectory completion into a conditional sampling problem.
This formulation is especially appealing in our setting for two reasons.
First, it naturally accommodates waypoint conditioning, which is required to connect the allocated timed waypoints.
Second, it provides flexible mechanisms for incorporating constraints during denoising, including strict boundary enforcement and projection-based correction for safety constraints.
In our implementation, we instantiate this module using the \texttt{Diffuser} architecture~\cite{janner2022planning}. Nevertheless, the overall planning framework is not tied to this specific backbone and can be combined with other diffusion-based trajectory generators.

\subsubsection{Trajectory Stitching for Long Horizons}
Directly generating a single long-horizon trajectory that satisfies all STL-induced conditions is challenging, since the required behavior may deviate substantially from the offline training distribution, which typically consists of much shorter trajectory snippets. To improve tractability, we therefore decompose the completion problem into a sequence of segment-wise generation problems and stitch the resulting segments together~\cite{li2024diffstitch,luo2025generative}.

Specifically, for each adjacent waypoint pair
\[
(\tilde{\x}_i,t_i),\;(\tilde{\x}_{i+1},t_{i+1}),
\]
we synthesize a state trajectory segment \(\s^{(i)}\) that: 
\begin{itemize}[leftmargin=*]
    \item starts at \(\tilde{\x}_i\) and ends at \(\tilde{\x}_{i+1}\); 
    \item contains exactly \(t_{i+1}-t_i+1\) states, corresponding to the planning interval \([t_i,t_{i+1}]\); 
    \item satisfies all invariance progress conditions active on that interval. 
\end{itemize}
If an invariance progress condition spans multiple adjacent waypoint intervals, then it is enforced on each affected segment over the corresponding sub-interval. The full state trajectory is finally obtained by concatenating all generated segments in temporal order. This segment-wise stitching strategy balances tractability and expressiveness: it enables long-horizon STL planning while keeping each individual generation problem close to the learned offline trajectory distribution.

\subsubsection{Constrained Generation Mechanisms}\label{sec:constrained_gen}
To realize the above segment-wise completion problem, we incorporate three complementary mechanisms into the diffusion sampling process, corresponding respectively to temporal structure control, boundary-condition enforcement, and invariance satisfaction. 
\begin{itemize}[leftmargin=*]
\item \textbf{Length control.}
The temporal horizon of a diffusion-based trajectory planner is determined by the dimensionality of the sampled noise input rather than being rigidly fixed by the network structure~\cite{janner2022planning}. We exploit this flexibility to generate trajectory segments whose lengths exactly match the prescribed waypoint intervals.
Although standard models trained on fixed-length trajectory crops exhibit some degree of length generalization, their performance typically degrades when the target horizon differs significantly from the training horizon. To improve robustness to variable segment lengths, we adopt the variable-horizon training strategy proposed in~\cite{liu2025vh}, in which the model is trained on trajectory segments of diverse lengths. This improves the model's temporal generalization ability and supports reliable segment generation for trajectory stitching.
\item \textbf{Waypoint inpainting.}
To enforce the boundary conditions of each segment strictly, we treat trajectory completion as an inpainting problem~\cite{janner2022planning}. After each denoising step, we replace the first and last states of the sample with the assigned boundary waypoints
\(
\tilde{\x}_i,\tilde{\x}_{i+1},
\)
so that the intermediate states are iteratively refined into a coherent transition between the two fixed endpoints.
\vspace{-6pt}
\item \textbf{Invariance constraints via projection.}
To enforce invariance progress conditions, each active condition
\(\IC(a,b,\mu)\)
is converted into the corresponding state-wise constraints
\[
h_\mu(\x_t)\ge 0,
\qquad \forall t\in[a,b].
\]
Such constrained diffusion generation has been studied extensively in recent work~\cite{xiao2023safediffuser,botteghi2023trajectory,mizuta2024cobl,zheng2024safe,christopher2024constrained}. In our prior work~\cite{liu2025zeroshot}, we used \texttt{SafeDiffuser}~\cite{xiao2023safediffuser}, but found it computationally expensive and less flexible for the diverse geometric constraints encountered here. We therefore adopt a projection-based mechanism inspired by~\cite{christopher2024constrained}, in which intermediate samples are projected back onto the feasible set after each denoising step. This yields a more flexible and computationally efficient way to enforce segment-wise invariance conditions. 
\end{itemize}

Overall, the trajectory generation module lifts the waypoint-level guarantees of Lemma~\ref{lemma:progress_allo} to full trajectory-level satisfaction by completing the allocated timed waypoint skeleton with segment-wise, data-supported state trajectories that preserve boundary conditions and enforce all active invariance requirements.

\subsection{Action Recovery and Control Protocol}\label{sec:action_control}
The trajectory generation module produces only a state trajectory \(\s\), rather than a joint state-action trajectory. Accordingly, the planner treats the generated state sequence as a high-level reference to be tracked during execution. This design separates high-level task planning from low-level actuation, which is appealing in our setting for two reasons: state trajectories are typically smoother and easier to model than high-frequency action sequences~\cite{ajay2022conditional}, and the resulting decoupling allows a downstream controller to compensate for execution errors and disturbances during tracking. This abstraction is most suitable when the state captures the main task-relevant feasibility structure; for systems with strong action-dependent constraints, executability must additionally rely on offline data support and, if necessary, execution-time correction.

To execute the planned state trajectory on the physical system, we introduce an \emph{action recovery} layer. In particular, control actions can be obtained either from a learned inverse dynamics model~\cite{agrawal2016learning} or from a low-level tracking controller, such as a PD controller, that tracks the generated state sequence. This execution layer is not part of the logical synthesis procedure itself; rather, it serves as the interface that maps the planned state trajectory to executable control inputs.

A key distinction from standard goal-reaching problems is that STL tasks impose explicit timing requirements. To preserve the temporal semantics of the planned trajectory during execution, we adopt a \emph{time-synchronous control protocol}. Specifically, we define a synchronization ratio \(k \in \mathbb{Z}_{+}\), meaning that each discrete planning step corresponds to exactly \(k\) control cycles of the underlying simulator or hardware platform. At runtime, for every planned transition \((\x_t,\x_{t+1})\), the inverse dynamics model or tracking controller is executed for exactly \(k\) control updates. In this way, the execution layer is aligned with the temporal structure of the planned trajectory. Unless otherwise specified, we set \(k=1\) in all experiments, corresponding to a one-to-one mapping between planning and control steps.

\subsection{Soundness and Completeness Analysis}
We conclude this section by discussing the theoretical properties of the overall basic planning module. We first establish its \emph{soundness} with respect to STL satisfaction, namely, whenever the planner successfully returns a completed trajectory, the resulting state signal satisfies the target STL specification.

\begin{theorem}[(Soundness of the Overall Planner)]
\label{theorem:diffusionSTL}
Let $\varphi$ be an STL specification satisfying all standing assumptions in Sections~\ref{sec:STL} and~\ref{sec:STLdecomposition}. Let
\(
(\tilde{\x}_0,t_0),\dots,(\tilde{\x}_n,t_n)
\)
be the timed waypoint sequence returned by Algorithm~\ref{alg:allocation}, and let
\(
\bs{\lambda}^\star
\)
be a feasible time-variable assignment for the final time-constraint set certified by Lemma~\ref{lemma:progress_allo}. Suppose the trajectory generation module outputs a state trajectory
\(
\s_0=\x_0\x_1\dots\x_T, T\ge t_n,
\)
such that:
\begin{enumerate}[label=(\roman*),leftmargin=*]
\item \(\x_{t_i}=\tilde{\x}_i\) for all \(i\in\{0,\dots,n\}\);
\item for every invariance progress condition \(\IC(a_\Lambda,b_\Lambda,\mu)\in\PP^\IC\), the signal \(\s_0\) satisfies \(\IC(a_\Lambda(\bs{\lambda}^\star),b_\Lambda(\bs{\lambda}^\star),\mu)\) over its full active interval.
\end{enumerate}
Then the resulting signal satisfies the original STL specification, i.e.,
\[
\s_0 \vDash \varphi.
\]
\end{theorem}
\begin{proof}
The proof is deferred to Appendix~\ref{apx:proof}.
\end{proof}

\begin{remark}
The theorem separates the roles of the two modules in DAG-STL. The allocation module guarantees that all reachability progress conditions are witnessed at the waypoint level and that no active invariance progress condition is violated at any selected waypoint under a feasible assignment \(\bs{\lambda}^\star\). The trajectory generation module then lifts this waypoint-level consistency to full trajectory-level satisfaction by enforcing all invariance progress conditions over the intervals between adjacent waypoints. Together with Lemma~\ref{lemma:task_dec}, this yields satisfaction of the original STL specification.

In our implementation, the assumptions imposed on the trajectory generation module are operationally enforced by the constrained sampling mechanisms introduced in Section~\ref{sec:traj_generation}: segment length is fixed by length control, boundary waypoints are enforced by inpainting, and invariance progress conditions are enforced by projection-based constrained sampling. Thus, the theorem abstracts the generation module in a solver-agnostic way, while remaining directly aligned with the concrete constrained diffusion implementation used in this work.
\end{remark}

Despite this soundness guarantee, the basic planning module is not \emph{complete}: it may fail to return a solution even when a feasible trajectory exists for the underlying system. This loss of completeness arises from three sources.

\begin{enumerate}[label=(\alph*),leftmargin=*]
\item \textbf{Logical conservativeness.}
When the original specification contains disjunctions, we first transform it into PNF and then apply the conservative branch-wise normalization described in Section~\ref{sec:STLdecomposition}. The resulting formula is generally stronger than the original one. Consequently, some trajectories that satisfy the original specification may be excluded by this preprocessing step.

\item \textbf{Search-level incompleteness in progress allocation.}
Algorithm~\ref{alg:allocation} relies on heuristic search components to remain tractable in the unknown-dynamics setting, including the learned transition-time predictor and the local candidate-generation strategy in \texttt{TimeAssign}. Although the algorithm uses backtracking, these heuristics may still delay or miss valid allocations in difficult cases. Section~\ref{sec:mc_planning} introduces an anytime refinement mechanism to mitigate this limitation.

\item \textbf{Data- and generator-level incompleteness.}
The trajectory generator is learned entirely from offline data and therefore only approximates the distribution of dynamically plausible trajectories induced by the unknown system. If the dataset does not sufficiently cover a required transition, the generator may fail to synthesize a valid segment even though such a segment exists in the true physical system.
\end{enumerate}

Particularly, the third source is fundamental to the purely offline setting considered in this work: without access to the true system dynamics or online interaction, no planner can in general guarantee that every physically feasible transition is represented faithfully by the offline data distribution or recovered by the learned generator.
For this reason, our method follows a \emph{sound-but-conservative} design philosophy. Rather than pursuing full completeness, which is generally unattainable in the offline unknown-dynamics setting, the basic planner returns a trajectory only when STL satisfaction can be certified at the planning level through decomposition, progress allocation, and constrained trajectory completion. At the same time, because logical satisfaction alone does not guarantee executability under unknown dynamics, the subsequent evaluation, refinement, and replanning modules are introduced to favor trajectories that are better supported by the offline data and therefore more likely to be successfully tracked during execution.

\section{Evaluation, Refinement, and Replanning}\label{sec:ver_re_re}
The Basic Planning module guarantees STL satisfaction at the planning level, but this guarantee alone may be insufficient to ensure reliable execution in the purely offline setting considered here. Since the true system dynamics and environment constraints are unavailable, the planner must rely entirely on learned components, including the diffusion-based trajectory generator and the transition-time predictor. Consequently, even an STL-correct open-loop plan may still be only weakly supported by the offline behavior manifold, for example due to locally unrealistic transitions or passages through low-support regions. This mismatch between \emph{logical correctness} and \emph{execution plausibility} constitutes a fundamental limitation of purely offline STL planning and was also the primary source of execution failures in the earlier pipeline~\cite{liu2025zeroshot}.

To bridge this gap, we further augment the basic planning module with three tightly integrated components. First, we introduce a \emph{rollout-free, data-driven consistency metric} that evaluates planned trajectories directly from offline dataset statistics, so as to identify locally weakly supported states and transitions and localize bottleneck segments. Second, we elevate the planning process itself through \emph{Anytime Refinement Search (ARS)}, which generalizes the Basic Planning allocator into a multi-hypothesis search procedure and iteratively refines the allocation via cost-guided backjumping under finite computational budgets. Third, we develop an \emph{online replanning protocol} that provides hierarchical closed-loop recovery against runtime disturbances and tracking drift.

\subsection{Dynamic Consistency Metric}\label{sec:dyn_metric}
Diffusion models generate trajectories from a \emph{global} sequence-level perspective and are often effective at matching the overall distribution of offline data. However, when conditioned on additional task requirements, a sampled trajectory may still contain a small number of \emph{local} inconsistencies, namely states or transitions that are only weakly supported by the offline behavior manifold and are therefore difficult for downstream controllers to track reliably. Motivated by this observation, we introduce a \emph{rollout-free, data-driven Dynamic Consistency Metric} that evaluates local trajectory quality directly from offline dataset statistics.

Note that this metric is not intended to provide an exact test of dynamical feasibility. Rather, it serves as a practical proxy for downstream executability in the fully offline setting: trajectories that are better supported by the offline data manifold are expected to be easier to realize and track. The metric plays two complementary roles in our refinement pipeline. First, it assigns \emph{step-wise} costs that localize the weakest-supported states or transitions, thereby enabling targeted refinement. Second, it aggregates these local costs into a \emph{trajectory-level} score for ranking multiple STL-compliant candidates according to their physical plausibility. It is worth noting that this evaluation is lightweight, amenable to batch parallelization, and performed directly on the planned state trajectory without requiring any rollout of an explicit dynamics model.

\subsubsection{Setup}
Let a planned state trajectory be
\(
\bs{\tau}=\x_0,\x_1,\dots,\x_T.
\)
Our goal is to evaluate local consistency along \(\bs{\tau}\) from two complementary perspectives: \emph{state support} and \emph{transition support}, both measured relative to the offline dataset \(\mathcal D\).

In practice, such consistency often depends only on task-relevant or kinematically meaningful quantities (e.g., positions and/or velocities), rather than the full raw state. To accommodate this, we introduce a generic feature map
\begin{equation}
    \z_t = \zeta(\x_t)\in\mathbb{R}^d,
\end{equation}
where \(\zeta(\cdot)\) may be the identity map or a projection onto a task-relevant subspace.

\subsubsection{Offline Fitting: Support Estimators}
From the offline dataset \(\mathcal D\), we derive two lightweight nonparametric support estimators. Both return nonnegative \emph{support costs}, where a larger value indicates weaker support from the offline data.
\begin{itemize}[leftmargin=*]
\item \textbf{State Support.}
A function
\(
s_x:\mathbb{R}^{d}\rightarrow\mathbb{R}_{\ge 0}
\)
that measures how well an individual feature state \(\z\) is supported by the offline state manifold.

\item \textbf{Transition Support.}
A function
\(
s_\Delta:\mathbb{R}^{2d}\rightarrow\mathbb{R}_{\ge 0}
\)
that measures how well a local feature transition is supported by the offline transition manifold.
Specifically, for a local transition \((\z_t,\z_{t+1})\), we define the joint transition feature
\begin{equation}
    \phi_t \triangleq \begin{bmatrix}\z_t \\ \Delta \z_t\end{bmatrix}\in\mathbb{R}^{2d},
\end{equation}
where
\(
\Delta \z_t = \z_{t+1}-\z_t.
\)
To avoid scale dominance among heterogeneous dimensions (e.g., position versus velocity), we apply per-dimension normalization to \(\phi_t\) using statistics computed from offline transitions, and denote the normalized feature by \(\tilde{\phi}_t\). The transition support cost is then evaluated on \(\tilde{\phi}_t\).
\end{itemize}

In our implementation, both \(s_x(\cdot)\) and \(s_\Delta(\cdot)\) are instantiated using a simple \(k\)NN-based support proxy, namely the distance to the \(k\)-th nearest neighbor~\cite{cover1967nearest}. This choice is classical from the perspective of local density and outlier scoring~\cite{ramaswamy2000outliers,breunig2000lof,sun2022dnn}, while remaining lightweight enough for batch trajectory evaluation in our offline planning loop.

\subsubsection{Step-Wise Costs}
For each transition step \(t\), we define three local costs from the feature trajectory \(\{\z_t\}_{t=0}^{T}\). The first two are \emph{data-support} terms, while the third is a \emph{control-oriented regularization} term.

\textbf{(1) State-Support Cost.}
We penalize visits to weakly supported states via
\begin{equation}
    c^{\text{state}}_t \;=\; s_x(\z_{t+1}),
\end{equation}
which measures whether the next planned state lies on, or near, the offline state manifold in the chosen feature space.
When \(\z\) (or a designated subset of \(\z\)) admits an explicit Euclidean interpretation, this evaluation can be strengthened by additionally applying \(s_x(\cdot)\) to a small set of interpolated interior points between \(\z_t\) and \(\z_{t+1}\) within that Euclidean subspace. This remains rollout-free and helps detect support violations that are not visible from the endpoints alone.

\textbf{(2) Transition-Support Cost.}
We penalize locally unsupported transitions via
\begin{equation}
    c^{\text{trans}}_t \;=\; s_\Delta(\tilde{\phi}_t),
\end{equation}
which evaluates whether the local displacement is typical under the offline transition manifold in the normalized joint space of \([\z_t;\Delta\z_t]\).

\textbf{(3) Heuristic Step Cost (data-independent regularizer).}
As a complementary signal, we introduce a lightweight \emph{data-independent} step cost
\(c^{\text{step}}_t\)
defined on feature increments. Unlike \(c^{\text{state}}_t\) and \(c^{\text{trans}}_t\), this term is not intended to estimate offline data support. Instead, it serves as a rollout-free regularizer that captures local tracking difficulty and can be adapted to the specific domain and low-level control protocol. In our implementation, it is composed of the following terms.

First, we penalize excessively large local displacements using a thresholded step-length cost:
\begin{equation}
    c^{\text{len}}_t \;=\; \big[\|\Delta \z_t\|_2 - \delta\big]_+,
\end{equation}
where \(\delta>0\) is a tunable threshold and \([\cdot]_+=\max(\cdot,0)\).

When \(\z\) admits an explicit Euclidean interpretation, we additionally penalize sharp turning and non-smooth local variation:
\begin{equation}
\begin{aligned}
    c^{\text{turn}}_t &\;=\; 1 - \frac{\Delta \z_{t-1}^\top \Delta \z_t}{\|\Delta \z_{t-1}\|_2\,\|\Delta \z_t\|_2 + \epsilon},\\
    c^{\text{smooth}}_t &\;=\; \|\Delta \z_t - \Delta \z_{t-1}\|_2,
\end{aligned}
\end{equation}
where \(\epsilon>0\) ensures numerical stability, and \(c^{\text{turn}}_0=c^{\text{smooth}}_0=0\).

We combine these heuristic terms into a single regularization cost:
\begin{equation}
    c^{\text{step}}_t \;=\; c^{\text{len}}_t \;+\; \lambda_{\text{turn}}\,c^{\text{turn}}_t \;+\; \lambda_{\text{smooth}}\,c^{\text{smooth}}_t,
\end{equation}
with nonnegative weights \(\lambda_{\text{turn}},\lambda_{\text{smooth}}\).

\subsubsection{Total Step Cost and Trajectory Score}
We combine the three local costs as
\begin{equation}
    c_t \;=\; w_{\text{state}}\, c^{\text{state}}_t \;+\; w_{\text{trans}}\, c^{\text{trans}}_t \;+\; w_{\text{step}}\, c^{\text{step}}_t,
\end{equation}
where all weights are nonnegative.

To obtain a trajectory-level evaluation that is robust to horizon length while remaining sensitive to a small number of severe bottlenecks, we aggregate the step-wise costs using a CVaR-style tail mean~\cite{rockafellar2000optimization}:
\begin{equation}
    \mathrm{CVaR}_{\eta}(\{c_t\}_{t=0}^{T-1}) 
    \triangleq \frac{1}{K}\sum_{i\in \mathcal{I}_{\eta}} c_i,
    \qquad
    K=\lceil \eta T\rceil,
\end{equation}
where \(\mathcal{I}_{\eta}\) indexes the \(K\) largest step costs and \(\eta\in(0,1]\) is a tail fraction. Compared with a simple average, this aggregation emphasizes rare but severe local failures; compared with a pure maximum, it is less brittle to isolated numerical outliers.

Finally, we define the \emph{Dynamic Consistency Score} as
\begin{equation}
    S(\bs{\tau}) = -\,\mathrm{CVaR}_{\eta}(\{c_t\}_{t=0}^{T-1}),
\end{equation}
so that higher values indicate trajectories with better local support and fewer severe bottlenecks.

The resulting trajectory-level score and step-wise costs will be used in two ways: to rank completed STL-compliant candidates globally, and to localize bottleneck segments for targeted refinement in the next subsection.

\begin{remark}\label{rem:taskspace}
A common alternative in high-dimensional control domains is to plan in a reduced task space and rely on a low-level controller for full-body tracking~\cite{luo2025generative}. Our metric applies to this setting directly by choosing the feature map \(\zeta(\cdot)\) to output task-space coordinates, so that all local support costs are computed from \((\z_t,\z_{t+1})\) and \(\Delta \z_t\) exactly as above. Because task-space dynamics are typically more weakly specified than full-state dynamics, under-supported task-space transitions naturally receive larger consistency costs and are therefore deprioritized during ranking and refinement.
\end{remark}

\subsection{Score-Guided Refinement Search}\label{sec:mc_planning}
As introduced in Section~\ref{sec:dyn_metric}, the Dynamic Consistency Metric provides a trajectory-level score together with step-wise costs that quantify local support under the offline data manifold. Empirically, poor consistency scores typically arise from two distinct sources:
\begin{enumerate}[label=(\alph*),leftmargin=*]
    \item \textbf{Sampling Variance.}
    Because trajectory generation is stochastic, the diffusion model may occasionally produce low-quality samples even when higher-quality completions exist under the same waypoint allocation. This issue can often be mitigated by parallel sampling and candidate selection.
    \item \textbf{Allocation Inconsistency.}
    The allocated state-time skeleton itself may be poorly aligned with the offline-supported transition structure, for example due to errors in the time predictor or an overly aggressive temporal allocation. Such a structural mismatch cannot be reliably corrected by simply re-sampling trajectories under the same allocation.
\end{enumerate}

In this section, we address the second source through \emph{planning-level} refinement. Rather than post-processing a generated trajectory directly, we refine the \emph{allocation process itself} by turning the allocator into an anytime branching search over diversified state-time hypotheses. Here, diversification arises along two complementary dimensions: spatial diversity from sampling different waypoint states, and temporal diversity from retaining multiple reachability-time hypotheses produced by the distributional time predictor. In this sense, the proposed refinement module is a strict generalization of the Basic Planning allocator rather than a post-hoc add-on. The trajectory-level consistency score is used to evaluate completed candidates, while the step-wise costs are used to localize bottleneck transitions and drive cost-guided backjumping.

This refinement strategy offers two advantages over restarting the planner from scratch:
\begin{enumerate}[label=(\roman*),leftmargin=*]
    \item \textbf{Incremental continuation.}
    Refinement reuses the existing allocation prefix together with its accumulated time-constraint context, thereby avoiding repeated global recomputation of timing windows and schedule checks.
    \item \textbf{Bottleneck-localized regeneration.}
    Using segment-level costs, the search can jump directly to the decision responsible for the dominant bottleneck and regenerate only the affected suffix, rather than rebuilding the entire plan from scratch.
\end{enumerate}

The refinement module consists of two interacting components:
\begin{enumerate}[label=(\roman*),leftmargin=*]
    \item \textbf{\texttt{MultiHypothesisAssign} (decision diversification).}
    A candidate-generation routine that generalizes deterministic assignment. Instead of committing to a single successor, it returns a set of locally admissible state-time hypotheses, thereby increasing branching diversity over both spatial targets and temporal allocations.
    \item \textbf{Anytime Refinement Search (ARS).}
    A stack-based anytime search procedure that explores the branched allocation space, evaluates each completed candidate using the dynamic consistency score, and performs cost-guided backjumping using the resulting step-wise costs to iteratively improve the best-found solution under finite computation budgets.
\end{enumerate}

\subsubsection{Multi-Hypothesis Assignment}

To mitigate allocation inconsistency, we further upgrade the deterministic \texttt{TimeAssign} (Algorithm~\ref{algorithm_TA}) to \texttt{MultiHypothesisAssign} (Algorithm~\ref{algorithm_MTA}). This extension directly leverages the distributional nature of the diffusion-based time predictor introduced in Section~\ref{sec:time_predict}. Instead of collapsing the predicted transition time between \((\x,\x')\) into a single estimate, \texttt{MultiHypothesisAssign} retains a set \(\mathcal{T}_{\mathrm{pred}}\) of duration hypotheses and treats them as alternative temporal candidates during allocation.

The routine performs up to \(N_{\mathrm{max}}\) sampling attempts to identify target states \(\x'\) satisfying \(\mu\). For each sampled state, every predicted duration hypothesis in \(\mathcal{T}_{\mathrm{pred}}\) is screened against the current timing window and conflict set, yielding a collection of locally admissible candidate pairs \((\x',t_{\mathrm{cand}})\). The resulting candidate set \(\mathcal{C}\) therefore introduces both spatial diversity (different target states \(\x'\)) and temporal diversity (different arrival times), reducing premature commitment to a brittle allocation.

\begin{algorithm}[t]
\small
\caption{The \texttt{MultiHypothesisAssign} procedure.}\label{algorithm_MTA}
\begin{algorithmic}[1]
\REQUIRE Reachability progress condition $\RC(a_\Lambda,b_\Lambda,\mu)$, current state $\x$, current time $t$, current time-constraint set $\TT$, determined invariance progress conditions $\PP^{\IC}_{\mathrm{det}}$, candidate capacity $K$, maximum sampling attempts $N_{\mathrm{max}}$
\ENSURE Candidate set $\mathcal{C}$ of locally admissible state-time hypotheses

\STATE $\mathcal{C}\leftarrow\emptyset$
\STATE $t_{\min}\gets a_{\Lambda,\TT}^{\mathrm{min}},\quad t_{\max}\gets b_{\Lambda,\TT}^{\mathrm{max}}$

\FOR{$i = 1$ \TO $N_{\mathrm{max}}$}
    \STATE Sample target state $\x'$ such that $\x'\vDash \mu$
    \STATE $\mathcal{O} \gets \texttt{ComputeConflict}(\x', \PP^{\IC}_{\mathrm{det}})$
    \STATE $\mathcal{T}_{\mathrm{pred}} \gets \texttt{TimePredict}(\x,\x')$
    
    \FORALL{$\hat{t} \in \mathcal{T}_{\mathrm{pred}}$}
        \STATE $t_{\mathrm{cand}} \leftarrow t + \hat{t}$
        \IF{$t_{\mathrm{cand}} \in [t_{\min}, t_{\max}] \setminus \mathcal{O}$}
            \STATE $\mathcal{C} \leftarrow \mathcal{C} \cup \{(\x', t_{\mathrm{cand}})\}$
            \IF{$|\mathcal{C}| \ge K$}
                \RETURN $\mathcal{C}$
            \ENDIF
        \ENDIF
    \ENDFOR
\ENDFOR
\RETURN $\mathcal{C}$
\end{algorithmic}
\end{algorithm}

\subsubsection{Anytime Allocation and Refinement}
The \texttt{AnytimeRefinementSearch} procedure (Algorithm~\ref{alg:ARS}) integrates these diversified hypotheses into an anytime depth-first search framework.
At each decision node, the search first branches over the remaining reachability progress conditions and then, for each such logical branch, inserts all locally admissible candidate pairs in \(\mathcal{C}\) into the search frontier. In this way, ARS preserves both multiple logical orderings and multiple local state-time realizations for future refinement.
Whenever a complete timed waypoint sequence is assembled, the trajectory generator is invoked and the resulting trajectory is evaluated by the Dynamic Consistency Metric:
\[
    \big(S(\bs{\tau}),\; \{c_t(\bs{\tau})\}_{t=0}^{T-1}\big)
    \leftarrow
    \texttt{ConsistencyEvaluation}(\bs{\tau}),
\]
where \(\{c_t(\bs{\tau})\}_{t=0}^{T-1}\) denotes the step-wise costs used for bottleneck localization.
Unlike verification-based pipelines, ARS does not terminate after passing a binary feasibility test; instead, it continues until a prescribed budget is exhausted and returns the best-scoring trajectory found so far.

\subsubsection{Cost-Guided Backjumping}
Standard DFS backtracks chronologically, which may waste computation on recent decisions even when the root cause of poor consistency lies earlier in the allocation.
To accelerate improvement, we introduce \emph{Cost-Guided Backjumping} based on segment-level risk.

Given the step-wise costs \(\{c_t\}_{t=0}^{T-1}\) returned by \texttt{ConsistencyEvaluation}, we partition them according to the current segment decomposition, where each segment corresponds to a contiguous subset of time indices.
For each segment \(k\), we compute a segment risk cost via tail-risk aggregation:
\begin{equation}
    R_k \;\triangleq\; \mathrm{CVaR}_{\eta}\big(\{c_t\}_{t\in \mathcal{I}_k}\big),
\end{equation}
where \(\mathcal{I}_k\) denotes the set of step indices associated with segment \(k\).
We then identify the worst segment
\begin{equation}
    k^\star \;\leftarrow\; \arg\max_{k} R_k,
\end{equation}
and prune the search stack so that refinement resumes from the decision node that instantiated segment \(k^\star\).
This directly redirects search effort toward the structural bottleneck that dominates execution risk.

The process continues until a global iteration budget \(N_{\mathrm{iter}}\) or a completed-solution budget \(N_{\mathrm{seq}}\) is reached.
Finally, ARS returns the best-effort trajectory
\[
    \bs{\tau}^\star \leftarrow \arg\max_{\bs{\tau}\in \mathcal{B}} S(\bs{\tau}),
\]
among all evaluated candidates \(\mathcal{B}\).

\begin{algorithm}[t]
\small
\caption{The \texttt{AnytimeRefinementSearch} procedure.}\label{alg:ARS}
\begin{algorithmic}[1]
\REQUIRE Initial state $\x_0$, STL decomposition $(\PP^{\RC},\PP^{\IC},\TT)$, completed-solution budget $N_{\mathrm{seq}}$, iteration budget $N_{\mathrm{iter}}$
\ENSURE Best-effort state trajectory $\bs{\tau}^\star$

\STATE Initialize \textit{stack} with root node
\(
(\x_0,0,\PP^{\RC}_{\mathrm{rem}}=\PP^{\RC}, \TT,\tilde{\s}=\emptyset)
\)
\STATE $\bs{\tau}^\star \leftarrow \text{NULL}, \qquad S_{\max} \leftarrow -\infty$
\STATE $n_{\mathrm{sol}} \leftarrow 0, \qquad n_{\mathrm{iter}} \leftarrow 0$

\WHILE{\textit{stack} $\neq \emptyset$ \textbf{and} $n_{\mathrm{iter}} < N_{\mathrm{iter}}$}
    \STATE $(\x, t, \PP^{\RC}_{\mathrm{rem}}, \TT, \tilde{\s}) \leftarrow$ \textit{pop}(\textit{stack})
    \STATE $n_{\mathrm{iter}} \leftarrow n_{\mathrm{iter}} + 1$
    
    \IF{$\PP^{\RC}_{\mathrm{rem}} = \emptyset$}
        \STATE $n_{\mathrm{sol}} \leftarrow n_{\mathrm{sol}} + 1$
        \STATE $\bs{\tau} \leftarrow \texttt{TrajectoryGeneration}(\tilde{\s}, \PP^{\IC}, \TT)$
        
        \algsep
        \algcomment{Consistency evaluation and incumbent update}
        \STATE $(S(\bs{\tau}), \{c_t\}_{t=0}^{T-1}) \leftarrow \texttt{ConsistencyEvaluation}(\bs{\tau})$
        \IF{$S(\bs{\tau}) > S_{\max}$}
            \STATE $S_{\max} \leftarrow S(\bs{\tau}),\bs{\tau}^\star \leftarrow \bs{\tau}$
        \ENDIF
        \IF{$n_{\mathrm{sol}} \ge N_{\mathrm{seq}}$}
            \RETURN $\bs{\tau}^\star$
        \ENDIF
        
        \algsep
        \algcomment{Cost-guided backjumping}
        \STATE Compute segment costs $\{R_k\}$ from $\{c_t\}_{t=0}^{T-1}$ via segment-wise $\mathrm{CVaR}_{\eta}$
        \STATE Identify the worst segment $k^\star \leftarrow \arg\max_k R_k$
        \STATE Prune \textit{stack} so that search resumes from the decision that instantiated segment $k^\star$
        \algsep
        
        \STATE \textbf{continue}
    \ENDIF
    
    \STATE Collect the currently determined subset $\PP^{\IC}_{\mathrm{det}} \subseteq \PP^{\IC}$
    
    \algsep
    \algcomment{Branching over remaining progress conditions and local state-time hypotheses}
    \FORALL{$\RC \in \PP^{\RC}_{\mathrm{rem}}$ in the heuristic order}
        \STATE $\mathcal{C} \leftarrow \texttt{MultiHypothesisAssign}(\RC, \x, t, \TT, \dots)$
        \FORALL{$(\x', t') \in \mathcal{C}$}
            \STATE Apply \texttt{UpdateConstraint} to obtain updated constraint set $\TT'$
            \STATE Let $(\PP^{\RC}_{\mathrm{rem}})' \leftarrow \PP^{\RC}_{\mathrm{rem}} \setminus \{\RC\}$
            \IF{$\TT'$ remains feasible}
                \STATE $\tilde{\s}' \leftarrow \tilde{\s} \cup \{(\x', t')\}$
                \STATE \textit{push} $\bigl(\x', t', (\PP^{\RC}_{\mathrm{rem}})', \TT', \tilde{\s}' \bigr)$ onto \textit{stack}
            \ENDIF
        \ENDFOR
    \ENDFOR
    \algsep
\ENDWHILE

\RETURN $\bs{\tau}^\star$
\end{algorithmic}
\end{algorithm}

\begin{remark}[(Improved Search Coverage and Degeneracy to the Baseline)]
Beyond improving dynamic consistency, ARS also improves the search coverage of the planning framework. Like the baseline allocator (Algorithm~\ref{alg:allocation}), it still branches over the remaining reachability progress conditions. The difference is that the baseline keeps at most one local state-time successor for each logical branch and terminates upon finding the first feasible complete assignment, which can lead to premature commitment to a brittle solution. In contrast, ARS widens each logical branch through \texttt{MultiHypothesisAssign}, which expands the search space over both waypoint states and predicted arrival times, and continues to explore completed candidates under finite budgets, thereby reducing the risk of missing high-quality solutions due to greedy early stopping.

Importantly, ARS is a strict \emph{generalization} of the Basic Planning procedure rather than a post-hoc add-on. When the candidate capacity is set to one and \texttt{MultiHypothesisAssign} degenerates to \texttt{TimeAssign}, while the anytime search terminates upon the first complete assignment, ARS reduces to the baseline planner. This provides a practical knob for trading off computation against solution quality.
\end{remark}

\subsection{Online Replanning}\label{sec:online_replanning}
Although the offline planning-and-refinement framework improves the quality of the nominal plan substantially, execution-time disturbances, model mismatch, and accumulated tracking error may still cause the system to deviate from the intended trajectory. To preserve task progress under such deviations, we introduce a hierarchical online replanning mechanism, summarized in Algorithm~\ref{alg:online_replanning}.

At execution step \(t\), let
\(
e_t = \|\x_t-\hat{\x}_t\|_2
\)
denote the time-aligned tracking error, where \(\hat{\x}_t\) is the nominal reference state at time \(t\). Given two thresholds
\(
\varepsilon_{\mathrm{loc}} < \varepsilon_{\mathrm{glob}},
\)
the controller operates in three regimes: nominal tracking when the deviation is small, local segment repair when the deviation is moderate, and global history-consistent replanning when the deviation becomes critical or local repair fails.

\begin{algorithm}[t]
\small
\caption{The \texttt{OnlineReplanning} procedure.}\label{alg:online_replanning}
\begin{algorithmic}[1]
\REQUIRE Current execution state $(\x_t,t)$, nominal reference trajectory $\hat{\s}$, nominal timed waypoint sequence $\tilde{\s}$, current progress/constraint context, thresholds $\varepsilon_{\mathrm{loc}}, \varepsilon_{\mathrm{glob}}$
\ENSURE Updated execution plan: repaired local suffix, globally replanned suffix, or fallback action

\algcomment{Tracking-error monitoring}
\STATE Compute tracking error $e_t = \|\x_t-\hat{\x}_t\|_2$

\algsep
\algcomment{Nominal tracking regime}
\IF{$e_t \le \varepsilon_{\mathrm{loc}}$}
    \STATE Continue tracking the nominal reference trajectory
    \RETURN nominal suffix of $\hat{\s}$
\ENDIF

\algsep
\algcomment{Local replanning: segment repair}
\IF{$\varepsilon_{\mathrm{loc}} < e_t \le \varepsilon_{\mathrm{glob}}$}
    \STATE Identify the next scheduled waypoint $(\tilde{\x}_k,t_k)$ from $\tilde{\s}$
    \STATE Query $\hat{t} \leftarrow \texttt{TimePredict}(\x_t,\tilde{\x}_k)$
    \IF{$t+\hat{t} \le t_k$}
        \STATE Attempt local segment repair from $\x_t$ to $\tilde{\x}_k$
        \IF{local repair succeeds}
            \RETURN repaired local segment + unchanged nominal suffix
        \ENDIF
    \ENDIF
\ENDIF

\algsep
\algcomment{Global replanning: history-consistent re-allocation}
\STATE Reconstruct the executed allocation prefix using scheduled waypoint times from $\tilde{\s}$
\STATE Inject the current execution state $(\x_t,t)$ into the reconstructed constraint context
\STATE Resume allocation and trajectory generation for the remaining suffix
\IF{global replanning succeeds}
    \RETURN globally replanned suffix
\ENDIF

\algsep
\algcomment{Fallback protocol}
\STATE Trigger fallback policy
\IF{Persistence mode}
    \STATE Continue tracking the current nominal suffix for a short horizon and retry replanning
    \RETURN persistence action
\ELSE
    \STATE Abort execution and declare task failure
    \RETURN abort action
\ENDIF
\end{algorithmic}
\end{algorithm}

\subsubsection{Local Replanning (Segment Repair)}\label{sec:local_replanning}
When the deviation is moderate, we attempt to repair only the current segment from the true execution state \(\x_t\) to the next scheduled waypoint
\(
(\tilde{\x}_k,t_k),
\)
namely, the next uncompleted reachability progress condition on the nominal timeline. As a lightweight pre-check, we use \texttt{TimePredict} to estimate the residual travel time \(\hat{t}\) to \(\tilde{\x}_k\). If
\(
t+\hat{t}\le t_k,
\)
the nominal deadline is still regarded as recoverable, and we invoke the trajectory generation module to synthesize a repaired segment from \(\x_t\) to \(\tilde{\x}_k\) while enforcing all invariance progress conditions active on the corresponding interval. If this pre-check fails, or if the segment repair attempt is unsuccessful, the procedure escalates to global replanning.

\subsubsection{Global Replanning (History-Consistent Re-allocation)}\label{sec:global_replanning}
When local repair is insufficient, we recompute the remaining allocation of progress conditions from the current execution state. Rather than restarting from scratch, we first reconstruct the allocator state induced by the executed prefix of the nominal plan in a \emph{history-consistent} manner.

Concretely, global replanning re-runs the allocation algorithm (Algorithm~\ref{alg:allocation}) from the beginning, but with the decisions along the already executed nominal prefix fixed in advance. At replay step \(j\), the allocator is forced to commit the previously selected reachability progress condition together with its nominal waypoint \(\tilde{\x}_j\) and its scheduled satisfaction time \(\tilde{t}_j\), and the same constraint-update operations as in the original search are re-applied. Thus, replay deterministically reconstructs the original allocation prefix together with its induced timing context and constraint updates.

A crucial design choice is that the replayed prefix is anchored to the \emph{scheduled} times \(\tilde{t}_j\) of the nominal plan rather than to the actual physical arrival times observed during execution. This realizes a time-triggered execution semantics that preserves the temporal structure of the original allocation. In particular, if the system reaches \(\tilde{\x}_j\) earlier than \(\tilde{t}_j\) and can remain within the corresponding satisfying region without violating the currently active invariance progress conditions, then it waits there until \(\tilde{t}_j\) before proceeding. If such waiting is infeasible, then the nominal temporal anchoring is no longer maintainable and global replanning must be triggered immediately. This scheduled-time anchoring is essential: the remaining symbolic timing relations were derived relative to the original schedule, and replacing scheduled completion times by earlier physical arrivals may invalidate the consistency of future allocations.

After reconstructing this history-consistent prefix, we inject the current execution state \((\x_t,t)\) as the effective starting node for the unfinished suffix and resume the standard allocation-and-generation pipeline for the remaining progress conditions. In this way, the method preserves already completed task progress and the temporal semantics of the nominal plan, while allowing the unfinished suffix to adapt to the observed disturbance.

\subsubsection{Fallback Protocol}
If global replanning fails within the available computational budget, we invoke a configurable fallback policy. In \emph{persistence mode}, the system continues tracking the current nominal suffix for a short horizon and retries replanning, which is useful when the deviation is transient or solver variability is significant. In \emph{abort mode}, the system terminates execution immediately and declares task failure, which is appropriate in safety-critical settings.

\begin{remark}
This hierarchy balances computational efficiency and robustness. Local replanning modifies at most a single segment and preserves the nominal schedule whenever possible. Global replanning preserves already completed task progress by reconstructing a history-consistent allocation prefix before replanning the unfinished suffix. The fallback layer ensures well-defined behavior when online replanning budgets are exhausted.
\end{remark}

\subsection{Computational Complexity Analysis}\label{sec:complexity}
The dominant test-time cost of DAG-STL lies in the symbolic search over progress allocations. Let \(m \triangleq |\PP^{\RC}|\) be the number of reachability progress conditions after preprocessing. Since the exact cost of local modules depends on implementation details, here we focus on the \emph{symbolic search-node complexity}, namely, how the number of visited allocation nodes scales with the logical structure of the task.

\paragraph{Search Complexity for a Fixed \(m\)}
For the basic allocator (Algorithm~\ref{alg:allocation}), a node at depth \(d\) corresponds to a partial assignment of \(d\) reachability progress conditions. In the worst case, all orderings of the remaining conditions remain feasible, so the number of visited nodes is bounded by
\[
\sum_{d=0}^{m} P(m,d)
=
\sum_{d=0}^{m}\frac{m!}{(m-d)!}
=
O(m!),
\]
where \(P(m,d)\) denotes the number of length-\(d\) permutations of \(m\) items. Thus, the basic planner has factorial worst-case complexity in the number of symbolic progress stages.

ARS retains the same explicit branching over the remaining progress conditions, and each logical branch can additionally produce up to \(K\) locally admissible state-time hypotheses through \texttt{MultiHypothesisAssign}. A loose worst-case upper bound on the corresponding search tree is therefore
\[
\sum_{d=0}^{m} P(m,d)\,K^d
=
O(m!K^m).
\]
The same reasoning applies to online replanning. If only \(s \le m\) reachability progress conditions remain unfinished, then the corresponding suffix search has worst-case node complexity \(O(s!)\) for the basic planner and \(O(s!K^s)\) for unbudgeted ARS-based replanning, while the implemented ARS remains capped at the search level by \(N_{\mathrm{iter}}\) and \(N_{\mathrm{seq}}\).

\paragraph{How Large Is \(m\) in Practice}
To interpret these bounds, it is crucial to understand how \(m\) depends on the STL formula. By construction, preprocessing converts each invariance progress condition into one triggering reachability condition plus one residual invariance condition, so \(m\) equals the number of original reachability progress conditions plus the number of original invariance progress conditions produced by decomposition. Let \(|V_{\varphi}|\) denote the number of nodes in the expression tree of \(\varphi\), and let \(M(\varphi)\) denote the number of reachability progress conditions induced by \(\varphi\) after preprocessing.

For a broad class of practical tasks, \(m\) remains of the same order as the expression-tree size. Conjunction simply adds the progress conditions of the two subtrees; an outer eventually operator \(\F_{[a,b]}\) only shifts existing conditions; and under our syntactic restriction on bounded-until prefixes, \(\phi\U_{[a,b]}\psi\) does not create variable-horizon replication of reachability obligations in the prefix. Consequently, when the formula does not contain outer-global eventuality structures that repeatedly replicate inner reachability obligations, one typically has
\[
m = M(\varphi) = O(|V_{\varphi}|),
\]
up to modest constant factors introduced by preprocessing. In this regime, the symbolic search dimension is governed mainly by the logical size of the specification rather than by the numerical length of its time intervals.

The situation changes for outer-global eventually patterns, typified by \(\G\F\)-type tasks. For a subformula of the form \(\G_{[a,b]}\psi\), the decomposition unfolds it into shifted copies indexed by \(k\in[a,b]\). If \(\psi\) itself contains reachability-type obligations, then these obligations are replicated across the outer interval. At a coarse level, this yields the scaling
\[
M\bigl(\G_{[a,b]}\psi\bigr)
\approx
(b-a+1)\,M(\psi),
\]
except that some constant-endpoint invariance copies can be merged and therefore do not fully replicate. Thus, without such outer-global recurrence, \(m\) is typically tied to the tree size; with \(\G\F\)-type structures, it additionally depends on interval length. This is exactly the phenomenon observed in the recurrent reachability case studies of Section~\ref{sec:casestudy}, where the two examples induce 125 and 202 reachability progress conditions, respectively.

Taken together, these observations highlight a key structural advantage of the DAG representation. Many robotic tasks are more naturally described as finite combinations of region visits, dwell requirements, ordering constraints, and safety conditions than as recurrent \(\G\F\)-type obligations repeated throughout the full horizon. For such long-horizon but milestone-like tasks, enlarging interval bounds mainly changes the timing constraints and does not necessarily enlarge the symbolic search tree. This relative insensitivity to the numerical horizon is particularly beneficial for high-frequency systems, in which even simple motion between distant regions may require many low-level time steps and therefore long STL intervals; the Maze2D and AntMaze experiments later in the paper are representative examples of this regime. At the same time, DAG-STL is not limited to it: when outer-global recurrence is present, planning becomes more demanding because \(m\) grows with interval length, but the recurrent reachability case studies show that the framework still retains meaningful capability to handle such \(\G\F\)-type tasks.

\paragraph{Why Practical Runtime Is Often Much Smaller Than the Loose Worst-Case Bound}
The conservative bounds above are rarely approached in practice because several mechanisms sharply reduce the effective search effort. At the formula level, disjunctions are removed before planning, constant-endpoint invariance copies under \(\G_{[a,b]}\) are merged whenever possible, and the syntactic restriction on bounded-until prefixes prevents variable-horizon replication inside the until prefix. At the search level, progress conditions are ordered by temporal urgency, \texttt{TimeAssign} and \texttt{MultiHypothesisAssign} apply strong local screening, \(N_{\max}\) and \(K\) cap local branching, \texttt{UpdateConstraint} immediately prunes temporally infeasible branches, and invariance conditions are activated lazily through their triggering reachability events.

Finally, both the baseline planner and ARS use depth-first exploration, which helps them reach a first complete allocation quickly. ARS then further improves practical efficiency through score-guided backjumping: after a full candidate is generated, refinement resumes from the ancestor responsible for the dominant bottleneck rather than chronologically enumerating many irrelevant sibling subtrees. As a result, the practical behavior of DAG-STL is typically much closer to a heavily pruned, budgeted search over a small symbolic skeleton than to the raw worst-case tree implied by \(O(m!)\) or \(O(m!K^m)\).

\section{Experiments}\label{sec:experiments}
In this section, we comprehensively evaluate DAG-STL. We first describe the common experimental setup, including the environments, datasets, implementation details, and shared trajectory-evaluation protocol used throughout all experiments (Section~\ref{sec:env}). Next, we present representative case studies to illustrate the end-to-end workflow and practical applicability of the framework in both navigation and complex manipulation tasks (Section~\ref{sec:casestudy}). Finally, we conduct large-scale quantitative experiments on both randomly generated STL tasks and controlled custom-built studies to evaluate DAG-STL under diverse algorithmic configurations and comparative settings (Section~\ref{sec:exp}).

\begin{figure*}[t]
    \centering
    \begin{subfigure}[c]{0.25\textwidth}
        \centering
        \includegraphics[width=\textwidth]{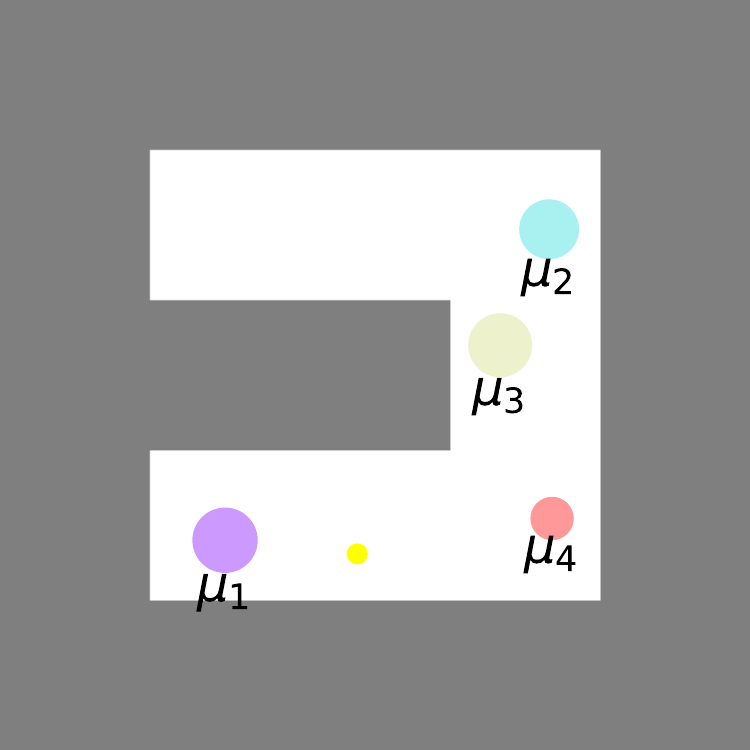}
    \end{subfigure}
    \begin{subfigure}[c]{0.25\textwidth}
        \centering
        \includegraphics[width=\textwidth]{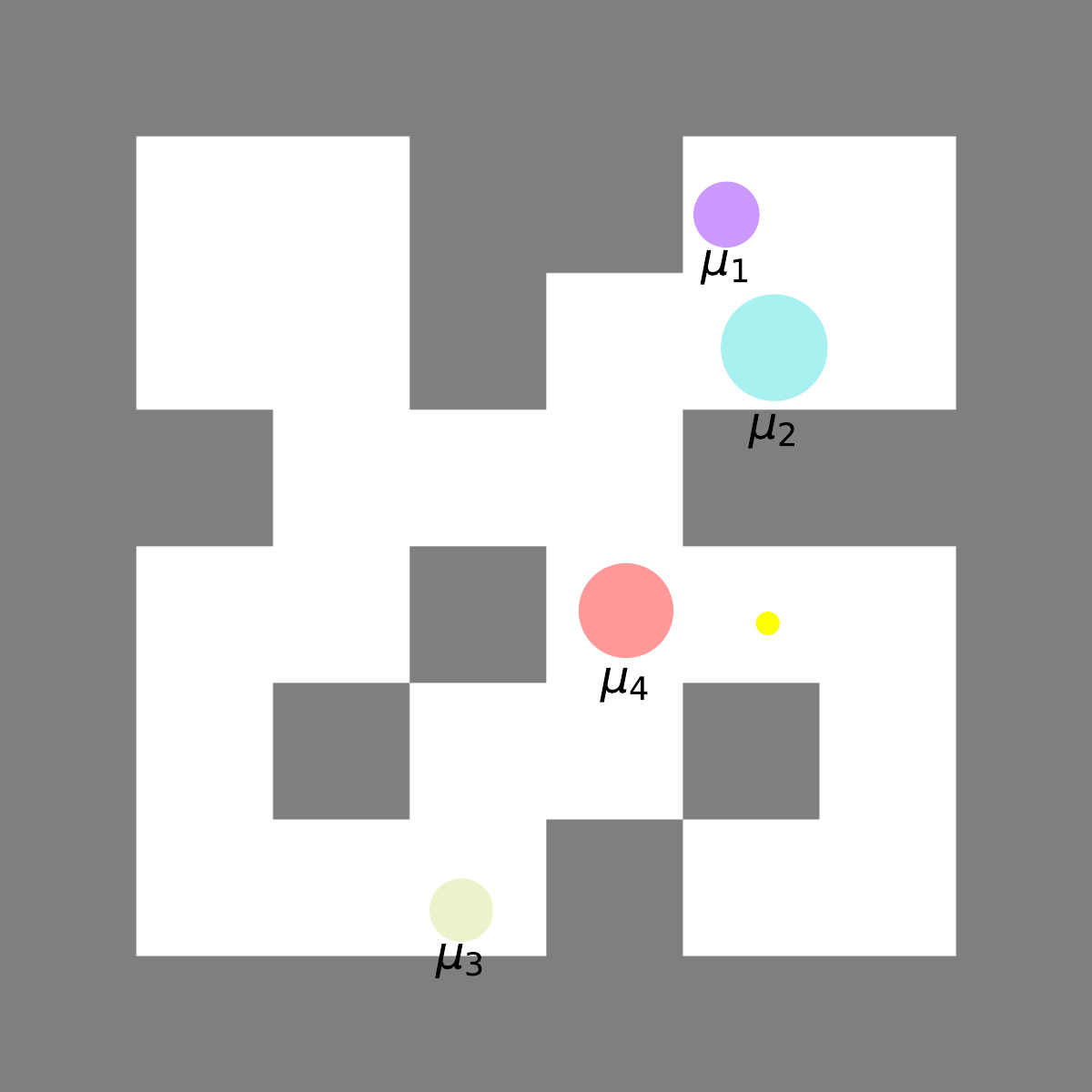}
    \end{subfigure}
    \begin{subfigure}[c]{0.25\textwidth}
        \centering
        \includegraphics[width=\textwidth]{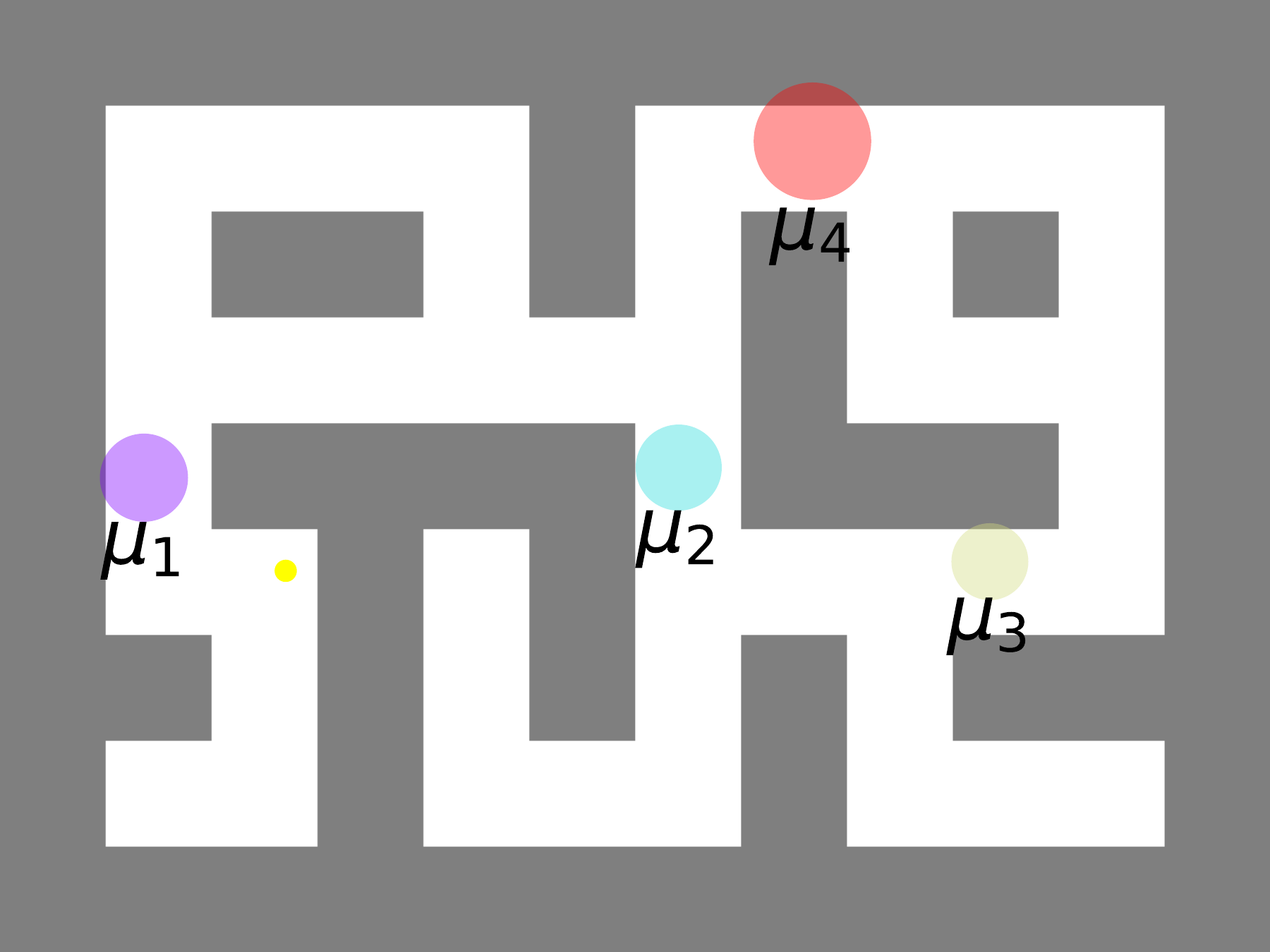}
    \end{subfigure}
    \begin{subfigure}[c]{0.25\textwidth}
        \centering
        \includegraphics[width=\textwidth]{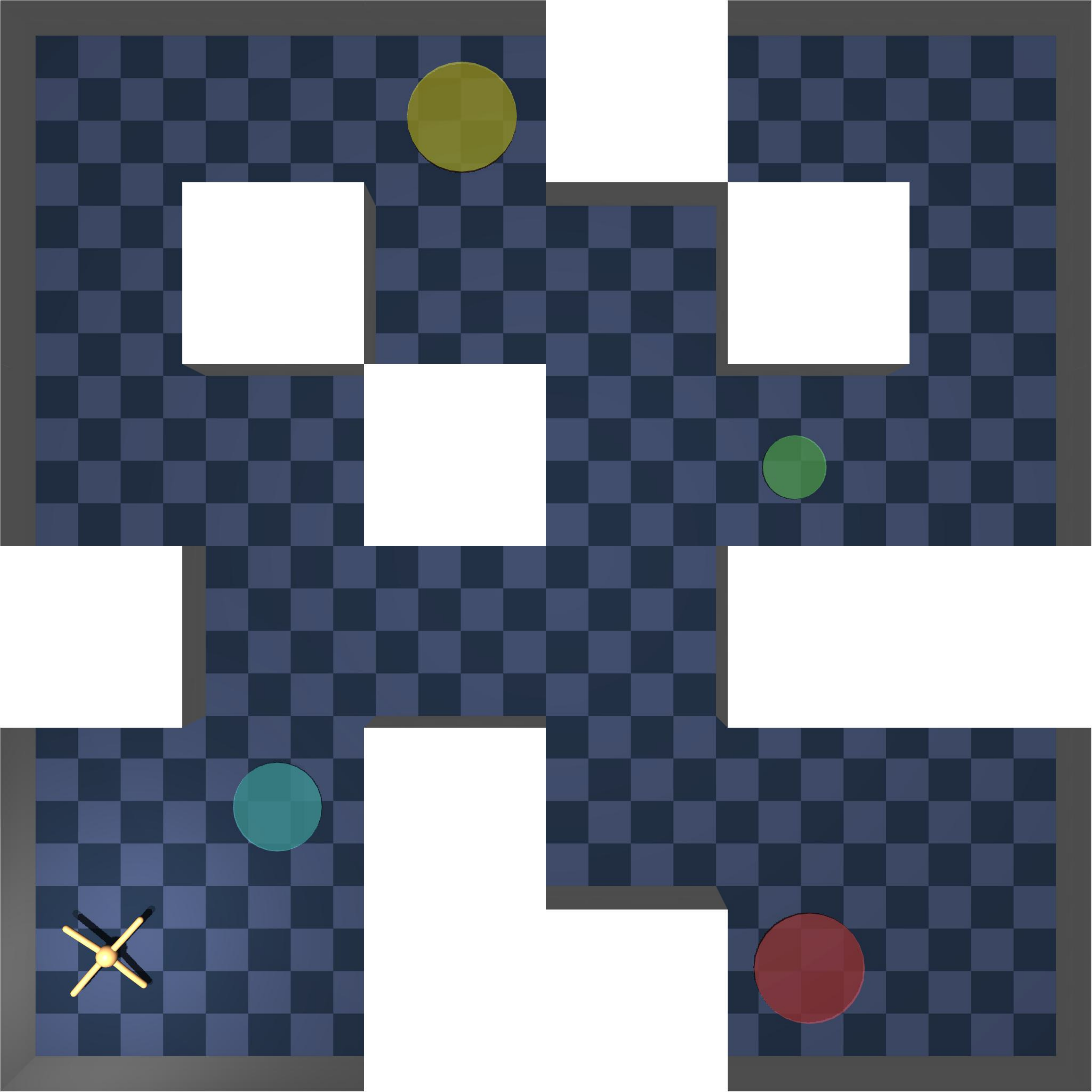}
    \end{subfigure}
    \begin{subfigure}[c]{0.25\textwidth}
        \centering
        \includegraphics[width=\textwidth]{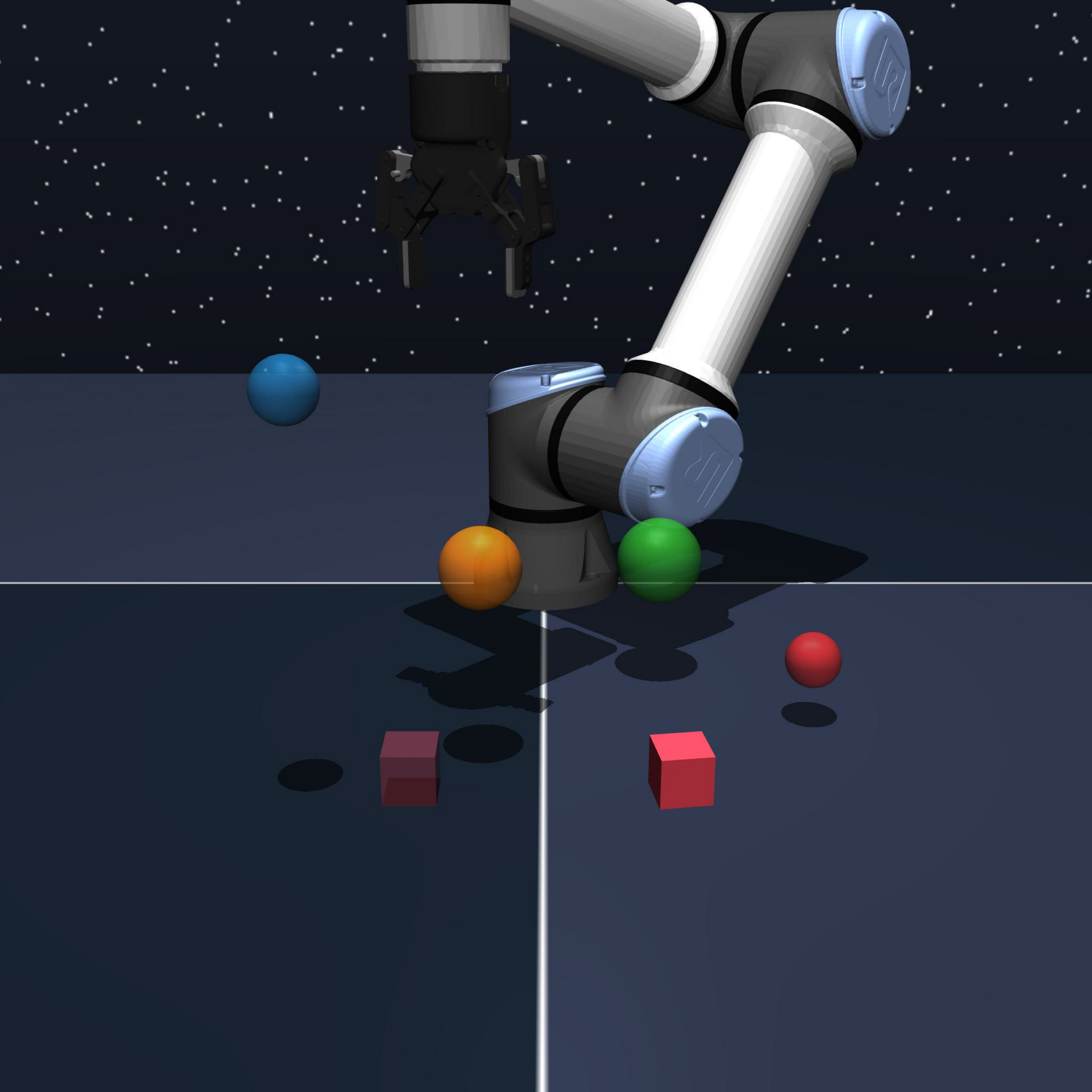}
    \end{subfigure}
    \begin{subfigure}[c]{0.25\textwidth}
        \centering
        \includegraphics[width=\textwidth]{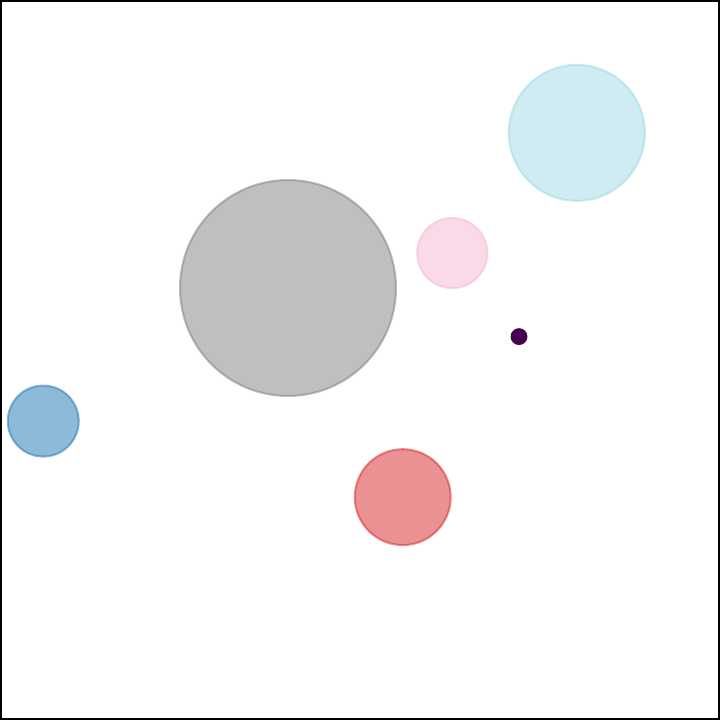}
    \end{subfigure}
    \caption{Visualization of the environments used in our experiments. From left to right and top to bottom: the three Maze2D environments (\textit{Umaze}, \textit{Medium}, \textit{Large}), the AntMaze environment, the Cube environment 
and the custom-built environment.}
    \label{fig:env}
\end{figure*}

\subsection{Experimental Setup}\label{sec:exp_setup}

\subsubsection{Environments and Datasets}\label{sec:env}
We evaluate DAG-STL on three benchmark domains spanning navigation and robotic manipulation, together with a custom-built environment designed for controlled algorithmic comparisons and supplementary stress tests. For the three benchmark domains, we use only their simulation environments and the provided task-agnostic offline trajectory datasets. We do not use the original benchmark task definitions or evaluation protocols; instead, all tasks considered in this paper are newly specified using STL. Visualizations of all environments are shown in Figure~\ref{fig:env}.

\begin{itemize}[leftmargin=*]
\item \textbf{D4RL~\cite{fu2020d4rl} Maze2D (\textit{umaze}, \textit{medium}, \textit{large})}.  
This domain consists of a point-mass agent navigating two-dimensional maze layouts of different scales and geometric complexities. The original benchmark focuses on standard goal-reaching tasks with predefined start and goal positions. In contrast, we formulate new STL tasks that require the agent to sequentially visit multiple designated regions subject to temporal ordering constraints, thereby transforming the problem from simple goal reaching into long-horizon temporally structured navigation.

\item \textbf{OGBench~\cite{ogbench_park2025} AntMaze (\textit{medium})}.  
This environment features an 8-DoF quadrupedal ant navigating a large two-dimensional maze. Compared with Maze2D, the underlying locomotion dynamics are substantially more complex. We again replace the original benchmark objective by STL specifications that require the ant to visit multiple target regions within prescribed time windows, thereby emphasizing compositionality, temporal coordination, and long-horizon planning under complex dynamics.

\item \textbf{OGBench Cube}.  
This domain is built on a 6-DoF UR5e robotic arm originally designed for dexterous object manipulation, such as cube rearrangement. For the large-scale quantitative experiments, we abstract away explicit grasping and object-contact reasoning, and define STL tasks directly on the Cartesian position of the end-effector. Concretely, the end-effector is required to reach a sequence of spatial target regions within designated time intervals, resulting in an STL-constrained waypoint-tracking problem. This abstraction enables systematic large-scale evaluation, while the manipulation case study in Section~\ref{sec:case_manipulation} demonstrates that the same framework also extends naturally to genuine object-interaction tasks.

\item \textbf{Custom Environment}.  
The benchmark domains above involve complex maps and hard-to-model dynamics, making direct comparison with classical model-based STL planners difficult or intractable. To enable a controlled comparison with an optimization-based reference and to isolate the behavior of our progress allocation module, we construct a custom two-dimensional environment consisting of a bounded square workspace with a central circular obstacle and simple double-integrator dynamics. In this setting, we compare DAG-STL against an optimization-based STL planner implemented in \texttt{stlpy}~\cite{kurtz2022mixed}, following the smooth robustness formulation of~\cite{gilpin2020smooth}. The baseline has full access to the analytical system dynamics and obstacle geometry, whereas our method uses only an offline trajectory dataset. For training, we generate around 90{,}000 collision-free trajectories by solving simple single reach-avoid problems with the optimization-based planner and then use these trajectories as task-agnostic offline data for our models. This controlled environment therefore serves as a testbed for evaluating how closely the data-driven planner can approach a full-model optimization-based reference while remaining within the offline unknown-dynamics setting.
\end{itemize}

\subsubsection{Implementation Details}\label{sec:implementation}
All experiments are implemented in Python and are conducted on a workstation running Ubuntu 22.04, equipped with an Intel i7-13700K CPU, an NVIDIA RTX 4090 GPU, and 64\,GB RAM. DAG-STL is realized as a modular framework, and several components in the pipeline admit multiple possible instantiations. In this paper, we adopt relatively simple and standard module choices so that the empirical evaluation focuses on the overall framework rather than on aggressive tuning of any individual component.

The diffusion-based trajectory generator is implemented based on the \texttt{Diffuser} framework of~\cite{janner2022planning}, with a U-Net backbone. Relative to the original implementation, we modify the training procedure by introducing the variable-horizon training strategy described in Section~\ref{sec:constrained_gen}~\cite{liu2025vh} to improve trajectory-length generalization, and adapt the trajectory dimensionality to match the state representation of each environment. The reachability-time predictor is implemented as a diffusion model with an MLP backbone. For each environment, both the trajectory generator and the time predictor are trained on the same corresponding offline dataset, and all testing results are obtained using the same trained models without any task-specific finetuning or retraining.

For execution, all environments follow the time-synchronous control protocol described in Section~\ref{sec:action_control}, in which each planning step corresponds to a fixed number \(k\) of control updates. In Maze2D and the custom-built environment, planning is performed directly in the low-dimensional state space, and execution is carried out by a PD controller with \(k=1\). In Cube and AntMaze, by contrast, planning is performed in a lower-dimensional task space and the planned trajectory is converted into executable actions at runtime. For Cube, we plan in the Cartesian end-effector space and use a PD controller with \(k=1\). For AntMaze, we plan in the \((x,y)\) task space and train an inverse dynamics model that takes as input the full current state together with the next \((x,y)\) target position and outputs the corresponding 8-dimensional action. Due to the greater difficulty of this inverse mapping, we use \(k=2\) in AntMaze. These environment-specific execution interfaces are used only as lightweight realization layers, and all compared methods within the same environment share the same rollout protocol, synchronization ratio, controller or inverse model, and evaluation procedure.

Unless otherwise specified, all compared methods are evaluated on the same hardware platform, and all reported planning times are measured using the same execution protocol. Additional hyperparameter settings and other environment-specific implementation details are provided in Appendix~\ref{apx:implementation}.

\subsubsection{Trajectory Resolution and Robustness Evaluation}\label{sec:robustness}
In all experiments, we use STL robustness as the common post-hoc criterion for determining whether a planned or executed trajectory satisfies the task specification. In particular, a trajectory is counted as successful if its STL robustness is nonnegative. The robustness value is \emph{not} used as a module of DAG-STL itself; it is introduced solely for experimental evaluation. Robustness values are computed using the open-source library \texttt{stlpy}~\cite{kurtz2022mixed}, which implements quantitative semantics for STL.

In the continuous-control environments considered in this paper, the generated and executed trajectories can become very long for complex long-horizon STL tasks. In our experiments, we found that directly evaluating STL robustness on such full-resolution trajectories using existing tools is often computationally prohibitive, and in many cases cannot be completed within an acceptable runtime budget.

To address this issue, we introduce a trajectory resolution factor \(\eta \in \mathbb{Z}_{+}\). In our implementation, one planning-scale time step corresponds to \(\eta\) trajectory-level steps during trajectory generation and execution. The STL specification, progress allocation, and waypoint timestamps all remain defined on the original planning scale. When computing robustness, we uniformly downsample the executed trajectory by the same factor \(\eta\) and evaluate STL robustness on the resulting planning-scale state sequence. In this way, \(\eta\) affects only the implementation resolution of trajectory realization and post-hoc robustness evaluation, without changing the planning-level STL task itself.

We emphasize that \(\eta\) is distinct from the synchronization ratio \(k\) introduced in Section~\ref{sec:action_control}. The former controls the trajectory resolution used for execution and post-hoc evaluation, whereas the latter specifies how planning steps are synchronized with low-level control updates.

\subsection{Case Studies}\label{sec:casestudy}
\subsubsection{Navigation Tasks in the Maze2D Environment}
To illustrate how DAG-STL handles temporally structured navigation problems, we present a set of representative case studies in the Maze2D environment~\cite{fu2020d4rl}. The first example demonstrates the end-to-end decomposition--allocation--generation workflow on a simple sequential visit task, the second considers a more challenging hybrid task involving bounded-until structure, temporal nesting, and invariance constraints, the third examines recurrent temporal obligations with outer-global eventually operators, and the final example highlights how the proposed Anytime Refinement Search can identify a more dynamically consistent solution under the same STL task. In all of these cases, all learned modules in the planner are trained solely on the STL task-agnostic offline trajectory dataset from D4RL, without access to the environment map or analytical system dynamics.

\paragraph{Sequential Visit Task}
We first consider a sequential navigation task in the Large Maze2D layout, where the agent is required to visit three target regions in order while avoiding two obstacle regions throughout the episode. The corresponding STL specification is
\begin{equation}\label{formula:case_study1}
\F_{[0,40]}\bigl(\mu_1 \wedge \F_{[0,40]}(\mu_2 \wedge \F_{[0,40]}\mu_3)\bigr)\wedge\G_{[0,120]}(\neg\mu_4 \wedge \neg\mu_5),    
\end{equation}
where each predicate \(\mu_i\) represents membership in a circular region, i.e.,
\[
\mu_i:\; r_i-\|\x-\bs{c}_i\|_2\ge 0,
\]
with center \(\bs{c}_i\) and radius \(r_i\). Intuitively, the task requires the agent to reach \(\mu_1\), \(\mu_2\), and \(\mu_3\) sequentially within prescribed temporal windows, while avoiding \(\mu_4\) and \(\mu_5\) at all times.

\begin{table}[tbp]
\centering
\begin{tabular}{cp{6cm}}
\toprule
\textbf{Type} & \textbf{Details} \\
\midrule
$\PP^{\RC}$ & 
$\RC(\lambda_1,\lambda_1,\mu_1)$, $\RC(\lambda_1+\lambda_2,\lambda_1+\lambda_2,\mu_2)$, \newline
$\RC(\lambda_1+\lambda_2+\lambda_3,\lambda_1+\lambda_2+\lambda_3,\mu_3)$, \newline
$\RC(0,0,\neg\mu_4)$, $\RC(0,0,\neg\mu_5)$ \\
\midrule
$\PP^{\IC}$ & 
$\IC(1,120,\neg \mu_4)$, $\IC(1,120,\neg \mu_5)$ \\
\midrule
$\TT$ & 
$\lambda_1 \in [0,40]$, $\lambda_2 \in [0,40]$, $\lambda_3 \in [0,40]$ \\
\bottomrule
\end{tabular}
\caption{Progress decomposition of the sequential visit task.}
\label{tab:constraints1}
\end{table}

Following Section~\ref{sec:STLdecomposition}, the formula is decomposed into progress conditions and time-variable constraints, as shown in Table~\ref{tab:constraints1}. Note that, after the preprocessing step in Section~\ref{sec:allocation}, the avoid requirement \(\G_{[0,120]}(\neg\mu_4\wedge\neg\mu_5)\) gives rise to reachability triggers together with residual invariance progress conditions. The progress allocation module then assigns both waypoint states and their corresponding satisfaction times to the reachability progress conditions, thereby producing a timed waypoint sequence; the assigned times are shown by the numbers next to the start state and regions \(\mu_1\), \(\mu_2\), and \(\mu_3\) in Figure~\ref{fig:case_1}. The trajectory generation module subsequently realizes this planning-scale skeleton by generating and stitching higher-resolution trajectory segments between adjacent waypoints while enforcing the active invariance conditions. The planned trajectory is shown in the left panel of Figure~\ref{fig:case_1}, and the corresponding executed trajectory under the control protocol of Section~\ref{sec:action_control} is shown on the right.

Using the common robustness evaluation protocol described in Section~\ref{sec:robustness}, the planned and executed trajectories obtain STL robustness values of \(0.180\) and \(0.115\), respectively. Since both values are nonnegative, both trajectories satisfy the specification. This example provides a direct visualization of the full DAG-STL workflow, from symbolic decomposition and timed waypoint allocation to trajectory generation and execution.

\begin{figure}[tbp]
    \centering
    \begin{subfigure}[b]{0.235\textwidth}
        \centering
        \includegraphics[width=\textwidth]{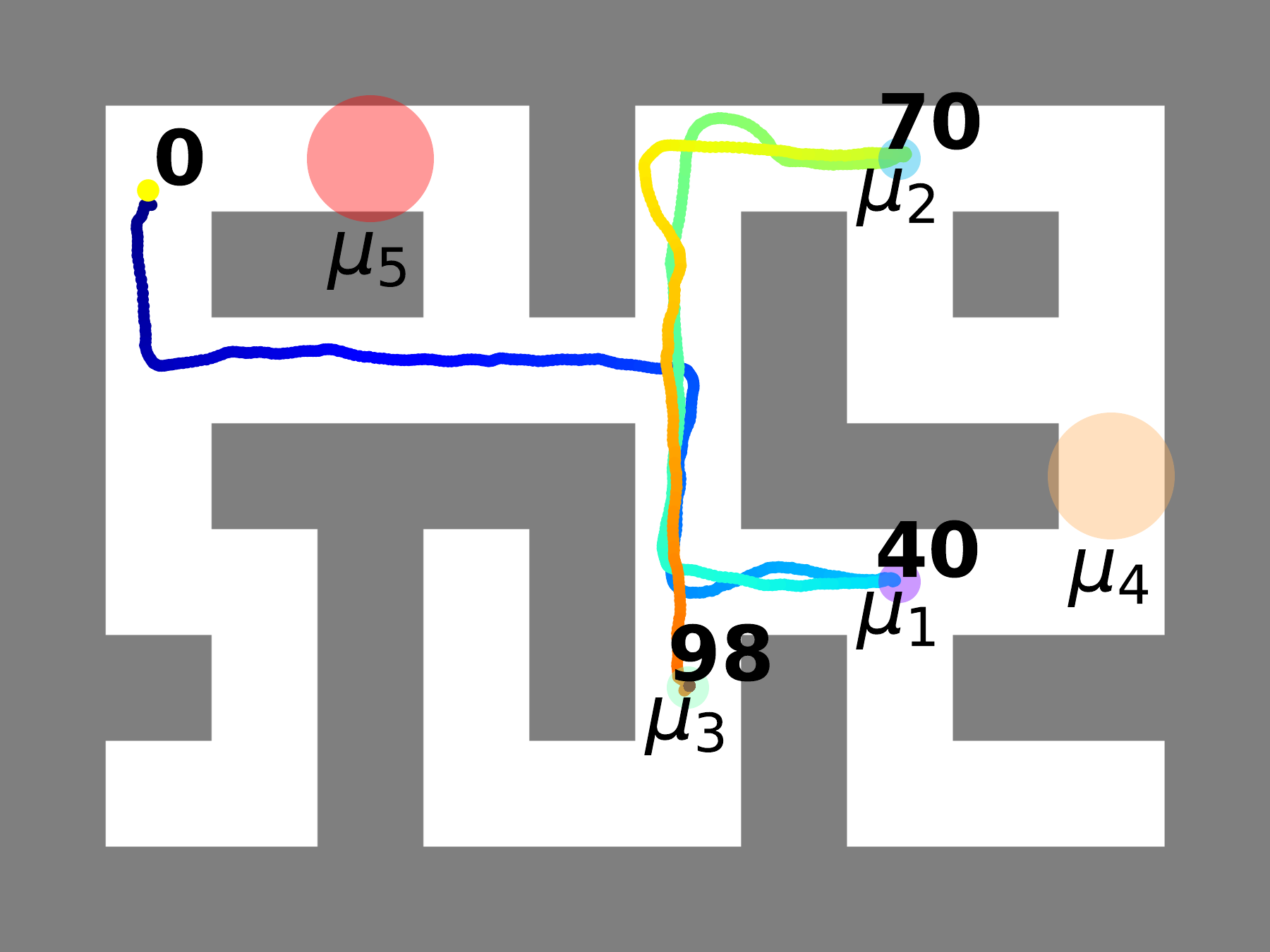}
    \end{subfigure}
    \begin{subfigure}[b]{0.235\textwidth}
        \centering
        \includegraphics[width=\textwidth]{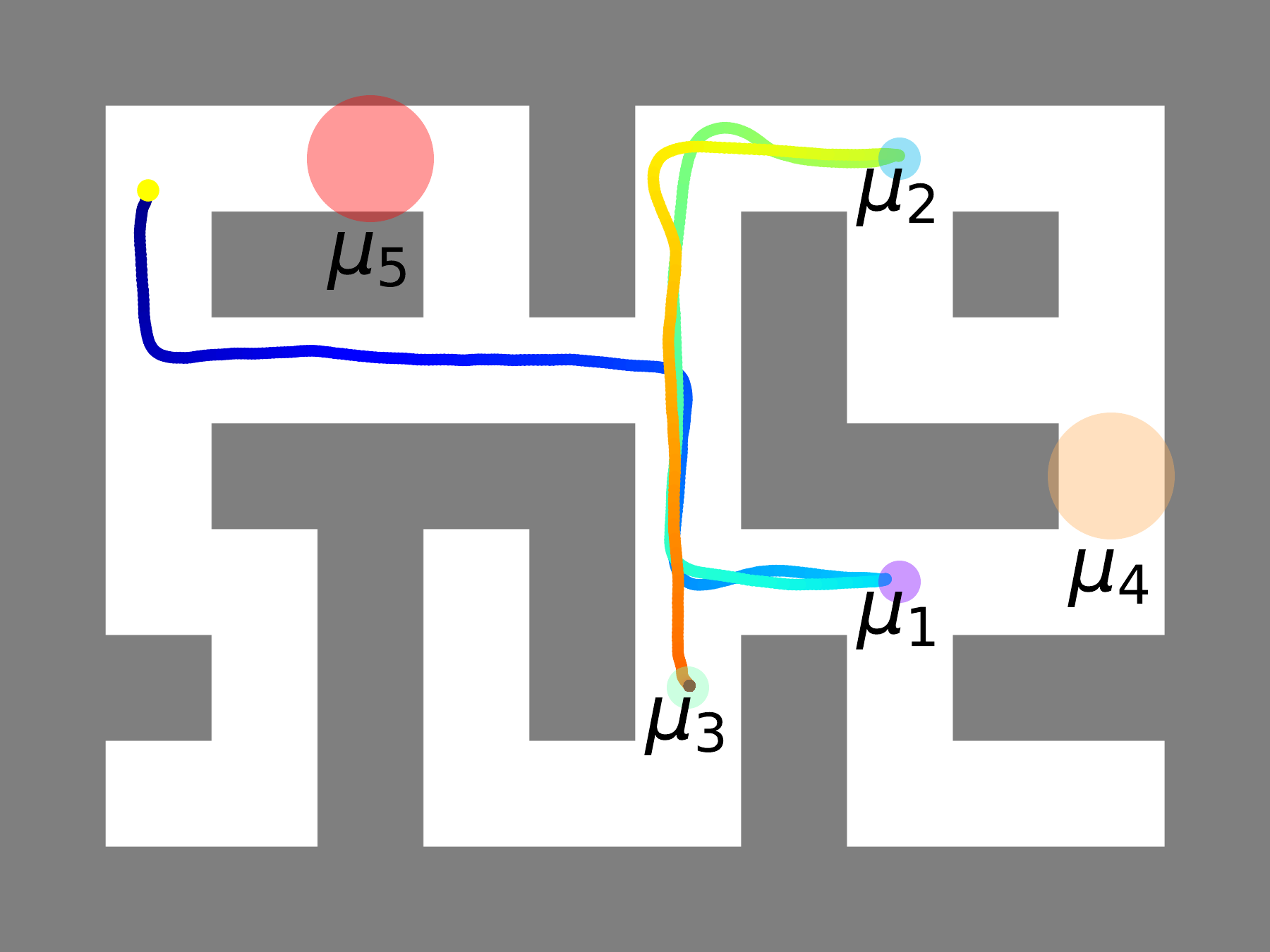}
    \end{subfigure}
    \caption{Planned trajectory (left) and executed trajectory (right) for the sequential visit task~\eqref{formula:case_study1}. The numbers next to the start point and target regions indicate the satisfaction times assigned by the progress allocation module on the planning time scale. The plotted trajectories are shown at the finer trajectory resolution (\(\eta=8\)) used for generation and execution; see Section~\ref{sec:robustness} for details.}
    \label{fig:case_1}
\end{figure}

\paragraph{Hybrid Task}
We next consider a more complex hybrid task in the Medium Maze2D layout. The scenario can be interpreted as an autonomous inspection-and-delivery mission: the agent must remain in corridor \(\mu_1\) until reaching checkpoint \(\mu_2\), then stay in buffer region \(\mu_4\) for 5 steps before entering assembly region \(\mu_3\), where it must again remain for 5 steps. Finally, it must reach exit region \(\mu_5\), while avoiding restricted region \(\mu_6\) throughout the mission. This task combines reachability, invariance, and bounded-until constraints within a single specification.

The corresponding STL formula is
\begin{equation}\label{formula:hybrid_task}
\Phi = \varphi_1 \wedge \varphi_2 \wedge \varphi_3 \wedge \varphi_4 \wedge \varphi_5,
\end{equation}
where
\[
\begin{aligned}
\varphi_1 &= \mu_1 \U_{[0,30]} \mu_2,\quad
\varphi_2 = \F_{[0,100]}\G_{[0,5]}\mu_3,\\
\varphi_3 &= \neg \mu_3 \U_{[0,100]}(\G_{[0,5]}\mu_4),\\
\varphi_4 &= \F_{[0,100]}\mu_5,\quad
\varphi_5 = \G_{[0,105]}\neg \mu_6.
\end{aligned}
\]

This example is substantially more challenging than the previous one, since it involves both nested temporal structure and nontrivial ordering constraints induced by the bounded-until operators. The progress allocation module first resolves these coupled temporal dependencies by assigning waypoint states and their corresponding satisfaction times, thereby producing a planning-scale timed waypoint sequence. The trajectory generation module then realizes this timed waypoint skeleton by generating and stitching higher-resolution trajectory segments between adjacent waypoints while enforcing the active invariance conditions. The resulting planned and executed trajectories are shown in Figure~\ref{fig:case_2}, where the assigned satisfaction times are again marked next to the corresponding regions.

Using the same post-hoc robustness evaluation protocol as above, the planned and executed trajectories obtain robustness values of \(0.150\) and \(0.079\), respectively. Both are nonnegative, indicating that the generated plan and the resulting execution satisfy the STL specification. This case highlights that DAG-STL can coordinate multiple reachability and invariance requirements simultaneously, and remains effective even when the task contains nested temporal dependencies and sustained dwell constraints.

\begin{figure}[tbp]
    \centering
    \begin{subfigure}[b]{0.235\textwidth}
        \centering
        \includegraphics[width=\textwidth]{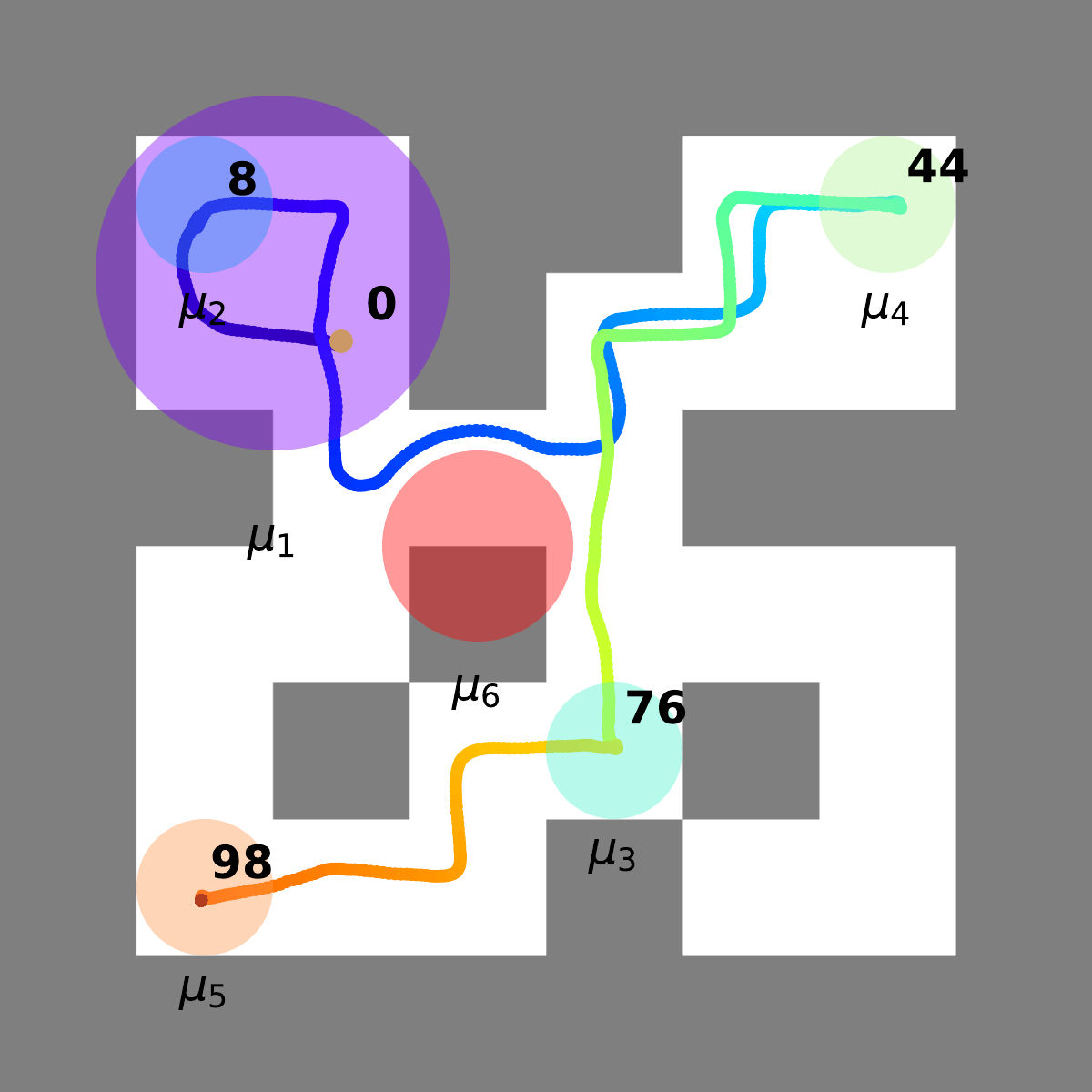}
    \end{subfigure}
    \begin{subfigure}[b]{0.235\textwidth}
        \centering
        \includegraphics[width=\textwidth]{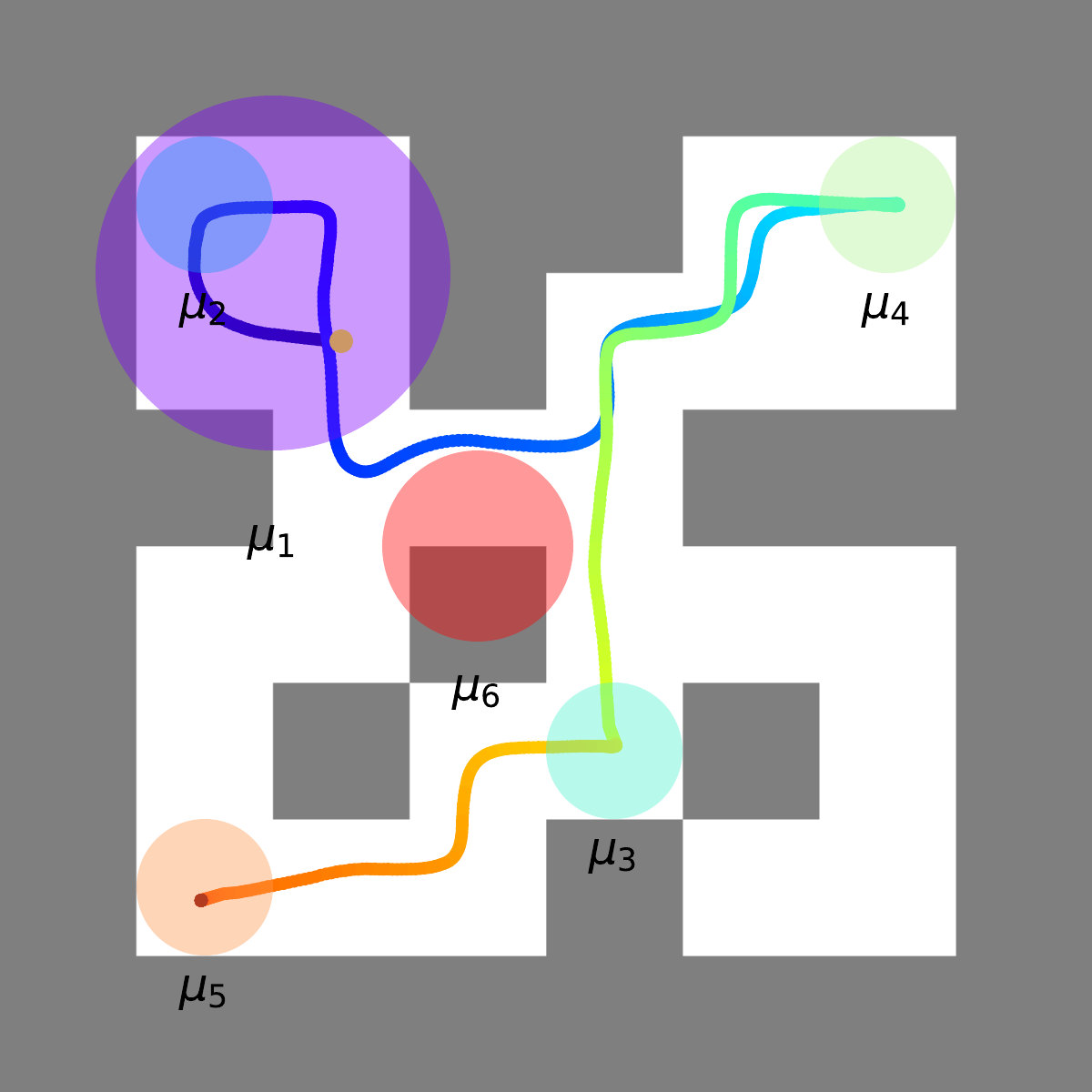}
    \end{subfigure}
    \caption{Planned trajectory (left) and executed trajectory (right) for the hybrid task~\eqref{formula:hybrid_task}. The numbers next to the start point and target regions indicate the satisfaction times assigned by the progress allocation module on the planning time scale. The plotted trajectories are shown at the finer trajectory resolution (\(\eta=8\)) used for generation and execution.}
    \label{fig:case_2}
\end{figure}

\paragraph{Recurrent Reachability Cases}
The previous case studies focus on STL tasks that can be naturally interpreted as finite collections of temporally ordered milestones. We next consider two additional cases with outer-global eventually constraints to illustrate that the proposed framework also supports recurrent temporal obligations of this type at the level of task decomposition and allocation. The main practical difficulty of such tasks is not their semantic incompatibility with DAG-STL, but the fact that, after decomposition, they can induce a large number of progress conditions and time variables, which substantially increases the complexity of planning.

The first task is
\begin{equation}\label{formula:case_study3_1}
\begin{aligned}
\G_{[0,120]}(\F_{[0,40]}\mu_1)\ \wedge\ \F_{[0,100]}(\mu_2 \wedge \F_{[0,60]}\mu_3)\\
\wedge\ \F_{[0,150]}\G_{[0,3]}\mu_4\ \wedge\ \G_{[0,160]}(\neg \mu_5).
\end{aligned}
\end{equation}
This specification combines a recurrent reachability requirement of the form \(\G\F\mu_1\) with a one-shot sequential reachability task, a dwell requirement, and a global avoidance constraint. In this case, the decomposition produces 125 reachability progress conditions, 2 invariance progress conditions, and 124 time variables. Despite the recurrent obligation, the planner is still able to find a valid trajectory within approximately \(6\) seconds, and the executed trajectory achieves STL robustness \(0.081\), confirming successful task satisfaction.

The second task is
\begin{equation}\label{formula:case_study3_2}
\G_{[0,100]}\bigl(\F_{[0,30]}(\mu_1 \wedge \F_{[0,30]}\mu_2)\bigr),
\end{equation}
which places a more deeply nested eventually structure inside the outer global operator. Compared with the previous case, this formula induces a substantially denser recurrent obligation structure after decomposition, yielding 202 reachability progress conditions and 202 time variables. Although the framework can still solve this instance successfully, the planning time increases substantially to approximately \(27\) seconds. The executed trajectory attains STL robustness \(0.192\), which is again nonnegative and therefore satisfies the specification.

Figure~\ref{fig:case_3} visualizes the two cases. The first two rows show the timed waypoint allocations on the planning horizon for the two tasks, illustrating how the allocation module schedules repeated and nested visitation requirements over time. The third row shows the corresponding executed trajectories. Taken together, these two cases show that nontrivial STL tasks with outer-global eventually obligations can be successfully solved within the proposed framework. They also provide an empirical indication of how recurrent temporal complexity affects the practical difficulty of planning in the current setting.

\begin{figure}[tbp]
    \centering
    \begin{subfigure}[b]{0.48\textwidth}
        \centering
        \includegraphics[width=\textwidth]{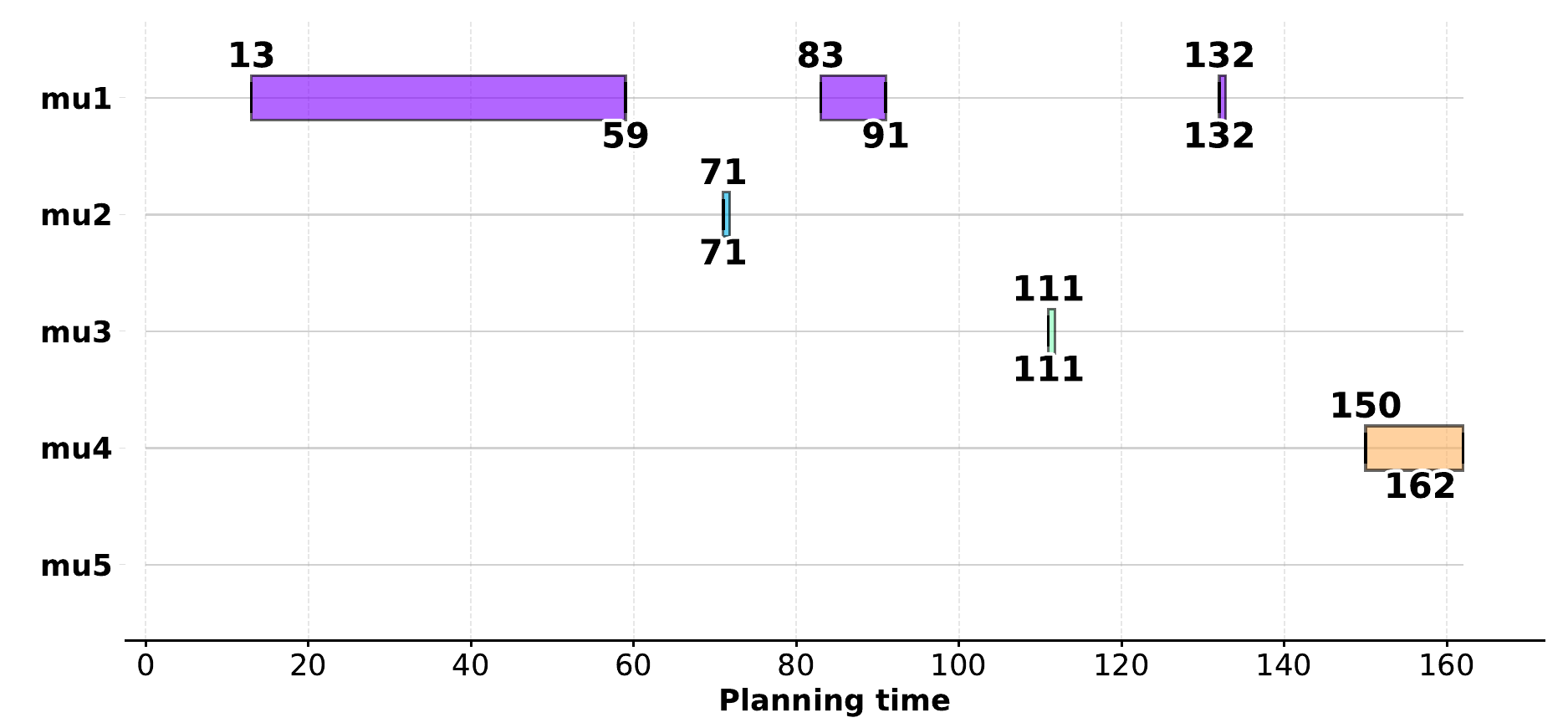}
    \end{subfigure}
    \begin{subfigure}[b]{0.48\textwidth}
        \centering
        \includegraphics[width=\textwidth]{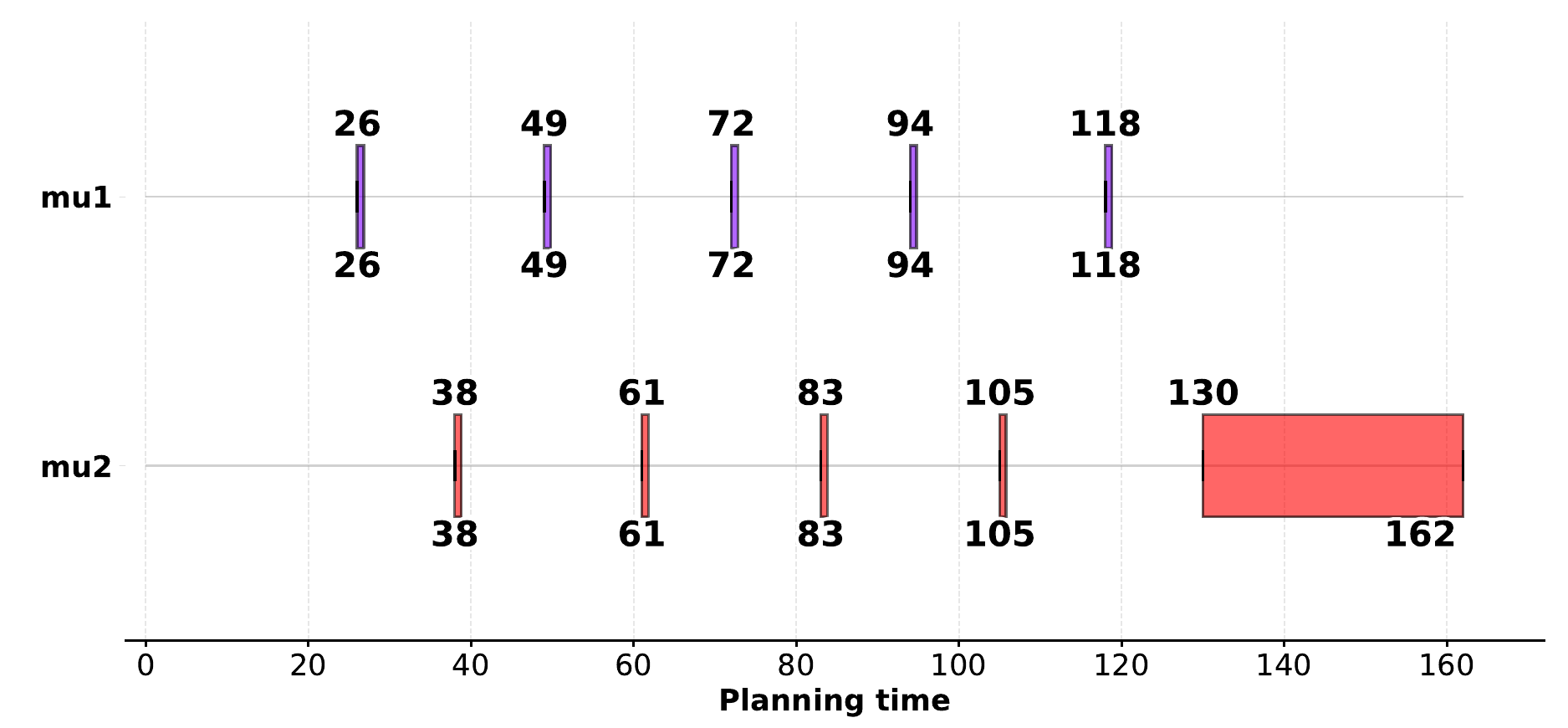}
    \end{subfigure}
    \begin{subfigure}[b]{0.235\textwidth}
        \centering
        \includegraphics[width=\textwidth]{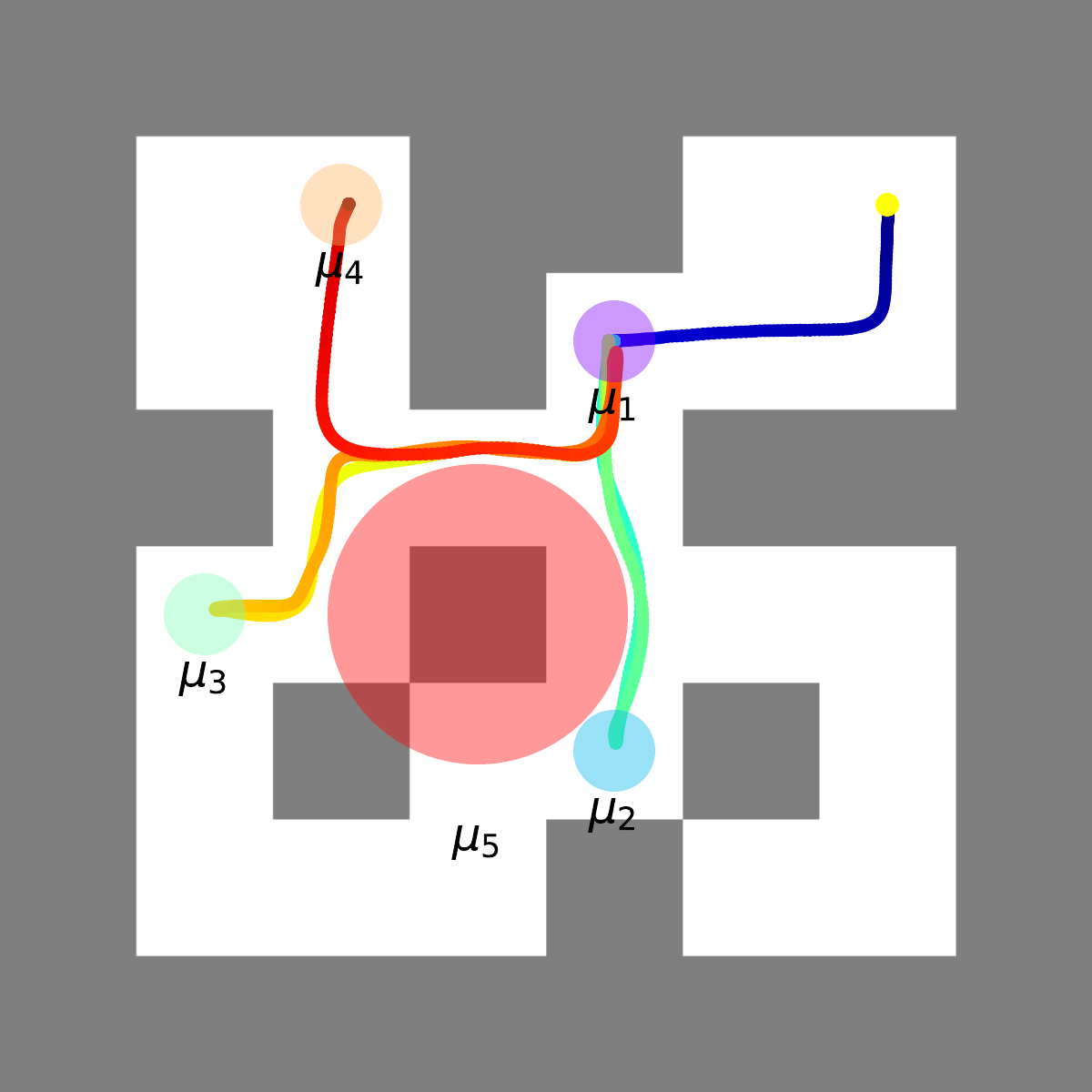}
    \end{subfigure}
    \begin{subfigure}[b]{0.235\textwidth}
        \centering
        \includegraphics[width=\textwidth]{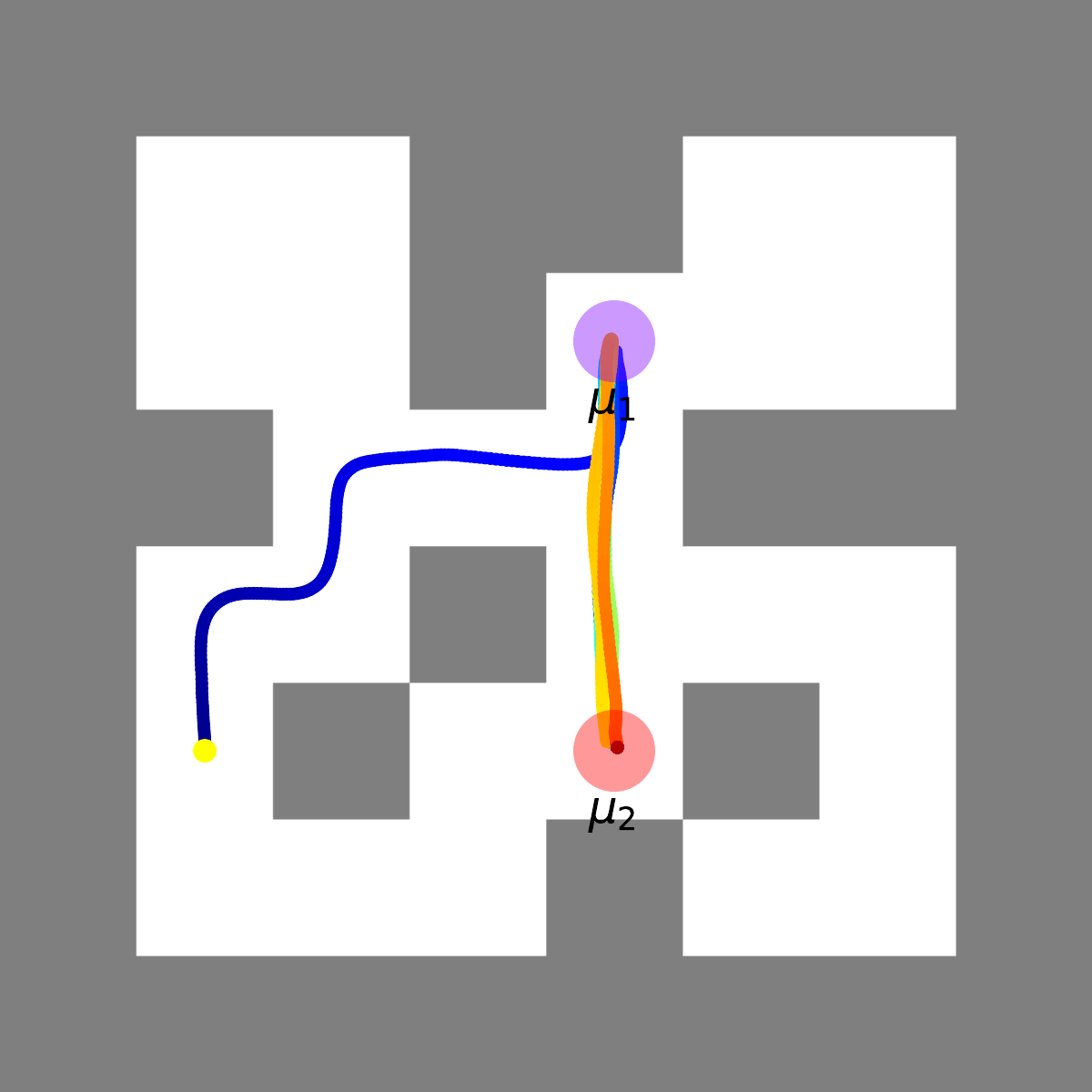}
    \end{subfigure}
    \caption{Two recurrent-reachability cases with outer-global eventually constraints. The first two rows show the timed waypoint allocations on the planning time scale for the two tasks, illustrating how the progress allocation module schedules repeated and nested visitation obligations over the horizon. The third row shows the corresponding executed trajectories at the finer execution resolution. The left-hand task corresponds to the mixed recurrent case~\eqref{formula:case_study3_1}, while the right-hand task corresponds to the more deeply nested recurrent case~\eqref{formula:case_study3_2}.}
    \label{fig:case_3}
\end{figure}

\paragraph{Refinement Case: Allocation Improvement with ARS}
The previous case studies focus on the basic decomposition--allocation--generation pipeline. We finally present a case in the Umaze layout to illustrate how ARS finds a better plan by exploring a richer allocation search space under the guidance of the dynamic consistency metric. We consider the task
\begin{equation}\label{formula:case_study4}
\F_{[0,20]}\bigl(\mu_1 \wedge \F_{[0,20]}\mu_2\bigr)\;\wedge\; \G_{[0,40]}(\neg \mu_3),
\end{equation}
which requires the agent to first reach region \(\mu_1\) within 20 steps, then reach region \(\mu_2\) within an additional 20 steps, while avoiding region \(\mu_3\) throughout the execution horizon.

Figure~\ref{fig:case_ars_compare} compares the trajectory returned by the basic planning backbone with the one returned by ARS. Although the basic planner finds an STL-feasible solution, the resulting trajectory is dynamically brittle. In particular, the first segment is assigned a relatively short time budget and therefore passes very close to the maze wall, leaving little clearance for reliable tracking. At the same time, both segments stay close to the avoidance region \(\mu_3\), and the basic plan attempts to bend around it using a locally unnatural arc-shaped motion that is weakly supported by the offline trajectory distribution and therefore difficult for a downstream controller to follow robustly.

By contrast, ARS returns a different timed waypoint allocation and associated trajectory realization. It assigns a longer time budget to the first segment, allowing the trajectory to more safely bypass the wall. Meanwhile, under the guidance of the dynamic consistency metric, ARS selects a trajectory that is more consistent with the offline-supported dynamics: instead of producing an unnecessary curved detour near \(\mu_3\), it chooses a much smoother path that passes almost straight along the boundary of the avoidance region.

This example highlights the practical role of ARS. It does not simply smooth a fixed nominal trajectory after planning; rather, it performs a more effective search over an expanded allocation space, uses the trajectory-level consistency evaluation to localize the weakest segment of the current solution, and then revises the corresponding allocation and trajectory jointly. The final output is therefore the best STL-feasible trajectory found within the available search budget, balancing logical correctness with execution-oriented dynamic consistency.

\begin{figure}[tbp]
    \centering
    \begin{subfigure}[b]{0.235\textwidth}
        \centering
        \includegraphics[width=\textwidth]{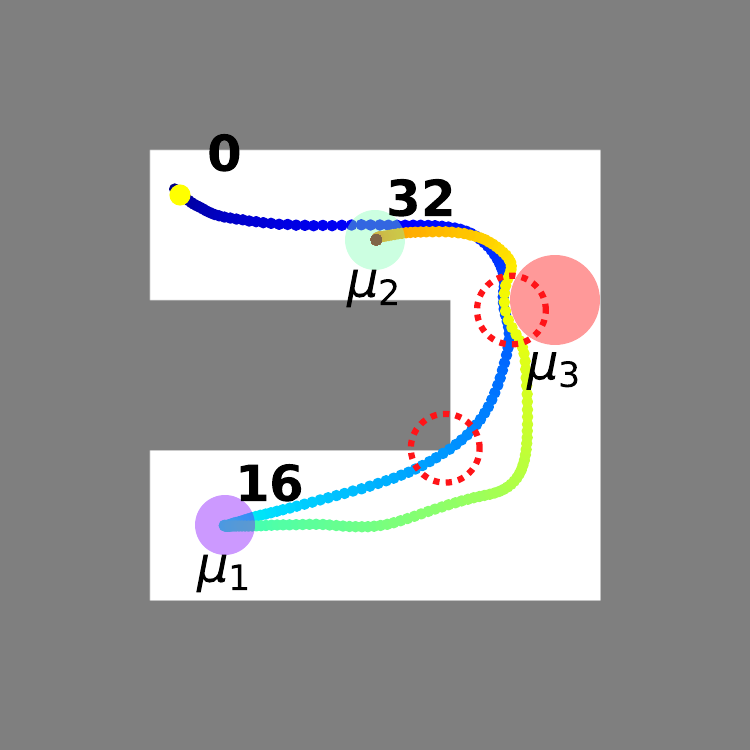}
    \end{subfigure}
    \begin{subfigure}[b]{0.235\textwidth}
        \centering
        \includegraphics[width=\textwidth]{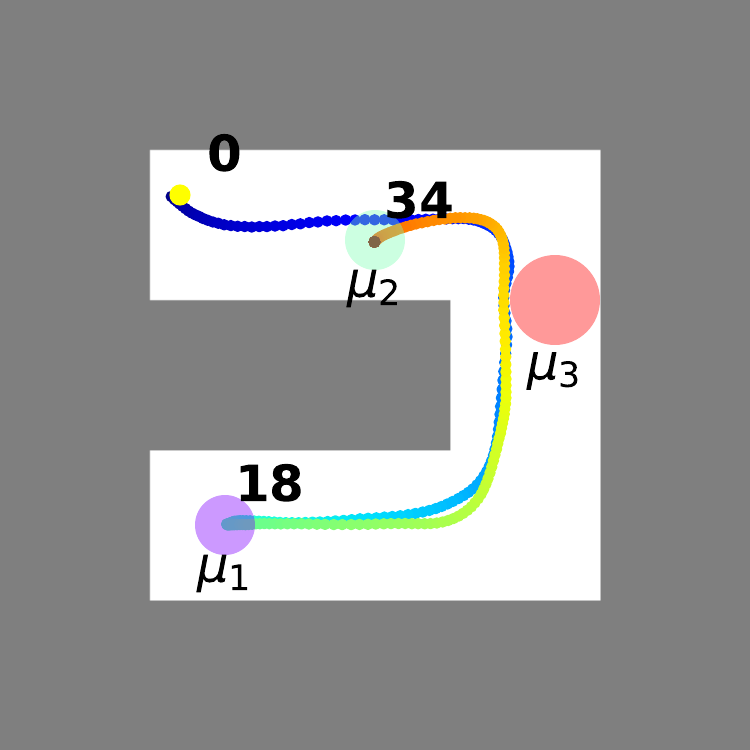}
    \end{subfigure}
    \caption{Comparison between the trajectory returned by the basic planning backbone (left) and the trajectory returned by ARS (right) for task~\eqref{formula:case_study4} in the Umaze layout. ARS assigns a longer time budget to the first segment and, under the dynamic consistency metric, selects a more natural trajectory realization that passes more safely around the wall and more directly along the boundary of the avoidance region \(\mu_3\). The plotted trajectories are shown at the finer trajectory resolution (\(\eta=8\)) used for generation and execution.}
    \label{fig:case_ars_compare}
\end{figure}

\subsubsection{Manipulation Tasks of a 6-DoF Robot Arm}\label{sec:case_manipulation}

To further demonstrate that DAG-STL extends beyond navigation, we consider manipulation tasks in the ``Cube'' scenario from OGBench~\cite{ogbench_park2025}. As illustrated in Figure~\ref{fig:case_cube_demo}, a 6-DoF UR5e robotic arm is required to rearrange cubes from their initial configurations to designated goal configurations. This case study provides representative examples of whether the proposed planning framework, trained only on single-cube task-agnostic data, can generalize in a zero-shot manner to more complex multi-stage and multi-object manipulation behaviors.

\begin{figure}[htbp]
    \centering
    \begin{subfigure}[b]{0.235\textwidth}
        \centering
        \includegraphics[width=\textwidth]{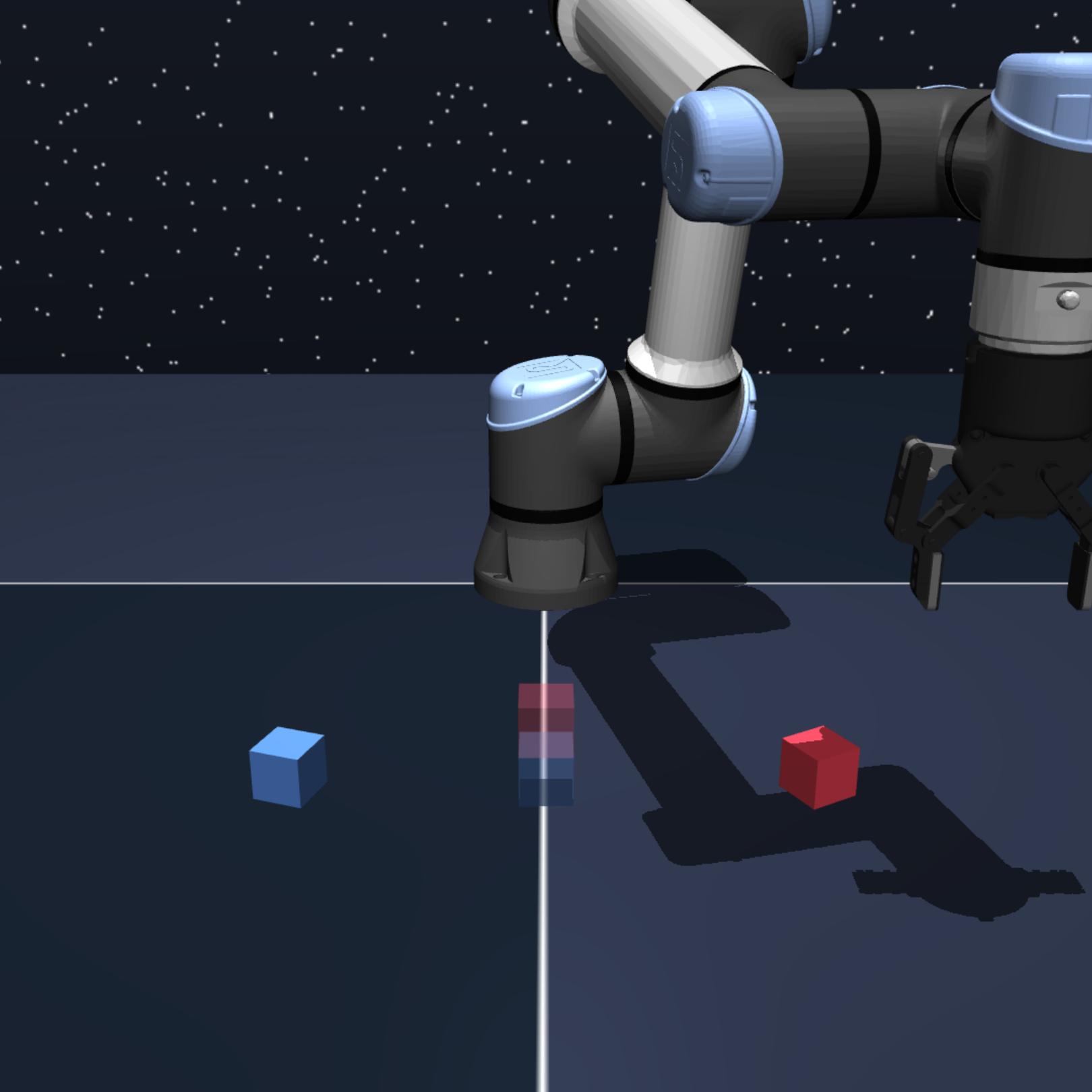}
    \end{subfigure}
    \begin{subfigure}[b]{0.235\textwidth}
        \centering
        \includegraphics[width=\textwidth]{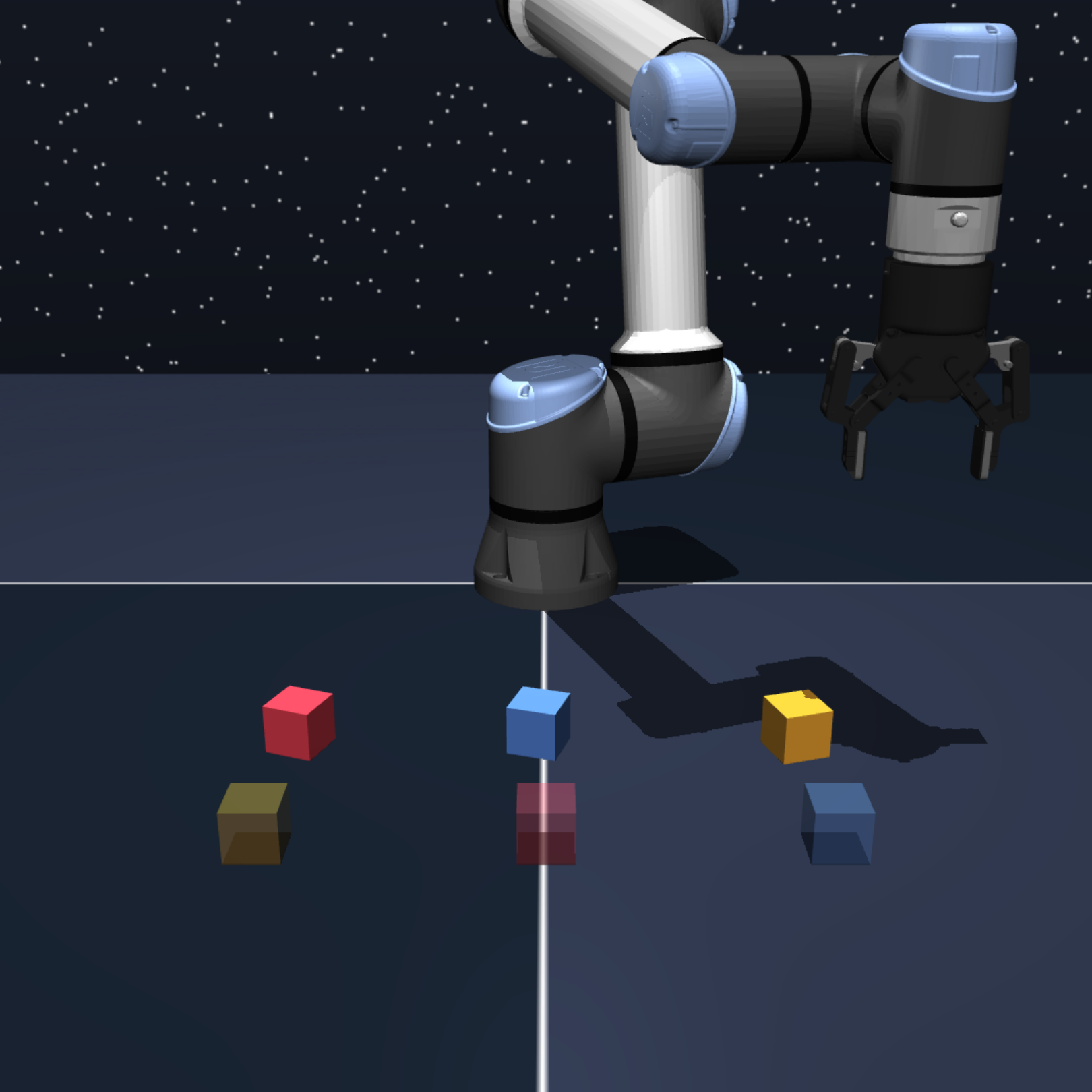}
    \end{subfigure}
    \caption{Example multi-stage manipulation tasks in the Cube environment. Solid cubes indicate the initial configuration and translucent cubes indicate the target configuration.}
    \label{fig:case_cube_demo}
\end{figure}

We formulate the manipulation objectives using STL over the Cartesian position \(\x\) of the end effector. For the \(i\)-th cube, let \(\bs{s}_i\) and \(\bs{g}_i\) denote its initial and target positions, respectively, and let \(r_i\) be the tolerance radius for successful grasping or placing. We define the pick and place predicates as
\[
pick_i:\; r_i-\|\x-\bs{s}_i\|_2 \ge 0,
\;
place_i:\; r_i-\|\x-\bs{g}_i\|_2 \ge 0.
\]
A single-cube manipulation task can then be written as
\[
\varphi_i = \F_{\mr{I}_i}\bigl(pick_i \wedge \F_{\mr{I}_i} place_i\bigr),
\]
where \(\mr{I}_i\) specifies the desired completion interval. For tasks involving multiple cubes without a prescribed order, the overall specification can be expressed as the conjunction of cube-specific subformulas,
\[
\varphi = \bigwedge_{i=1}^{N_{\mathrm{cube}}}\varphi_i.
\]
When a specific order of operations is required, we instead use nested STL expressions to encode the corresponding temporal dependencies. For example, a two-cube sequential manipulation task can be written as
\[
\F_{\mr{I}_1}\Bigl(pick_1 \wedge \F_{\mr{I}_1}\bigl(place_1 \wedge \F_{\mr{I}_2}(pick_2 \wedge \F_{\mr{I}_2} place_2)\bigr)\Bigr),
\]
which requires the robot to complete the manipulation of cube 1 before starting that of cube 2. More generally, multi-stage tasks can be represented by deeper nested formulas of the same form. This is particularly useful for manipulation tasks such as stacking, where the lower cube must be placed before the upper cube, or structured rearrangement tasks such as swapping, where one cube must first be moved away before another can be placed into its target location.

In this case study, STL is used only to specify the end-effector motion. We intentionally do not encode gripper actions and object-contact events directly into the STL formula, since doing so would introduce a heterogeneous hybrid specification over both continuous motion variables and discrete manipulation modes, which is beyond the scope of the present framework. Instead, grasping is handled by a simple control heuristic: the gripper closes automatically when the end effector enters a pick region and opens when it enters a place region. This simplification allows us to focus on the core question of whether the proposed logic-guided planning framework can generate temporally structured manipulation motions from task-agnostic trajectory data. All learned modules are trained exclusively on the \textit{cube-single-play} dataset from OGBench, which contains only single-cube manipulation trajectories with randomly sampled start and goal positions. No multi-cube demonstrations or task-specific retraining are used. During planning, we use the basic planning configuration (without ARS or online replanning), and the resulting trajectories are executed via a PD controller.
\begin{figure*}[t]
    \centering
    % Row 1: initial configurations
    \begin{subfigure}[b]{0.235\textwidth}
        \centering
        \includegraphics[width=\textwidth]{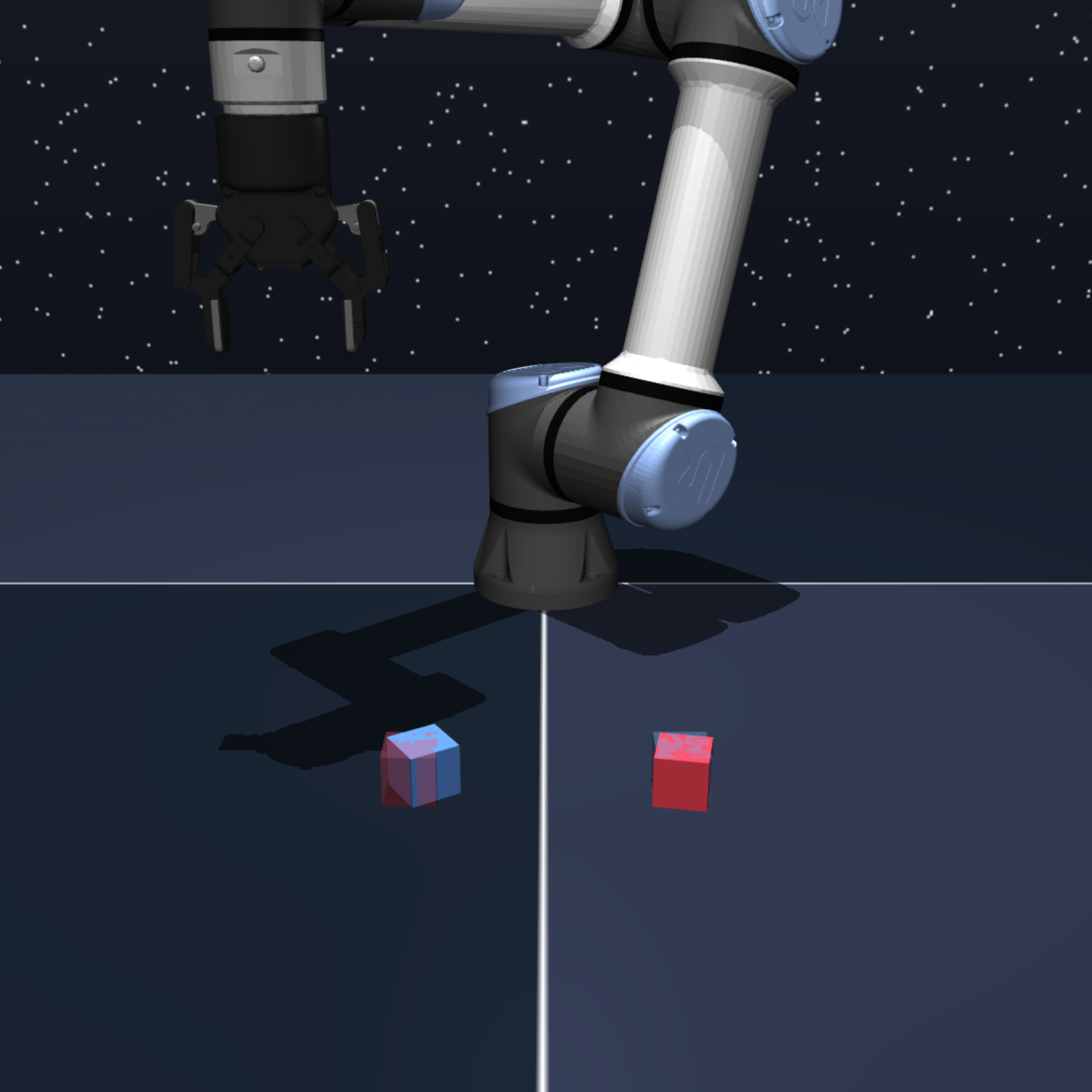}
    \end{subfigure}
    \begin{subfigure}[b]{0.235\textwidth}
        \centering
        \includegraphics[width=\textwidth]{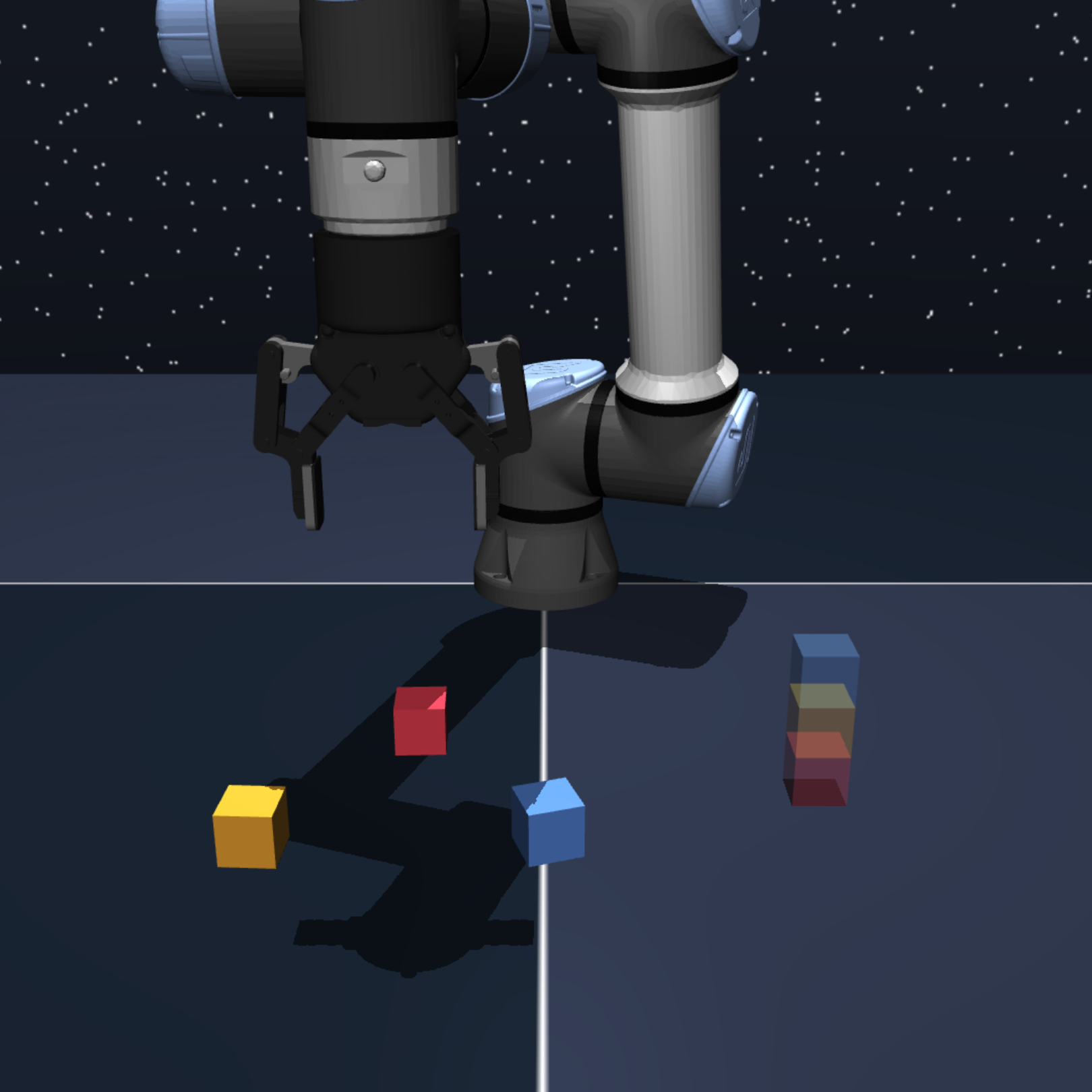}
    \end{subfigure}
    \begin{subfigure}[b]{0.235\textwidth}
        \centering
        \includegraphics[width=\textwidth]{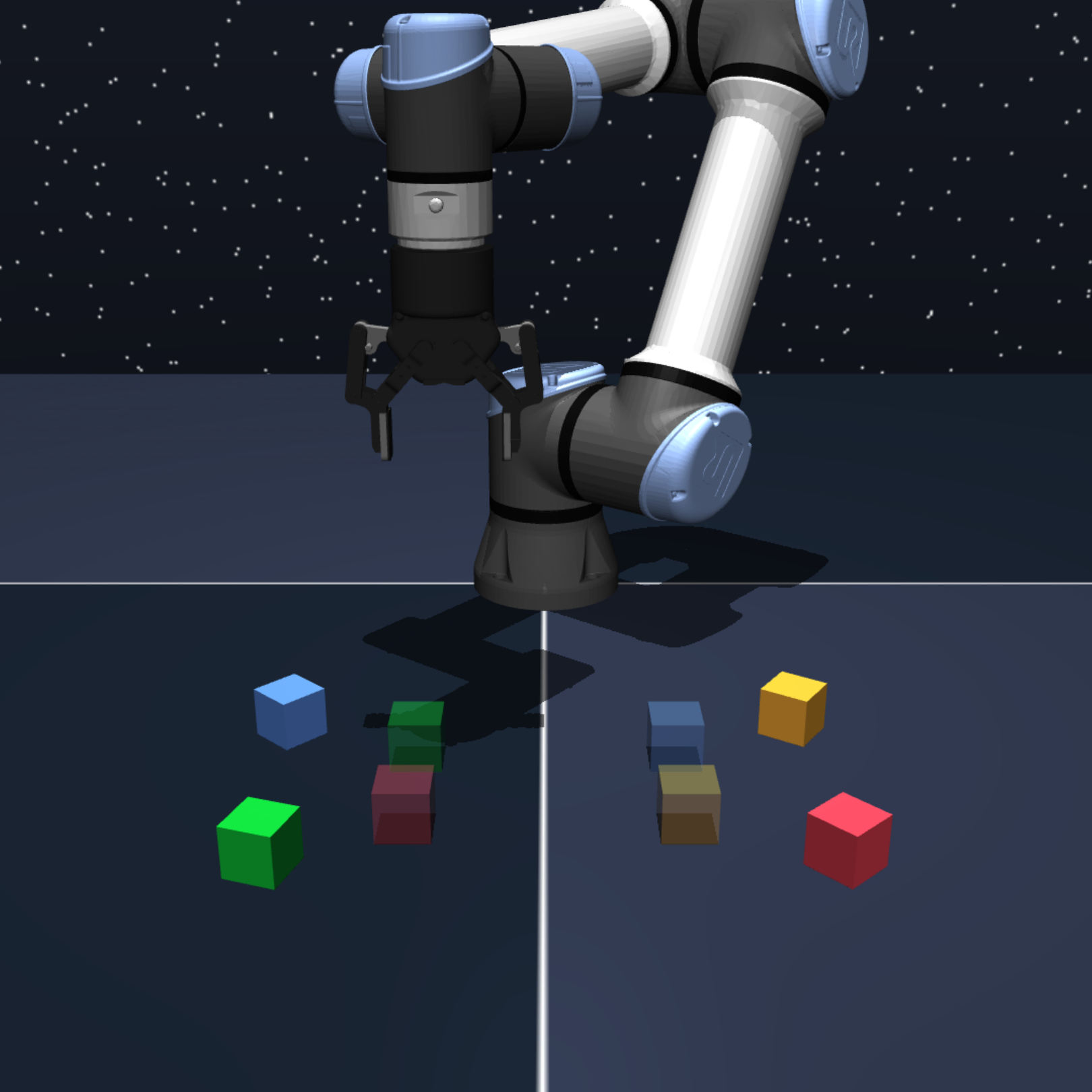}
    \end{subfigure}
    \begin{subfigure}[b]{0.235\textwidth}
        \centering
        \includegraphics[width=\textwidth]{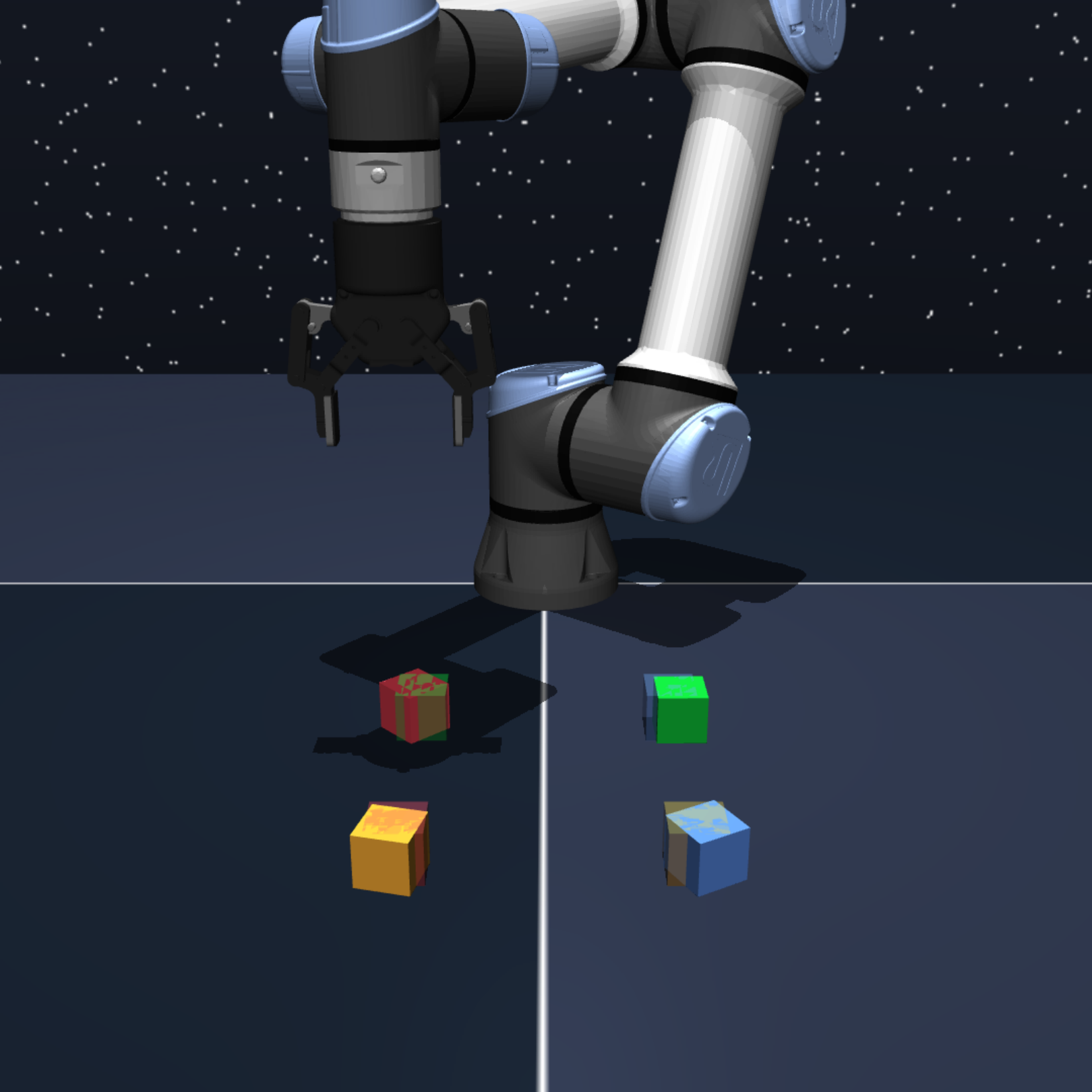}
    \end{subfigure}

    % Row 2: final configurations
    \begin{subfigure}[b]{0.235\textwidth}
        \centering
        \includegraphics[width=\textwidth]{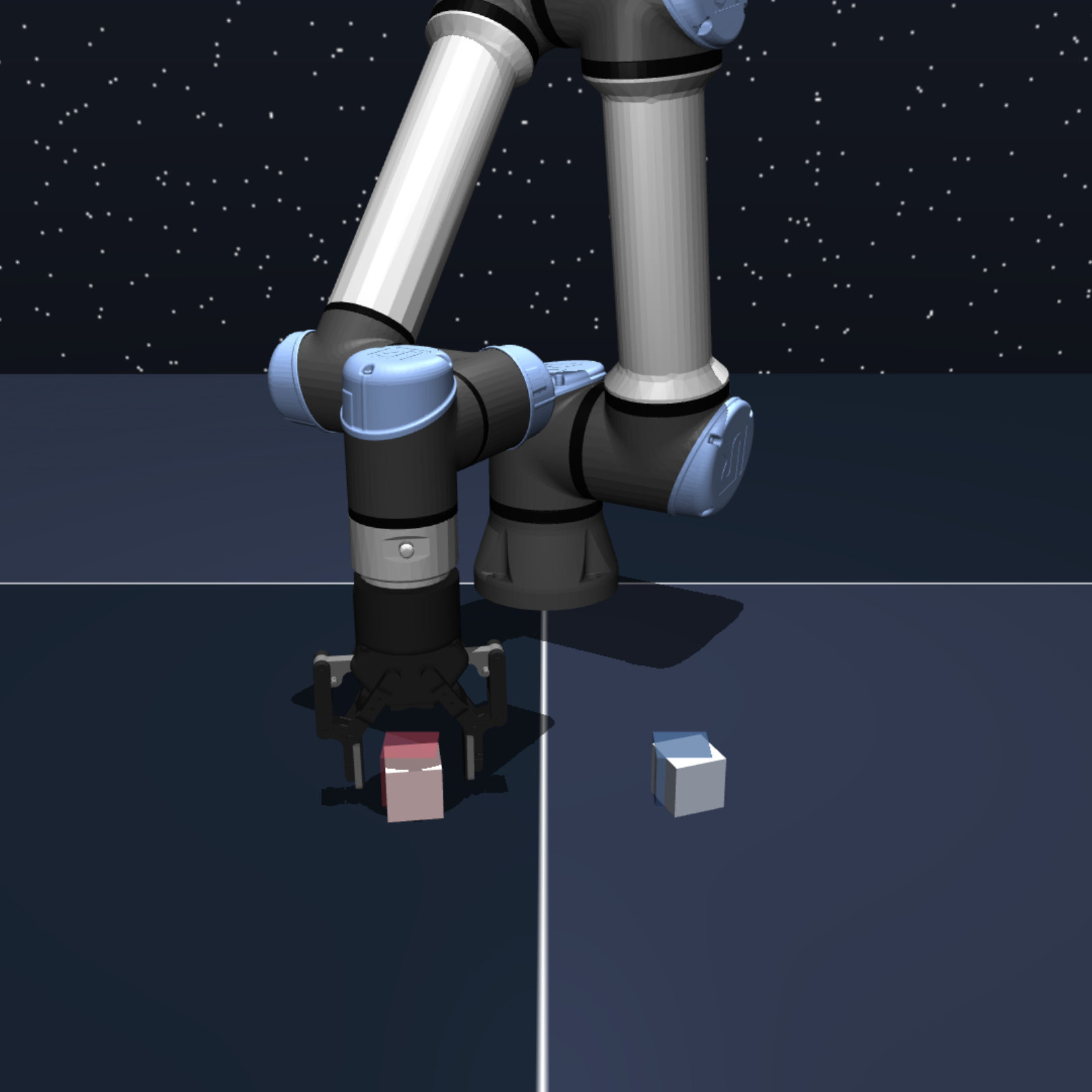}
    \end{subfigure}
    \begin{subfigure}[b]{0.235\textwidth}
        \centering
        \includegraphics[width=\textwidth]{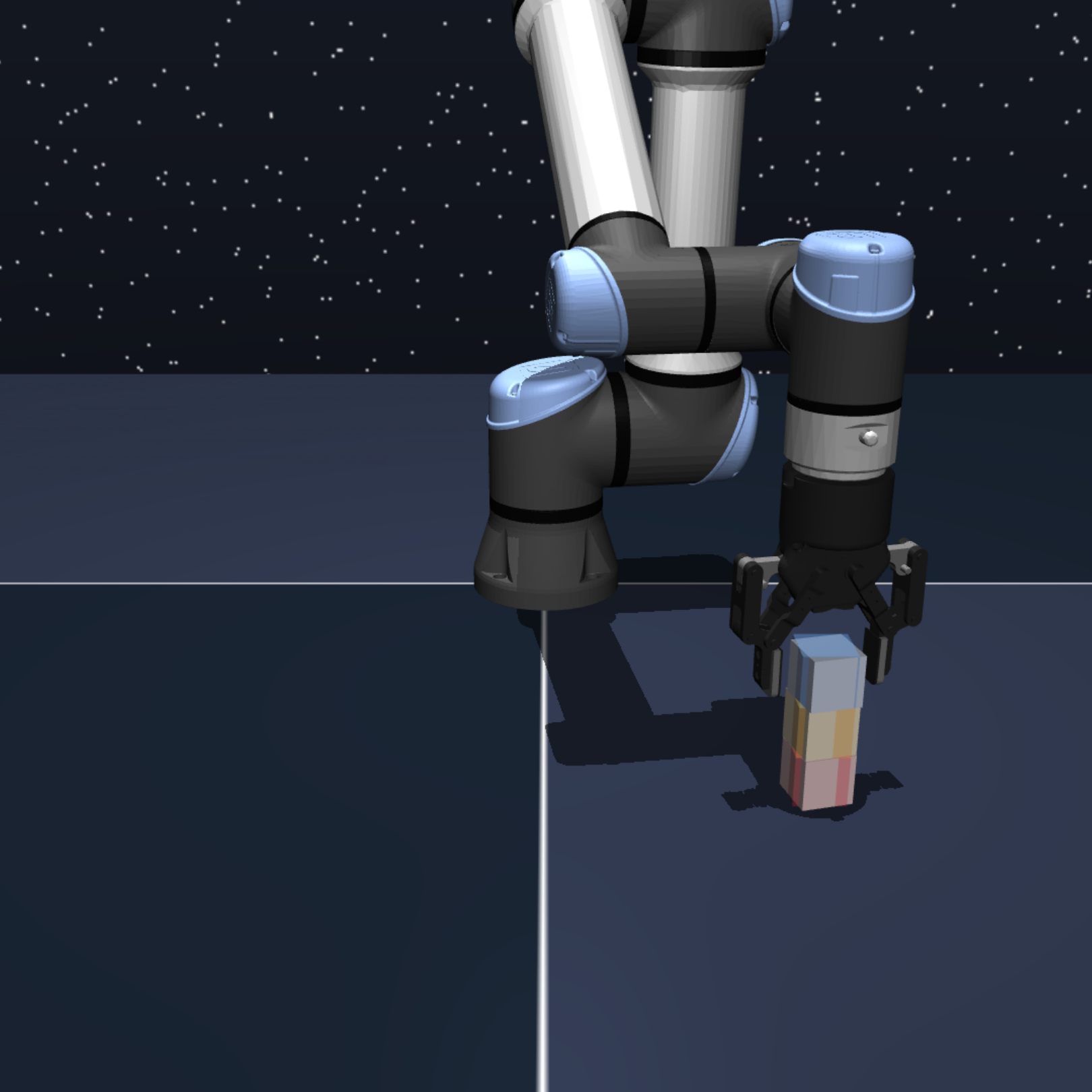}
    \end{subfigure}
    \begin{subfigure}[b]{0.235\textwidth}
        \centering
        \includegraphics[width=\textwidth]{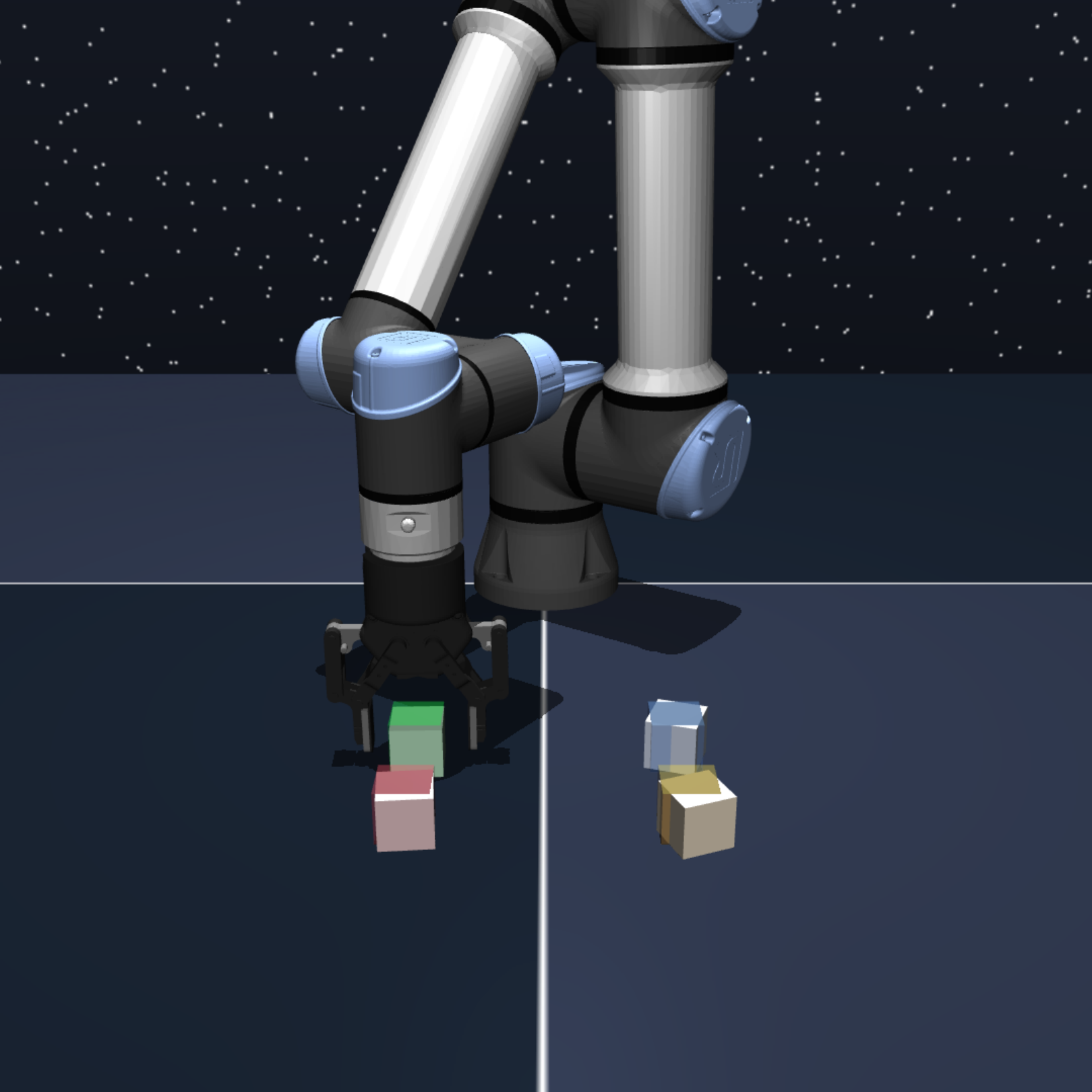}
    \end{subfigure}
    \begin{subfigure}[b]{0.235\textwidth}
        \centering
        \includegraphics[width=\textwidth]{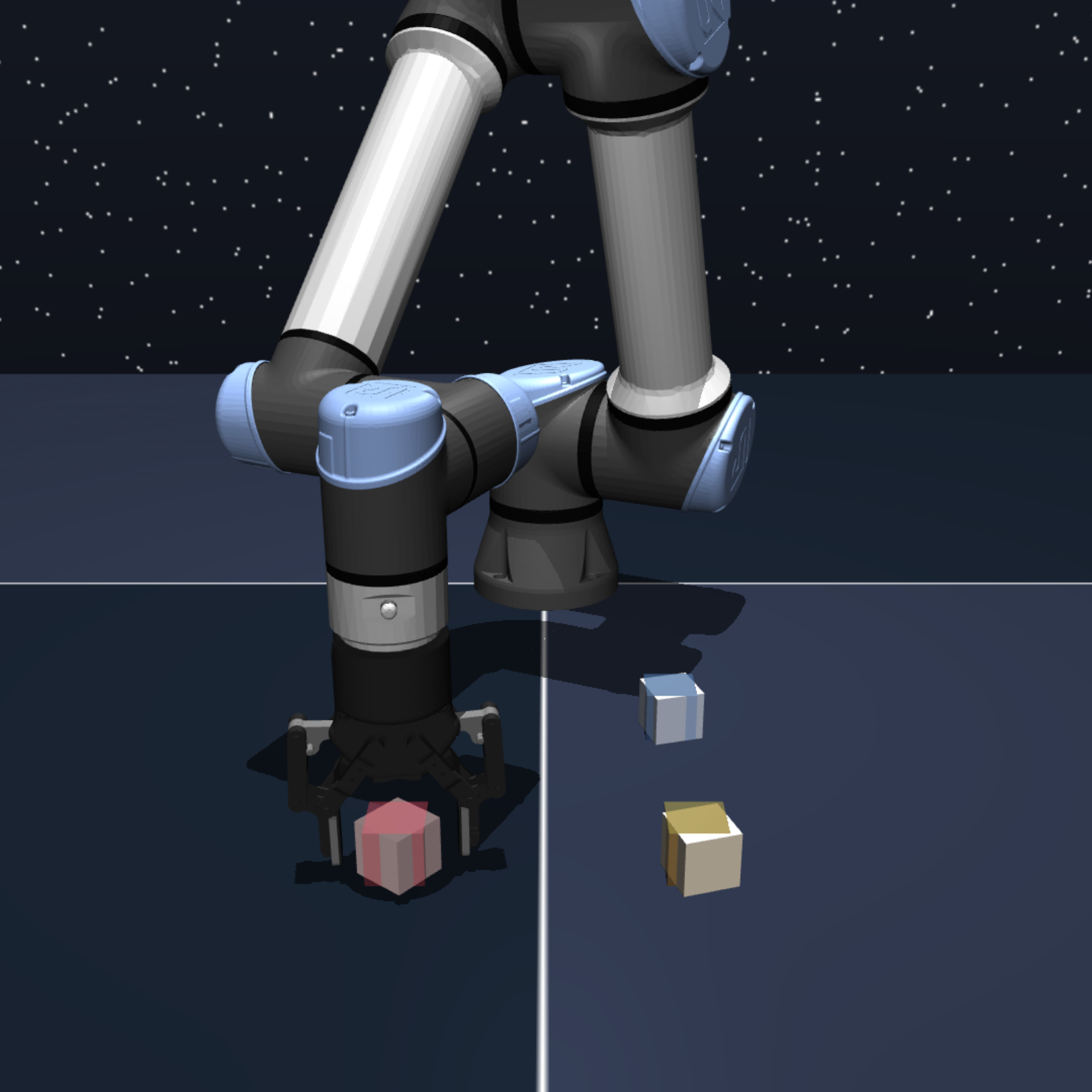}
    \end{subfigure}

    % Row 3: trajectories
    \begin{subfigure}[b]{0.235\textwidth}
        \centering
        \includegraphics[width=\textwidth]{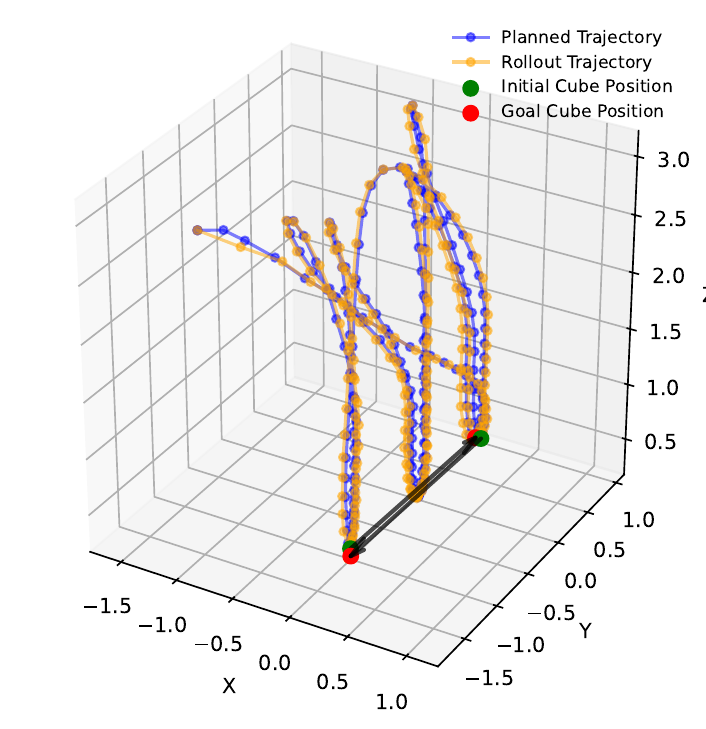}
    \end{subfigure}
    \begin{subfigure}[b]{0.235\textwidth}
        \centering
        \includegraphics[width=\textwidth]{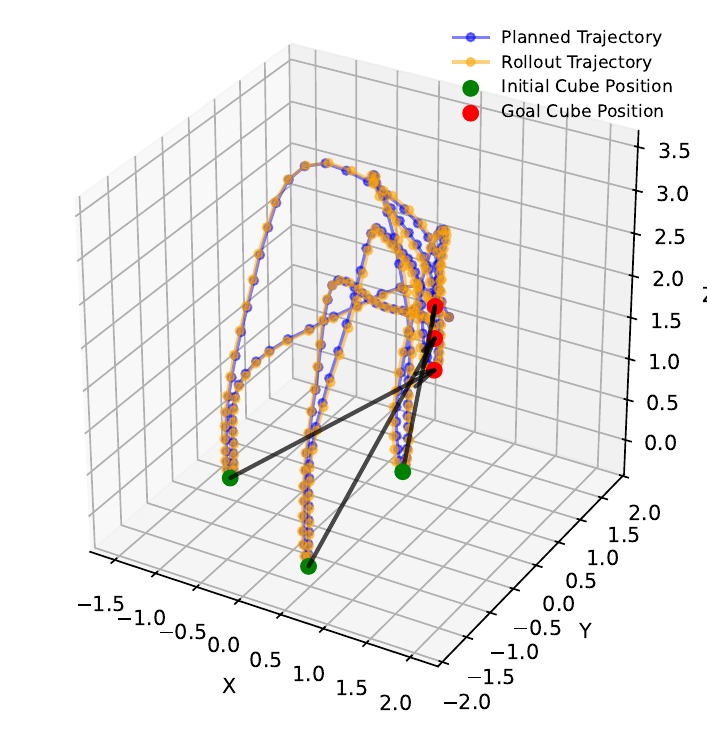}
    \end{subfigure}
    \begin{subfigure}[b]{0.235\textwidth}
        \centering
        \includegraphics[width=\textwidth]{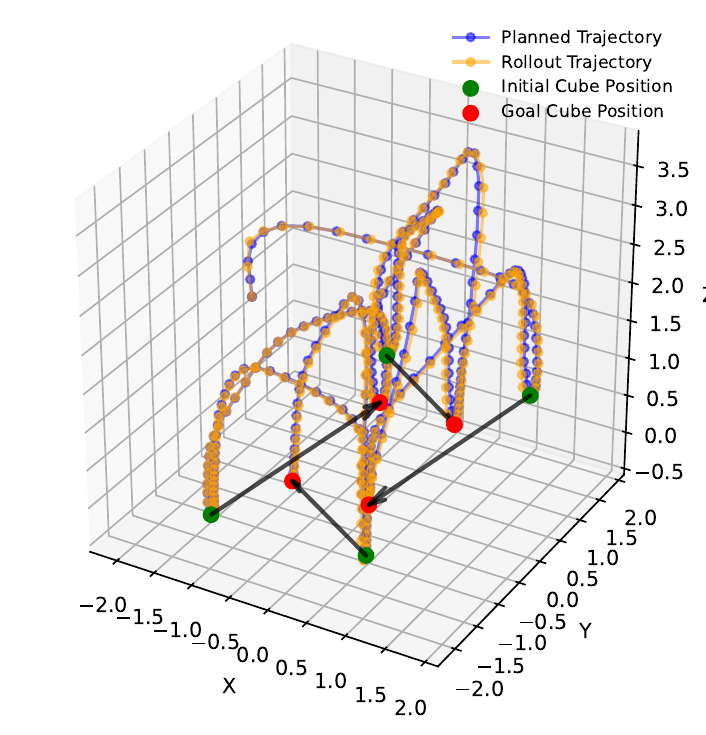}
    \end{subfigure}
    \begin{subfigure}[b]{0.235\textwidth}
        \centering
        \includegraphics[width=\textwidth]{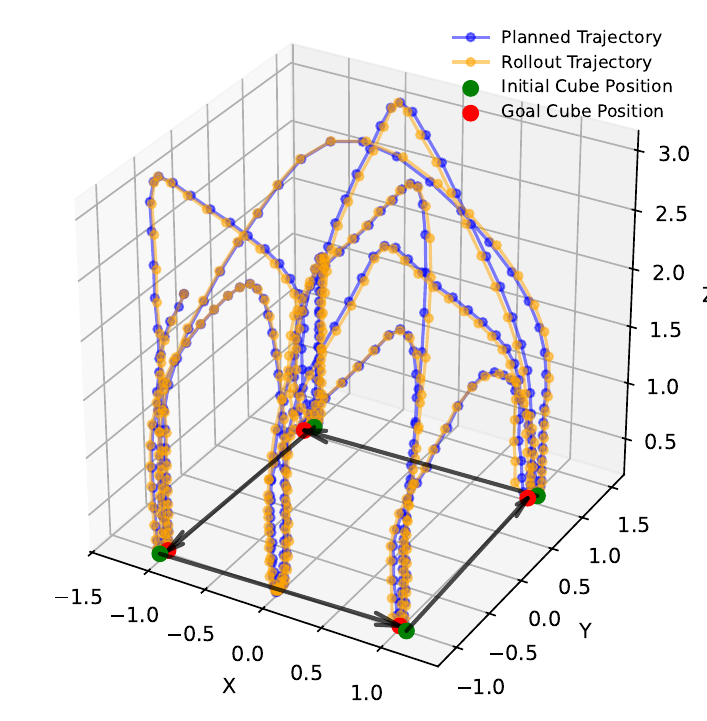}
    \end{subfigure}

    \caption{Representative manipulation cases in the Cube environment. Each column corresponds to one task instance. From top to bottom, the rows show the initial configuration, the final configuration, and the planned (blue)/executed (yellow) end-effector trajectories, respectively.}
    \label{fig:cube_cases}
\end{figure*}

\paragraph{Representative Examples}
Figure~\ref{fig:cube_cases} presents a set of representative manipulation cases. Additional cases and dynamic demonstrations are available at \url{https://cps-sjtu.github.io/DAG-STL}. Each column corresponds to one task instance, covering moving, swapping, stacking, and rotating behaviors in scenarios with two to four cubes. From top to bottom, the rows show the initial configuration, the final configuration after execution, and the planned/executed end-effector trajectories, respectively.

These representative cases show that DAG-STL is not restricted to navigation problems, but can also serve as a logic-based interface for multi-stage robot manipulation. In particular, the framework can synthesize temporally structured end-effector motions for moving, swapping, stacking, and related multi-object behaviors, despite being trained only on task-agnostic single-cube data. A systematic quantitative evaluation of the Cube domain is deferred to the experimental section~\ref{sec:exp_cube}, where we further assess performance on benchmark manipulation tasks and extended task settings beyond the original benchmark.

\subsection{Quantitative Evaluation}\label{sec:exp}
\subsubsection{Task Generation and Filtering}\label{sec:templates}
To evaluate the zero-shot generalization capability of DAG-STL on diverse STL specifications, we construct an evaluation set from randomly generated initial states and STL tasks. Rather than relying only on a few hand-crafted examples, this protocol is designed to systematically cover a representative range of temporal and logical structures relevant to long-horizon STL planning.

Specifically, we define nine STL task templates, summarized in Table~\ref{tab:STLtemplate}. These templates cover the main specification patterns considered in this paper, including simple reach-avoid tasks, conjunctions of multiple temporal goals, bounded-until constraints, nested eventuality structures, and tasks combining reachability with invariance requirements. They are not intended to span the full STL syntax, but rather to provide a representative and systematically controllable family of task patterns for evaluating offline STL planning under unknown dynamics. To instantiate concrete tasks from these templates, we randomly sample both the time intervals appearing in the temporal operators and the geometric parameters of the goal or avoidance regions associated with the atomic predicates. In the navigation environments (Maze2D, AntMaze, and the custom environment), these regions are defined as circles in the two-dimensional task space; in the Cube environment, they are defined as spheres in the three-dimensional Cartesian space of the end-effector position. This procedure yields a diverse set of STL tasks with varying horizon lengths, temporal nesting depths, and spatial configurations.

To obtain a controlled evaluation set and avoid trivially infeasible random specifications, we further apply a feasibility screening step. In the benchmark environments (Maze2D, AntMaze, and Cube), where no sound and complete feasibility oracle is available, we use ARS as a screening procedure: for each sampled task, we run the allocation process under a fixed computational budget and retain the task only if at least one feasible timed waypoint sequence is found. In the custom-built environment, by contrast, we use the optimization-based baseline~\cite{gilpin2020smooth} as a model-aware screening tool, since that environment admits an explicit analytical model. Only tasks for which the baseline returns a feasible solution are included in the final testing set.

\begin{table}[tbp]
\centering
\caption{STL task templates used in the experiments.}
\label{tab:STLtemplate}
\begin{tabular}{cl}
\toprule
\textbf{Type} & \textbf{STL Templates} \\
\midrule
1 & $\F_{\mr{I}_1}\mu_1 \wedge \G(\neg \mu_2)$ \\
2 & $\F_{\mr{I}_1}\mu_1 \wedge \F_{\mr{I}_2}\mu_2$ \\
3 & $\F_{\mr{I}_1}\mu_1 \wedge (\neg \mu_1 \U_{\mr{I}_1} \mu_2) $ \\
4 & $\F_{\mr{I}_1}(\mu_1 \wedge (\F_{\mr{I}_2}(\mu_2 \wedge \F_{\mr{I}_3} (\mu_3\wedge \F_{\mr{I}_4}(\mu_4)))))$ \\
5 & $\F_{\mr{I}_1}(\mu_1 \wedge (\F_{\mr{I}_2}(\mu_2 \wedge \F_{\mr{I}_3} (\mu_3)))) \wedge \G(\neg \mu_4\wedge\neg \mu_5)$ \\ 
6 & $\F_{\mr{I}_1} (\mu_1) \wedge \F_{\mr{I}_2} (\mu_2) \wedge \F_{\mr{I}_3} (\mu_3) \wedge \G(\neg \mu_4)$ \\ 
7 & $\F_{\mr{I}_1} (\G_{\mr{I}_2} (\mu_1)) \wedge \F_{\mr{I}_3} (\mu_2) \wedge \G(\neg \mu_3)$ \\ 
8 & $\F_{\mr{I}_1} (\mu_1 \wedge \F_{\mr{I}_2}( \G_{\mr{I}_3} (\mu_2)))$ \\ 
9 & $\F_{\mr{I}_1} (\mu_1 \wedge \F_{\mr{I}_2} (\mu_2) \wedge \F_{\mr{I}_3} (\mu_3) \wedge \G_{\mr{I}_4} (\mu_4))$ \\
\bottomrule
\end{tabular}

\end{table}

\subsubsection{Algorithm Variants}\label{sec:variants}
To evaluate the contribution of the main algorithmic components, we consider five variants of DAG-STL in the quantitative experiments. These variants are designed to isolate the effects of (i) the allocation strategy, namely basic planning versus ARS, (ii) metric-guided anytime refinement beyond first-solution search, and (iii) the integrated online replanning mechanism during execution.

\begin{itemize}[leftmargin=*]
    \item \textbf{Basic Planning (B).}
    This variant uses the basic decomposition--allocation--generation pipeline described in Section~\ref{sec:allocation} and Section~\ref{sec:traj_generation}, without ARS and without online replanning. It returns the first STL-feasible solution found by the basic allocation search and serves as the nominal baseline of the proposed framework.

    \item \textbf{Basic Planning + Online Replanning (B+OR).}
    This variant augments Basic Planning with the online replanning mechanism described in Section~\ref{sec:online_replanning}. It evaluates how much execution-time recovery alone can improve robustness when the nominal plan is still generated by the basic planner.

    \item \textbf{ARS.}
    This variant replaces the basic allocation module with Anytime Refinement Search. Instead of stopping at the first feasible timed waypoint allocation, it continues exploring alternative allocations under a fixed budget and returns the best STL-feasible trajectory according to the dynamic consistency metric.

    \item \textbf{ARS First Solution (ARS-FS).}
    This variant retains the multi-hypothesis branching mechanism of ARS, but terminates as soon as the first STL-feasible solution is found. It does not use the dynamic consistency metric to rank completed candidates or to guide further refinement. Comparing this variant with full ARS isolates the benefit of metric-guided anytime improvement beyond enlarged search coverage alone.

    \item \textbf{ARS + Online Replanning (ARS+OR).}
    This variant combines ARS with the online replanning mechanism. It represents the most complete version evaluated in the experiments, pairing an execution-oriented offline planner with runtime recovery.
\end{itemize}

\subsection{Experiments in Maze2D}
\subsubsection{Comparison with a Diffusion-Based Baseline}\label{sec:exp_maze}

We first compare the Basic Planning (B) backbone of DAG-STL with a representative diffusion-based baseline to isolate the effect of the \emph{test-time planning formulation}. The goal of this experiment is not to compare different generative backbones or training pipelines, but to examine whether long-horizon STL tasks are better handled by \emph{direct robustness-guided trajectory generation} or by the proposed \emph{decomposition--allocation--generation} strategy. We conduct the comparison on the three Maze2D layouts (\textit{Umaze}, \textit{Medium}, and \textit{Large}) using the filtered random STL tasks introduced in Section~\ref{sec:templates}.

\paragraph{Baseline}
As a representative baseline for direct robustness-guided diffusion planning, we adopt the method of~\cite{zhong2023guided}, which guides diffusion-based trajectory sampling using gradients of STL robustness computed by STLCG~\cite{leung2023backpropagation}. For convenience, we denote this adapted baseline by \textbf{RGD} (\emph{robustness-guided diffusion}) in the remainder of the paper. We note that the original method was not proposed for the long-horizon planning setting considered here; rather, we adapt its core robustness-guided sampling mechanism to our offline planning scenario as a controlled baseline representing the direct ``generate-the-whole-trajectory'' paradigm. In contrast to DAG-STL, RGD attempts to generate a complete STL-satisfying trajectory in a single stage, without explicit task decomposition or timed waypoint allocation. For a fair comparison, both methods use the same diffusion-model backbone and are trained on the same D4RL offline dataset. The difference therefore lies in the test-time planning strategy rather than in the underlying generative model or training data.

\begin{table*}[t]
\centering 
\caption{Comparison between RGD and the Basic Planning (B) backbone of DAG-STL in the Maze2D environments on STL task Types 1--3. For each environment-template pair, results are aggregated over 200 test cases. U: Umaze; M: Medium; L: Large.}
\setlength{\tabcolsep}{4pt}
\begin{tabular}{r >{\raggedleft\arraybackslash}p{0.8cm} rr rr rr}
\toprule
\multirow{2}{*}{\textbf{Env}} & \multirow{2}{*}{\textbf{Type}} & \multicolumn{2}{r}{\textbf{Success Rate(\%)$\uparrow$}} & \multicolumn{2}{r}{\textbf{Average Robustness Value$\uparrow$}} & \multicolumn{2}{r}{\textbf{Planning Time (s)$\downarrow$}}  \\
\cmidrule(r){3-4} \cmidrule(r){5-6} \cmidrule(r){7-8}
 &  & \multicolumn{1}{r}{\textbf{RGD}} & \multicolumn{1}{r}{\textbf{ours (B)}} & \multicolumn{1}{r}{\textbf{RGD}} & \multicolumn{1}{r}{\textbf{ours (B)}} & \multicolumn{1}{r}{\textbf{RGD}} & \multicolumn{1}{r}{\textbf{ours (B)}} \\
\midrule
\multirow{3}{*}{\textbf{U}} & 1 & 97.5 & 99.5 & 0.084$\pm$0.042 & 0.164$\pm$0.037 & 1.86$\pm$0.37 & 0.47$\pm$0.02 \\
 & 2 & 30.0 & 100.0 & -0.249$\pm$0.361 & 0.099$\pm$0.034 & 4.02$\pm$1.93 & 0.96$\pm$0.02 \\
 & 3 & 29.5 & 89.5 & -0.176$\pm$0.245 & 0.080$\pm$0.045 & 18.14$\pm$1.27 & 1.06$\pm$0.02 \\
\midrule
\multirow{3}{*}{\textbf{M}} & 1 & 94.0 & 95.5 & 0.178$\pm$0.062 & 0.256$\pm$0.064 & 3.70$\pm$0.70 & 0.79$\pm$0.02 \\
 & 2 & 26.5 & 92.5 & -0.729$\pm$0.809 & 0.183$\pm$0.125 & 20.96$\pm$16.62 & 1.81$\pm$0.04 \\
 & 3 & 28.0 & 83.5 & -0.380$\pm$0.504 & 0.167$\pm$0.105 & 46.53$\pm$4.59 & 1.90$\pm$0.03 \\
\midrule
\multirow{3}{*}{\textbf{L}} & 1 & 86.5 & 89.5 & 0.235$\pm$0.207 & 0.365$\pm$0.119 & 6.43$\pm$1.54 & 1.24$\pm$0.03 \\
 & 2 & 23.0 & 88.0 & -1.110$\pm$1.033 & 0.286$\pm$0.217 & 11.47$\pm$5.46 & 2.52$\pm$0.05 \\
 & 3 & 23.5 & 78.5 & -0.605$\pm$0.687 & 0.208$\pm$0.235 & 80.96$\pm$10.07 & 2.60$\pm$0.04 \\
\bottomrule
\end{tabular}

\label{tab:experiment_maze}
\end{table*}

\paragraph{Evaluation Protocol and Metrics}
For each environment-template pair, we evaluate both methods on the same set of 200 test tasks sampled according to the protocol in Section~\ref{sec:templates}. Execution and robustness evaluation follow the common procedure in Section~\ref{sec:robustness}, with trajectory resolution factor \(\eta=8\), using the same low-level execution settings described in Section~\ref{sec:implementation}. We report the following metrics:
\begin{itemize}[leftmargin=*]
    \item \textbf{Execution Success Rate (SR):} the percentage of cases whose executed trajectories achieve nonnegative STL robustness;
    \item \textbf{Average Robustness Value (RV):} the average robustness of the executed trajectories per case, computed after discarding the top and bottom 5\% of samples to reduce outlier effects;
    \item \textbf{Planning Time (PT):} the average wall-clock planning time per case, computed after discarding the top and bottom 5\% of samples to reduce outlier effects.
\end{itemize}

\paragraph{Results and Analysis}
The results are summarized in Table~\ref{tab:experiment_maze}. We focus this direct comparison on Types 1--3 because, on more complex templates, RGD rarely produces trajectories that are both STL-satisfying and executable, making the comparison much less informative. The full evaluation of DAG-STL across all templates is provided later in Section~\ref{sec:exp_variants}.

Three observations are most important. First, the Basic Planning (B) backbone of DAG-STL consistently outperforms RGD in all reported settings. Second, the gap widens rapidly with task complexity: while both methods achieve nontrivial success rates on the simplest template (Type 1), RGD deteriorates sharply on Types 2 and 3, whereas DAG-STL maintains high execution success. For example, in Umaze Type 2, the success rate increases from \(30.0\%\) to \(100.0\%\), and in Large Type 3, it increases from \(23.5\%\) to \(78.5\%\). Third, the same trend is reflected in robustness quality: RGD yields negative average robustness on all Type 2 and Type 3 tasks, whereas DAG-STL remains positive in all reported settings.

This comparison is meaningful because RGD represents a direct way to adapt robustness-guided diffusion to our planning setting: generate the entire trajectory at once and optimize it toward STL satisfaction through gradient guidance. However, this paradigm becomes increasingly ineffective for long-horizon structured tasks. On the one hand, directly maximizing the robustness of a full trajectory under temporally compositional STL constraints leads to a difficult nonconvex optimization problem at test time. On the other hand, trajectories that satisfy complex long-horizon tasks often deviate substantially from the trajectories most commonly represented in the offline dataset. As a result, robustness guidance tends to continually push samples away from the learned trajectory prior. As task complexity increases, these two effects compound, leading to lower success rates and higher planning cost.

In contrast, DAG-STL addresses this difficulty by decomposing the original STL task into a sequence of shorter-horizon requirements. This reduces the optimization burden of each generation step and keeps intermediate segments closer to behaviors that can be represented by the learned diffusion prior. As a result, even the Basic Planning backbone achieves markedly higher success rates and robustness than the direct baseline.

DAG-STL is also substantially more efficient. Across all reported settings, it reduces planning time by a large margin, and on the more difficult cases the gain often approaches an order of magnitude or more. In Umaze, the average runtime decreases from \(1.86\)--\(18.14\) seconds for RGD to \(0.47\)--\(1.06\) seconds for DAG-STL; in the Large layout, it decreases from \(6.43\)--\(80.96\) seconds to \(1.24\)--\(2.60\) seconds. This efficiency gain is consistent with the qualitative difference between the two planning paradigms: RGD repeatedly refines full-trajectory samples under robustness guidance, whereas DAG-STL solves a structured sequence of subproblems with much shorter effective horizons.

Overall, these results support the main motivation of DAG-STL: under the same diffusion backbone and offline training data, the decomposition--allocation--generation strategy is substantially more effective than direct robustness-guided diffusion planning for structured long-horizon STL tasks.

\subsubsection{Comparison of DAG-STL Variants}\label{sec:exp_variants}

We next compare the five DAG-STL variants introduced in Section~\ref{sec:variants} on the full set of nine STL task templates in the three Maze2D layouts. The goal of this experiment is to isolate the contribution of the main algorithmic components of the framework, namely: (i) replacing the basic allocation procedure with ARS, (ii) using metric-guided anytime refinement beyond first-solution search, and (iii) incorporating online replanning during execution. For each environment-template pair, we evaluate each variant on 200 test tasks sampled and filtered according to the protocol in Section~\ref{sec:templates}, and execution is carried out using the same evaluation procedure described in Section~\ref{sec:robustness}. We report the execution success rate of each variant in Table~\ref{tab:mode_comparison}.

% Requires \usepackage[table]{xcolor}
% Requires \usepackage{multirow}
\providecommand{\bestcell}[1]{\cellcolor{blue!18}#1}
\providecommand{\secondbestcell}[1]{\cellcolor{blue!8}#1}
\begin{table*}[htbp]
\centering
\caption{Execution success rates (\%) of different DAG-STL variants in Maze2D. U: Umaze; M: Medium; L: Large. Dark blue cells denote the best value, while light blue cells denote the second-best value.}
\label{tab:mode_comparison}
\begin{tabular}{l l *{9}{r}}
\toprule
Env & Variant & Type 1 & Type 2 & Type 3 & Type 4 & Type 5 & Type 6 & Type 7 & Type 8 & Type 9 \\
\midrule
\multirow{5}{*}{U} & B & \secondbestcell{99.5} & \bestcell{100.0} & 89.5 & 98.0 & 90.5 & \secondbestcell{94.0} & 98.5 & 98.0 & 95.0 \\
 & B+OR & \bestcell{100.0} & \secondbestcell{99.5} & 93.0 & 98.0 & 93.0 & \bestcell{94.5} & \bestcell{99.5} & 99.0 & 96.0 \\
 & ARS-FS & \secondbestcell{99.5} & \bestcell{100.0} & 92.0 & \secondbestcell{98.5} & 96.5 & 93.5 & 98.5 & 99.0 & \bestcell{97.0} \\
 & ARS & \bestcell{100.0} & \bestcell{100.0} & \secondbestcell{95.0} & 98.0 & \secondbestcell{98.0} & 92.0 & \bestcell{99.5} & \secondbestcell{99.5} & \secondbestcell{96.5} \\
 & ARS+OR & \bestcell{100.0} & \bestcell{100.0} & \bestcell{96.0} & \bestcell{99.0} & \bestcell{98.5} & \bestcell{94.5} & \secondbestcell{99.0} & \bestcell{100.0} & 95.5 \\
\midrule
\multirow{5}{*}{M} & B & 95.5 & 92.5 & 83.5 & 78.0 & 85.0 & 82.0 & 89.5 & 84.0 & 66.0 \\
 & B+OR & \secondbestcell{96.5} & 95.5 & 92.0 & 83.0 & 90.5 & 86.5 & 92.0 & 86.5 & 66.5 \\
 & ARS-FS & \secondbestcell{96.5} & 94.5 & 90.0 & 90.5 & 87.5 & 87.5 & 95.0 & 93.0 & 91.0 \\
 & ARS & \bestcell{98.5} & \secondbestcell{97.0} & \secondbestcell{93.5} & \secondbestcell{91.5} & \secondbestcell{92.0} & \secondbestcell{90.0} & \bestcell{98.0} & \bestcell{97.0} & \secondbestcell{91.5} \\
 & ARS+OR & \bestcell{98.5} & \bestcell{98.5} & \bestcell{95.5} & \bestcell{95.5} & \bestcell{93.0} & \bestcell{92.0} & \secondbestcell{97.0} & \secondbestcell{96.5} & \bestcell{93.0} \\
\midrule
\multirow{5}{*}{L} & B & 89.5 & 88.0 & 78.5 & 71.0 & 73.5 & 81.0 & 81.0 & 89.0 & 60.0 \\
 & B+OR & 92.5 & 90.5 & 84.0 & 74.0 & 81.5 & 88.5 & 89.0 & 93.0 & 61.0 \\
 & ARS-FS & \secondbestcell{95.5} & 95.5 & 84.0 & 86.0 & 80.5 & 85.0 & 91.0 & 96.0 & 83.0 \\
 & ARS & \bestcell{98.5} & \bestcell{99.0} & \secondbestcell{90.0} & \secondbestcell{94.0} & \secondbestcell{85.5} & \secondbestcell{91.5} & \secondbestcell{93.5} & \secondbestcell{98.0} & \secondbestcell{87.0} \\
 & ARS+OR & \bestcell{98.5} & \secondbestcell{98.5} & \bestcell{91.0} & \bestcell{95.5} & \bestcell{87.5} & \bestcell{92.5} & \bestcell{94.0} & \bestcell{98.5} & \bestcell{90.5} \\
\bottomrule
\end{tabular}
\end{table*}

\paragraph{Overall Trends}
Several clear trends are evident from Table~\ref{tab:mode_comparison}. First, the full version, ARS+OR, achieves the strongest overall performance across the three Maze2D layouts and the nine STL templates. It attains the best or second-best result in nearly all reported settings, and is particularly dominant in the more challenging Medium and Large layouts. Second, the gap between variants is small on simpler templates and in the easiest layout, but widens substantially as the temporal structure of the task and the geometric difficulty of the environment increase. This indicates that the main benefit of the proposed components is not merely to improve already easy cases, but to preserve execution success under increasing temporal and geometric complexity.

\paragraph{Effect of ARS Over Basic Planning}
Comparing B with ARS shows that replacing the basic first-feasible allocation procedure with Anytime Refinement Search leads to substantial gains, especially in the Medium and Large layouts. In the Umaze environment, where the basic planner already performs strongly, the improvement is moderate but still consistent on several of the more structured templates. By contrast, in the Medium and Large layouts, ARS improves success rates across almost all task types, often by a large margin. For example, in Medium Type 9, the success rate increases from \(66.0\%\) for B to \(91.5\%\) for ARS, and in Large Type 4, it increases from \(71.0\%\) to \(94.0\%\). These results suggest that as the timed waypoint allocation problem becomes more combinatorial, stopping at the first feasible solution is increasingly suboptimal, and continued search over alternative allocations becomes important for obtaining executable plans.

\paragraph{Effect of Metric-Guided Anytime Refinement}
The comparison between ARS-FS and ARS isolates the benefit of metric-guided anytime refinement beyond enlarged search coverage alone. In general, ARS outperforms ARS-FS more clearly in the more difficult environments, especially in the Large layout, where ARS achieves higher success rates on all nine templates. This indicates that the gain of ARS cannot be explained solely by the multi-hypothesis branching mechanism; the dynamic consistency metric used to rank and refine candidate allocations also plays an important role. In the easier Umaze layout, the gap between ARS-FS and ARS is smaller and occasionally mixed, which is consistent with the fact that many candidate allocations are already executable in relatively simple geometries. As the environment becomes more challenging, however, selecting among feasible nominal plans according to execution-oriented quality becomes increasingly important.

\paragraph{Effect of Online Replanning}
Online replanning further improves performance in many settings, although the magnitude of the gain depends on the quality of the nominal offline plan. Comparing B+OR with B shows that runtime recovery is especially helpful in the Medium and Large layouts, where the nominal plan is more likely to encounter execution drift. For example, in Large Type~6, adding online replanning to B raises the success rate from \(81.0\%\) to \(88.5\%\), and in Medium Type~3, it raises the success rate from \(83.5\%\) to \(92.0\%\). In the easier Umaze layout, where B already performs strongly, the improvement is smaller and can occasionally be neutral or slightly mixed.

A similar pattern appears when comparing ARS+OR with ARS. Since ARS already provides high-quality nominal plans, the additional margin brought by online replanning is naturally smaller, but it remains visible on several difficult templates, such as Medium Type~4 (\(91.5\%\rightarrow95.5\%\)) and Large Type~9 (\(87.0\%\rightarrow90.5\%\)). At the same time, on several easier templates the difference is marginal and can even be slightly negative. One reason is that, for efficiency, the replanning module in ARS+OR invokes ARS-FS rather than the full ARS. As a result, once replanning is triggered, the recovered suffix is based on the first feasible allocation rather than the higher-quality refined nominal trajectory that ARS could obtain offline, so the replanned continuation can occasionally be less favorable than the original nominal rollout. Overall, these results suggest that online replanning primarily acts as an execution-time recovery mechanism that is most useful when the nominal plan is strong but the rollout remains vulnerable to accumulated deviation.

\paragraph{Combined Effect of Offline and Online Components}
Taken together, the results show that the strongest performance is obtained by combining an execution-oriented offline planner with runtime recovery. Basic planning alone is often sufficient on simple tasks, but its performance degrades noticeably on longer-horizon and more structured templates. ARS substantially improves the quality of the nominal plan by metric-guided searching over alternative timed waypoint allocations, and online replanning further increases robustness during execution by adapting to deviations encountered at execution time. The full ARS+OR variant therefore provides the most reliable behavior across the full range of tested STL tasks.

\providecommand{\bestcell}[1]{\cellcolor{blue!18}#1}
\providecommand{\secondbestcell}[1]{\cellcolor{blue!8}#1}
\begin{table}[t]
\centering
\caption{Average total decision time (s) of different DAG-STL variants in Maze2D, averaged over the nine STL task templates in each environment. U: Umaze; M: Medium; L: Large. Detailed per-template timing results are deferred to Appendix~\ref{apx:additional_timing}. Dark blue cells denote the best value, while light blue cells denote the second-best value.}
\label{tab:mode_time_comparison}
\begin{tabular}{l c c c}
\toprule
Variant & U & M & L \\
\midrule
B & \bestcell{1.21$\pm$0.44} & \bestcell{2.28$\pm$0.91} & \bestcell{3.19$\pm$1.22} \\
B+OR & \secondbestcell{1.25$\pm$0.47} & \secondbestcell{2.57$\pm$1.02} & \secondbestcell{4.44$\pm$1.73} \\
ARS-FS & 3.17$\pm$2.10 & 4.13$\pm$2.24 & 5.06$\pm$2.39 \\
ARS & 9.49$\pm$3.38 & 12.59$\pm$4.04 & 16.83$\pm$5.00 \\
ARS+OR & 9.59$\pm$3.62 & 13.91$\pm$4.85 & 19.95$\pm$8.27 \\
\bottomrule
\end{tabular}
\end{table}

\paragraph{Computational Cost}
We next examine the computational cost of the different variants. Note that the timing metric used here differs slightly from that in Section~\ref{sec:exp_maze}. There, \emph{Planning Time} corresponds to the offline planning time because no online replanning is used. Here, we report \emph{Total Decision Time}, defined as the total time from task input to obtaining a final executable trajectory. For variants without online replanning, this quantity coincides with the offline planning time; for variants with online replanning, it additionally includes the replanning time incurred during execution. Table~\ref{tab:mode_time_comparison} summarizes the average decision time over the nine STL templates in each environment, while detailed per-template results are deferred to Appendix~\ref{apx:additional_timing}.

Table~\ref{tab:mode_time_comparison} shows a clear and stable runtime ordering across all three Maze2D layouts. The basic planner B is the fastest variant, with average decision times of \(1.21\), \(2.28\), and \(3.19\) seconds in Umaze, Medium, and Large, respectively. B+OR remains close to B, while ARS-FS introduces a moderate additional cost. By contrast, ARS and ARS+OR are substantially more expensive, with average decision times increasing to \(9.49\!\sim\!19.95\) seconds depending on the environment. This trend is expected: B terminates once a feasible timed allocation is found, whereas ARS continues to explore and refine alternative allocations in order to improve execution-oriented plan quality. The resulting increase in computation is therefore the cost of obtaining more reliable nominal plans, especially in the more difficult environments.

The comparison between ARS-FS and ARS further clarifies this trade-off. In all three environments, ARS-FS is much faster than ARS, indicating that metric-guided anytime refinement introduces a nontrivial additional search cost beyond first-solution branching alone. At the same time, as shown in Table~\ref{tab:mode_comparison}, ARS typically achieves higher success rates than ARS-FS, particularly in the Medium and Large layouts. This suggests that the additional computation of ARS is not redundant; rather, it is directly associated with better allocation quality and improved execution performance.

For the variants with online replanning, the reported time should be interpreted more carefully. In B+OR and ARS+OR, part of the measured time is incurred during execution, since the total includes both the initial offline planning phase and the additional online replanning time required before a final executable trajectory is obtained. Consequently, the increase from B to B+OR, or from ARS to ARS+OR, should not be interpreted solely as additional offline planning overhead; it also reflects the runtime cost of execution-time recovery.

The difference between B+OR and ARS+OR is also consistent with the replanning mechanisms used in the two variants. In B+OR, each online replanning step invokes the basic planner B, whereas in ARS+OR, online replanning invokes ARS-FS, which still performs multi-candidate branching but does not include metric-guided anytime refinement. Thus, the overhead introduced by OR depends not only on how often replanning is triggered, but also on the complexity of the replanner used at each trigger. As a result, the additional total decision cost introduced by OR is typically smaller for B+OR than for ARS+OR. This also explains why ARS+OR can yield stronger execution performance while incurring a larger runtime overhead, especially in the more difficult environments.

Overall, these timing results complement the success-rate analysis by revealing a clear computation--performance trade-off. B provides the lowest decision time but is less reliable on difficult tasks. Adding OR often improves execution success, but at the price of introducing decision-making effort during execution, which may be undesirable in settings with strict real-time constraints. In such cases, purely offline variants, especially B, may be preferable, or one may use a lightweight replanner such as B during execution. By contrast, ARS incurs substantially higher offline search cost, but yields much stronger nominal plans and higher execution success on complex long-horizon tasks. It is therefore better suited to settings in which offline planning budget is less critical and execution reliability is the primary objective. Among all variants, ARS+OR provides the strongest overall reliability, but also typically incurs the highest total decision cost.

\subsubsection{Analysis of Reachability Time Predictor}\label{exp:time_predictor}

\begin{figure*}[t]
    \centering
    \includegraphics[width=0.97\textwidth]{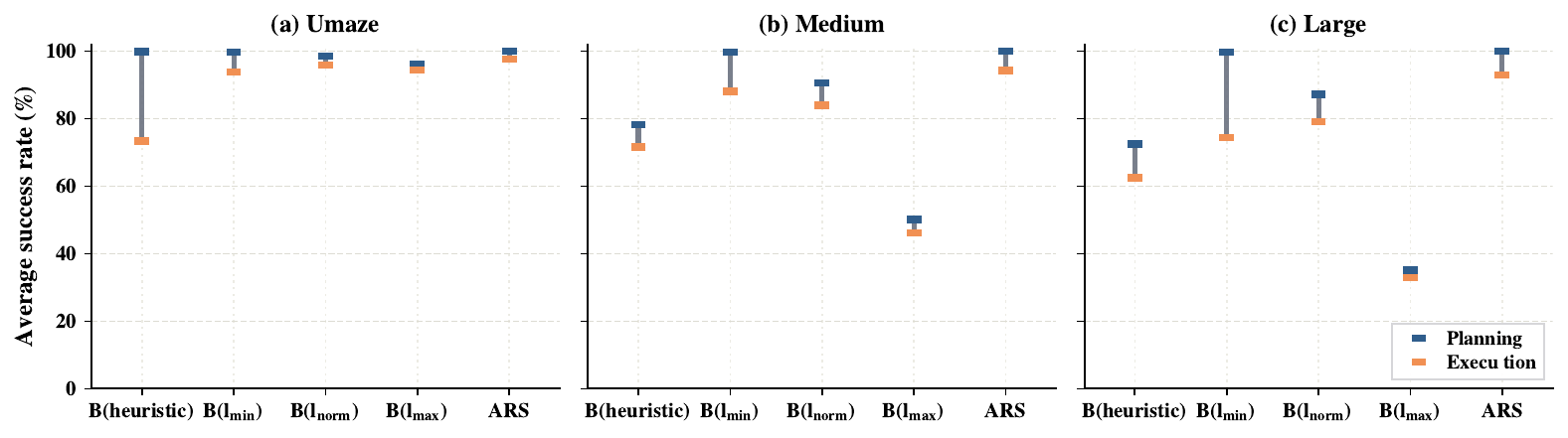}
    \caption{Comparison of average planning and execution success rates across Maze2D environments. In each panel, the x-axis lists the five methods: B(Heuristic), B($l_{\min}$), B($l_{\mathrm{norm}}$), B($l_{\max}$), and ARS, and the y-axis reports success rate (\%). For each environment-method pair, planning and execution success rates are averaged over the nine task templates. The two values are shown as horizontal markers connected by a vertical segment.}
    \label{fig:exp_TP}
\end{figure*}

As described in Section~\ref{sec:time_predict}, our time predictor does not output a single deterministic transition time. Instead, it models a conditional distribution over feasible trajectory lengths between two states and supports sampling-based prediction. The default prediction \(l_{\mathrm{norm}}\) corresponds to an unguided sample reflecting a typical length under the learned distribution, while \(l_{\min}\) and \(l_{\max}\) are obtained by guided sampling toward shorter and longer lengths, respectively. The latter two therefore represent relatively extreme modes of the same learned predictor, whereas \(l_{\mathrm{norm}}\) corresponds to the most typical prediction under the data distribution.

The role of the time predictor in our framework is to inject heuristic dynamical feasibility information into task-level allocation. More specifically, the timed waypoint sequence produced by allocation should not only satisfy the STL timing constraints, but should also be easy to realize later by low-level trajectories that remain consistent with the offline trajectory distribution. From this perspective, different prediction modes induce different trade-offs. A more aggressive time estimate may make it easier to find a nominal allocation, but the resulting timed waypoints can be harder to realize reliably during trajectory generation and execution. A more conservative estimate may improve the realizability of the allocated waypoints, but can also make the allocation itself overly restrictive under limited timing slack.

To analyze this trade-off, we compare four variants of the \textbf{Basic Planning} (\textbf{B}) that differ only in the time prediction mode: the learned predictor using \(l_{\mathrm{norm}}\), the same learned predictor using \(l_{\min}\) only, the same learned predictor using \(l_{\max}\) only, and a heuristic predictor based on scaled Manhattan distance. We additionally include ARS as a reference. Unlike the B variants, ARS does not commit to a single fixed timing mode; instead, it retains and explores multiple temporal hypotheses sampled from the learned predictor and then uses dynamic-consistency feedback from the generated trajectories to refine its allocation choices. 

For each environment-template pair, we evaluate 200 test tasks as in the preceding experiments. Figure~\ref{fig:exp_TP} reports, for each environment and method, the planning success rate (whether the planner finds a complete candidate allocation) and the execution success rate (whether the final executed trajectory satisfies the STL task), averaged over the nine task templates.

\paragraph{Results and Analysis}
Figure~\ref{fig:exp_TP} shows a clear trade-off between nominal schedulability and downstream executability induced by the time prediction mode. Using \(l_{\min}\), which corresponds to an aggressive tail of the learned length distribution, generally increases planning success, indicating that shorter predicted transition times make it easier for the allocator to assemble a complete timed waypoint sequence. However, this benefit often comes with a larger gap between planning and execution success, showing that many of the resulting nominal allocations are harder to realize reliably by low-level trajectories supported by the offline data distribution. Using \(l_{\max}\), which corresponds to an overly conservative tail, produces the opposite behavior: the planning-to-execution gap is usually smaller, suggesting that the resulting timed waypoints are easier to realize once found, but the planning success rate drops substantially, especially in the Medium and Large layouts, because the inflated time estimates make the allocator too restrictive under limited timing windows. In contrast, the default learned prediction \(l_{\mathrm{norm}}\), which represents the most typical sample under the learned distribution, provides a better balance between these two effects and therefore serves as a more effective default mode for Basic Planning.

At the same time, the figure suggests that the framework is not overly brittle to moderate timing error. Even a simple heuristic predictor retains meaningful success rates in all three environments. This tolerance is consistent with the structure of the problem and the data: STL tasks usually provide nonzero timing slack, the offline dataset may contain multiple feasible trajectories of different lengths and paths between the same waypoint pair, and the trajectory generator together with low-level execution can absorb moderate timing mismatch. Thus, the framework does not require an oracle-like time predictor to remain functional. Nevertheless, systematic extreme bias remains harmful: geometric heuristics are consistently inferior to the learned predictor, and the fixed extreme modes \(l_{\min}\) and especially \(l_{\max}\) expose opposite failure modes of overly aggressive and overly conservative allocation.

Finally, ARS achieves the strongest overall performance because it partially breaks this fixed-mode trade-off. Instead of committing to a single conservative or aggressive timing estimate, ARS can preserve multiple temporal hypotheses during allocation and then revise these choices according to the dynamic consistency of the generated trajectories. In this sense, ARS can better absorb the residual systematic bias that may remain even when using \(l_{\mathrm{norm}}\) alone.

\subsubsection{Effectiveness of the Dynamic Consistency Metric}\label{sec:exp_dcm_effectiveness}
As introduced in Section~\ref{sec:dyn_metric}, the proposed Dynamic Consistency Metric (DCM) is designed to measure how well a planned trajectory is supported by the offline training distribution at the local dynamical level. The underlying expectation is that trajectories with better data-supported local transitions should also be easier to execute reliably under the true rollout dynamics. To study this relationship more directly, we additionally perform a reach-avoid evaluation that is independent of the STL planning stack. This experiment uses the Large Maze2D layout and focuses only on ranking multiple candidate trajectories generated for the same local task, rather than on end-to-end STL planning.

\paragraph{Setup and Metrics}
We randomly generate 200 reach-avoid cases in Maze2D-Large. Each case contains one start state, one circular goal region, and two circular avoidance regions. To match the actual use of DCM as closely as possible while isolating the trajectory-generation-to-execution pipeline from the rest of the STL planner, we reuse exactly the same pretrained trajectory generator, time predictor, controller, and rollout settings as those used in the full DAG-STL pipeline. For every case, we therefore follow the same local generation protocol used inside DAG-STL: for each candidate, we first sample a goal state from the goal region, use the same time predictor as in the main Maze2D experiments to predict the corresponding transition length, and then run the diffusion generator with that predicted horizon to produce a reference trajectory. Repeating this process 10 times yields 10 candidates per case and 2000 candidate trajectories in total. Before execution, each candidate is scored only by DCM, without using any rollout information. We then measure whether the executed trajectory reaches the goal while remaining outside both avoidance regions. Figure~\ref{fig:dcm_effectiveness} reports two complementary views: the upper panel shows execution success rate as a function of the DCM rank among the 10 candidates within each case, while the lower panel shows execution success rate after grouping all 2000 candidates by global DCM-score decile.

\begin{figure}[tbp]
    \centering
    \includegraphics[width=0.94\linewidth]{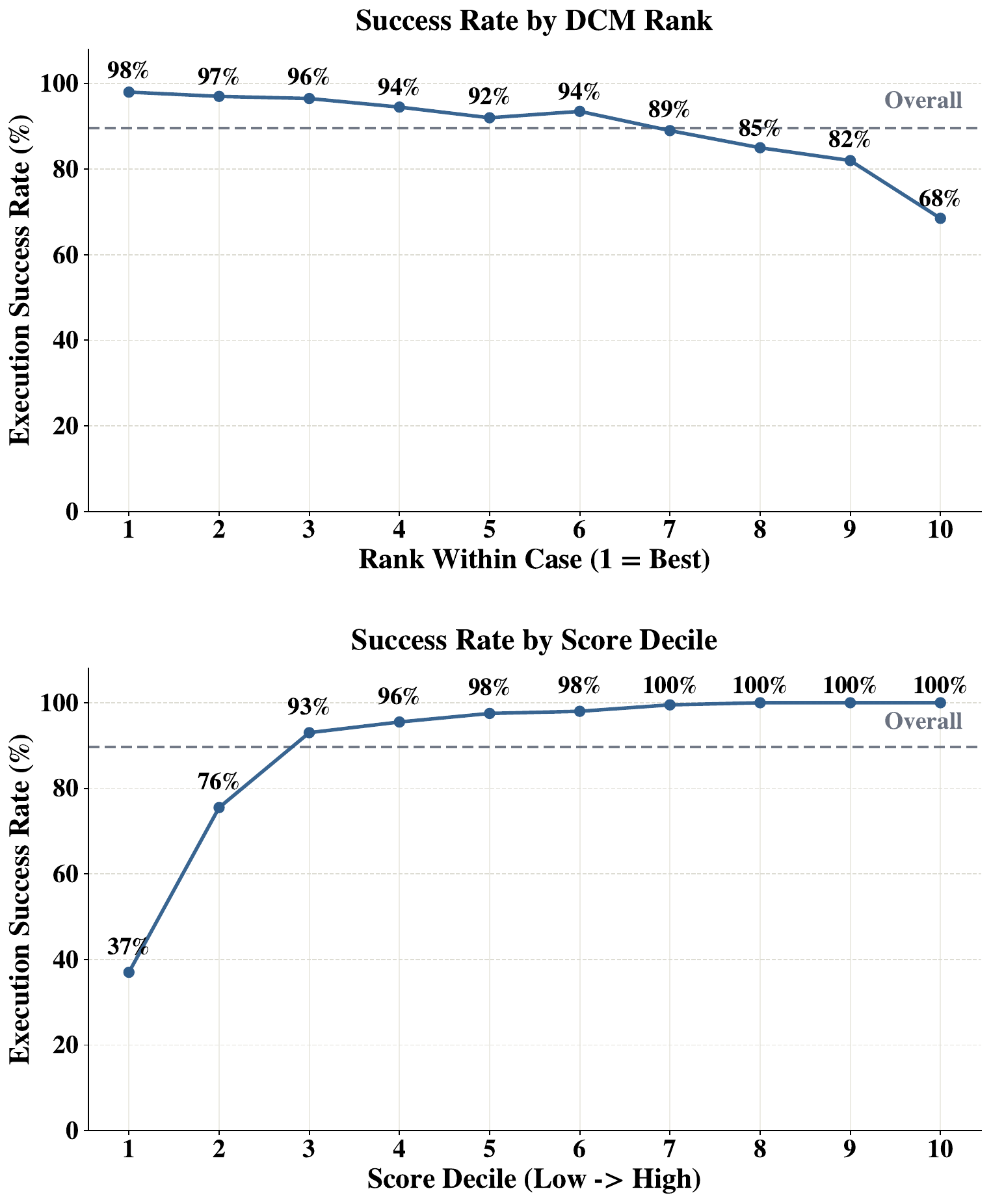}
    \caption{Effectiveness of DCM in Maze2D-Large. Upper panel: execution success rate versus DCM rank among 10 sampled candidates from the same case. Lower panel: execution success rate across global DCM-score deciles.}
    \label{fig:dcm_effectiveness}
\end{figure}

\paragraph{Results and Analysis}
Figure~\ref{fig:dcm_effectiveness} shows that DCM is strongly correlated with the final execution outcome. Within the same case, the candidate ranked first by DCM reaches a \(98.0\%\) execution success rate, whereas the lowest-ranked candidate drops to \(68.5\%\). A similarly clear stratification appears globally: the lowest DCM-score decile succeeds only \(37.0\%\) of the time, while the highest deciles are essentially saturated at \(99.5\%\) to \(100\%\). This trend is also reflected in the aggregate statistics: over all 2000 candidates, the overall execution success rate is \(89.6\%\), but successful trajectories have a much higher average DCM score than failed ones, and the resulting area under the receiver operating characteristic curve (AUROC) reaches \(0.930\).

These results provide direct support for the role of DCM in the full framework. Even outside the STL allocation loop, DCM can reliably identify the candidates that are most likely to remain executable under the true rollout dynamics. In particular, selecting only the top-ranked candidate in each case already raises the case-level success rate to \(98.0\%\), which is very close to the \(98.5\%\) oracle upper bound obtained by accepting a case whenever any one of its 10 candidates succeeds. This confirms that DCM is an effective posterior reranking and candidate-screening module for diffusion-based trajectory planning.

\subsection{Experiments in Other Environments}\label{sec:exp_other}
\subsubsection{Cross-Dynamics Evaluation in AntMaze}\label{sec:exp_ant}
We next evaluate DAG-STL in the AntMaze environment to examine whether the trends observed in Maze2D persist under substantially more complex system dynamics. Compared with Maze2D, AntMaze involves a higher-dimensional locomotion system and a more challenging execution interface, making it a more demanding testbed for logic-guided long-horizon planning.

We report four representative variants: \textbf{Basic Planning (B)}, \textbf{Basic Planning with Online Replanning (B+OR)}, \textbf{ARS}, and \textbf{ARS+OR}. The task generation, execution protocol, and evaluation metrics follow the same setup as in the preceding experiments. For each STL template, results are aggregated over 200 test cases. Table~\ref{tab:exp_antmaze} reports the execution success rate (SR) and total decision time (DT) for the nine STL task templates.

\begin{table*}[tbp]
\centering
\setlength{\tabcolsep}{2pt}
\caption{Results in AntMaze. SR (\%): execution success rate; DT (s): total decision time; Dark blue cells denote the best value, while light blue cells denote the second-best value.}
\label{tab:exp_antmaze}
\providecommand{\bestcell}[1]{\cellcolor{blue!18}#1}
\providecommand{\secondbestcell}[1]{\cellcolor{blue!8}#1}
\small
\begin{tabular}{l l *{9}{r}}
\toprule
 & \textbf{Method} & \textbf{Type 1} & \textbf{Type 2} & \textbf{Type 3}& \textbf{Type 4} & \textbf{Type 5} & \textbf{Type 6} & \textbf{Type 7} & \textbf{Type 8} & \textbf{Type 9} \\
\midrule
\multirow{4}{*}{SR} & B & 81.5 & 90.0 & 74.0 & 77.5 & 67.5 & 70.0 & 71.5 & 85.0 & 63.0 \\
 & B+OR & 84.0 & 89.5 & 83.5 & 80.0 & 67.5 & \bestcell{79.5} & 78.5 & 81.5 & 60.5 \\
 & ARS & \secondbestcell{91.0} & \bestcell{95.0} & \bestcell{87.5} & \secondbestcell{81.5} & \secondbestcell{70.5} & \secondbestcell{77.0} & \secondbestcell{79.0} & \bestcell{92.5} & \bestcell{77.0} \\
 & ARS+OR & \bestcell{92.0} & \secondbestcell{92.5} & \secondbestcell{87.0} & \bestcell{85.0} & \bestcell{71.5} & \secondbestcell{77.0} & \bestcell{82.5} & \secondbestcell{88.0} & \secondbestcell{76.5} \\
\cmidrule{1-11}
\multirow{4}{*}{DT} & B & \bestcell{1.04$\pm$0.02} & \bestcell{2.06$\pm$0.05} & \bestcell{2.18$\pm$0.05} & \bestcell{4.15$\pm$0.10} & \bestcell{3.11$\pm$0.12} & \bestcell{3.16$\pm$0.17} & \bestcell{2.13$\pm$0.06} & \bestcell{2.07$\pm$0.05} & \bestcell{4.21$\pm$0.15} \\
 & B+OR & \secondbestcell{2.29$\pm$1.45} & \secondbestcell{4.25$\pm$2.51} & \secondbestcell{4.29$\pm$2.37} & \secondbestcell{11.29$\pm$8.17} & \secondbestcell{8.96$\pm$7.35} & \secondbestcell{6.06$\pm$2.90} & \secondbestcell{4.07$\pm$2.03} & \secondbestcell{4.44$\pm$2.70} & \secondbestcell{11.74$\pm$8.33} \\
 & ARS & 7.35$\pm$0.32 & 10.43$\pm$1.50 & 17.20$\pm$1.84 & 18.91$\pm$5.96 & 14.70$\pm$3.99 & 14.62$\pm$4.01 & 11.00$\pm$2.26 & 11.19$\pm$2.78 & 21.75$\pm$6.84 \\
 & ARS+OR & 7.88$\pm$1.69 & 12.35$\pm$4.85 & 20.35$\pm$9.96 & 28.27$\pm$34.30 & 20.36$\pm$17.31 & 20.72$\pm$22.08 & 13.21$\pm$6.47 & 14.08$\pm$8.19 & 37.01$\pm$39.43 \\
\bottomrule
\end{tabular}
\end{table*}

\paragraph{Results and Analysis}
The results in Table~\ref{tab:exp_antmaze} show that the main trends observed in Maze2D remain visible in AntMaze, despite the substantially more complex locomotion dynamics. Overall, DAG-STL still achieves strong execution success rates in this harder setting: ARS and ARS+OR consistently obtain the best results, with success rates above \(90\%\) on several templates and around \(70\%\!\sim\!80\%\) even on more challenging ones. This indicates that the framework remains effective beyond simple point-mass navigation and can still produce a substantial number of successful long-horizon executions under much more difficult dynamics.

Within this setting, ARS and ARS+OR provide the strongest overall performance. In particular, ARS performs best or second-best on nearly all task types and substantially outperforms the basic planner on the more structured templates. For example, on Type~9, the execution success rate improves from \(63.0\%\) for B to \(77.0\%\) for ARS. This suggests that, even under more difficult locomotion dynamics, searching over multiple candidate timed allocations remains important for obtaining executable plans.

The effect of online replanning is still visible, but it is notably less uniform than in Maze2D. Comparing B+OR with B shows that runtime recovery improves performance on several task types, especially Types~3, 6, and 7, but can also be neutral or slightly harmful on others. A similar pattern appears when comparing ARS+OR with ARS: online replanning improves performance on several templates, such as Types~1, 4, and 7, but does not consistently dominate ARS. We attribute this behavior to two main factors. For efficiency, the replanning module in ARS+OR invokes ARS-FS rather than the full ARS, so once replanning is triggered, the recovered suffix may be based on the first feasible solution rather than the higher-quality refined solution that ARS could obtain offline. In addition, AntMaze exhibits much stronger rollout jitter than Maze2D, so execution deviations are more likely to push the agent into states that are weakly represented in the offline dataset, such as states close to maze walls, at narrow passage boundaries, or in corner regions that are rarely occupied by successful trajectories in the data. Replanning from such out-of-distribution states can make it harder to recover a high-quality continuation. Taken together, these effects explain why OR remains useful in AntMaze but does not provide the same consistently positive margin as in Maze2D.

The timing results reveal a computation--performance trade-off similar to that in Maze2D. B is by far the fastest variant, while ARS and ARS+OR incur substantially higher decision time due to their deeper search over candidate allocations. Adding online replanning further increases total decision time because part of the computation is shifted into execution. Nevertheless, the success-rate improvements of ARS and ARS+OR indicate that this additional computation remains worthwhile when reliability on complex long-horizon tasks is the primary objective.

Overall, the AntMaze results support the same conclusion as the Maze2D study: the advantage of DAG-STL is not tied to simple navigation dynamics. Its core components, especially multi-candidate allocation and refinement search, remain effective under substantially more complex locomotion dynamics, while the behavior of online replanning becomes more sensitive to execution noise and distribution shift.

\subsubsection{Cross-Domain Evaluation in the Cube Environment}\label{sec:exp_cube}

We next turn to the ``Cube'' scenario from OGBench~\cite{ogbench_park2025} to evaluate whether DAG-STL generalizes beyond navigation and can support temporally structured robot manipulation. Here we focus on cross-domain generalization rather than fine-grained comparison among planning variants, since the environment's strong low-level execution interface reduces the sensitivity to high-level planning differences. Accordingly, throughout this section we use the \textbf{Basic Planning} (\textbf{B}) configuration as the default planner and evaluate whether it is already sufficient to support zero-shot manipulation generalization. Following the setup in Section~\ref{sec:case_manipulation}, STL is imposed only on the end-effector motion, while grasping and release are handled by a simple control heuristic. All learned modules are trained exclusively on the \textit{cube-single-play} dataset from OGBench, which contains only single-cube manipulation trajectories with randomly sampled start and goal cube positions. No multi-cube demonstrations or task-specific retraining are used.

\paragraph{Benchmark Manipulation Tasks}
We first report results on the original OGBench benchmark tasks. As in Section~\ref{sec:case_manipulation}, each benchmark manipulation task is converted into an STL specification over the end-effector motion by composing cube-specific pick-and-place subformulas; the detailed construction is omitted here for brevity. To assess zero-shot generalization, we evaluate the framework on four levels of difficulty: \textit{cube-single}, \textit{cube-double}, \textit{cube-triple}, and \textit{cube-quadruple}, involving one to four cubes, respectively. Each scenario contains five benchmark tasks, giving 20 tasks in total, and each task is repeated with ten random seeds. Since these benchmark tasks are defined primarily by object rearrangement outcomes, success is evaluated directly from the final object configuration rather than STL robustness: a trial is counted as successful only if all cubes reach their designated goal locations at the end of execution.

\begin{table*}[tbp]
\centering
\caption{Success rate (\%) on the original OGBench Cube benchmark tasks. Each scenario (\textit{cube-single}, \textit{cube-double}, \textit{cube-triple}, \textit{cube-quadruple}) contains five benchmark tasks, and each task is repeated with ten random seeds.}
\label{tab:cube_manipulation_results}
\small
\begin{tabular}{lccccc|c}
\toprule
\textbf{Scenario} & \textbf{Task 1} & \textbf{Task 2} & \textbf{Task 3} & \textbf{Task 4} & \textbf{Task 5} & \textbf{Avg. SR (\%)} \\
\midrule
\textit{Cube-Single}     & 100.0 & 90.0 & 90.0 & 100.0 & 100.0 & \textbf{96.0} \\
\textit{Cube-Double}     & 100.0 & 100.0 & 100.0 & 100.0 & 90.0  & \textbf{98.0} \\
\textit{Cube-Triple}     & 100.0 & 100.0 & 50.0 & 100.0 & 90.0  & \textbf{88.0} \\
\textit{Cube-Quadruple}  & 100.0 & 100.0 & 0.0  & 100.0 & 30.0  & \textbf{66.0} \\
\bottomrule
\end{tabular}
\end{table*}

The results in Table~\ref{tab:cube_manipulation_results} show strong zero-shot generalization from single-cube training data to substantially more complex benchmark manipulation tasks. The average success rate exceeds \(95\%\) in the \textit{cube-single} and \textit{cube-double} settings, and remains \(88\%\) and \(66\%\) in the \textit{cube-triple} and \textit{cube-quadruple} settings, respectively. This indicates that logic-guided end-effector planning alone is already sufficient to induce nontrivial multi-stage object manipulation behaviors, even when all learned components are trained only on task-agnostic single-cube trajectories. The main failures occur in tasks with stronger physical interactions, especially stacking-related ones, where unintended contacts between the end effector and nearby cubes can propagate to object displacement or collapse. This is consistent with the scope of the current formulation: the planner explicitly reasons over end-effector motion, but does not model the full contact-rich object dynamics of multi-cube manipulation.

\begin{table*}[t]
\centering
\caption{Extended STL-style manipulation task families in the Cube domain. Each family contains two fixed cube-count templates, shown here in canonical STL form.}
\label{tab:cube_stl_families}
\small
\begin{tabular}{
>{\raggedright\arraybackslash}m{2.5cm}
>{\raggedright\arraybackslash}m{1.2cm}
>{\raggedright\arraybackslash}m{5.2cm}
>{\raggedright\arraybackslash}m{5.5cm}
}
\toprule
\textbf{Task family} & \textbf{\# Cubes} & \textbf{STL template} & \textbf{Interpretation} \\
\midrule
Post-Placement Protection
& 2, 3
& \(
\F_{\mr{I}_1}\!\bigl(pick_1\wedge\F_{\mr{I}_1}\bigl(place_1\wedge\G_{\mr{I}_2}protect_1\wedge
\F_{\mr{I}_3}(pick_2\wedge\F_{\mr{I}_3}(place_2\wedge\G_{\mr{I}_4}protect_2\wedge\cdots))\bigr)\bigr)
\)
& After each cube is placed, all already completed goal regions remain protected during subsequent manipulations. \\
\midrule
Forbidden Rearrangement
& 1, 4
& \(
\G_{\mr{I}}\ avoid_{\mathcal O}\wedge
\bigwedge_{i=1}^{N}\F_{\mr{I}_{i}}(pick_i\wedge\F_{\mr{I}_{i}}place_i)
\)
& Rearrangement must be completed while always avoiding one or more forbidden workspace regions \(\mathcal O\). \\
\midrule
Structured Stacking
& 2, 3
& \(
\F_{\mr{I}_1}\!\bigl(pick_1\wedge\F_{\mr{I}_1}\bigl(place_1\wedge\G_{\mr{I}_2}protect_1\wedge
\F_{\mr{I}_3}(pick_2\wedge\F_{\mr{I}_3}(place_2\wedge\G_{\mr{I}_4}protect_2\wedge\cdots))\bigr)\bigr)
\)
& Cubes must be stacked in a prescribed bottom-to-top order, and completed lower layers remain protected during later placements. \\
\midrule
Unlock-Then-Place
& 2, 4
& \(
\bigl(avoid_{\mathcal O_{\mathrm{lock}}}\bigr)\U_{\mr{I}_1}\bigl(pick_u\wedge\F_{\mr{I}_1}place_u\bigr)
\wedge \F_{\mr{I}_2}(pick_t\wedge\F_{\mr{I}_2}place_t)
\)
& A locked target region $\mathcal{O}_{\mathrm{lock}}$ cannot be entered until a designated unlock cube has first been moved to its goal. \\
\bottomrule
\end{tabular}
\end{table*}

\paragraph{STL-Style Manipulation Tasks Beyond the Original Benchmark}
While the original OGBench tasks already demonstrate meaningful cross-domain transfer, they do not fully reflect the range of temporal-logic structure that DAG-STL is designed to handle. We therefore further extend the Cube domain with four families of STL-style manipulation tasks, each instantiated with two cube-count variants.

For the \(i\)-th cube, let \(\bs{s}_i\) and \(\bs{g}_i\) denote its initial and goal positions, respectively. As in Section~\ref{sec:case_manipulation}, we define
\[
pick_i:\; r_i-\|\x-\bs{s}_i\|_2 \ge 0,
\;
place_i:\; r_i-\|\x-\bs{g}_i\|_2 \ge 0.
\]
In addition, \(protect_i\) denotes the predicate that the end effector remains outside a protection region centered at \(\bs{g}_i\), and \(avoid_{\mathcal O}\) denotes the predicate that the end effector remains outside a forbidden region \(\mathcal O\). Table~\ref{tab:cube_stl_families} summarizes the canonical STL templates used for the four extended manipulation task families. For the larger cube-count variants, the same compositional pattern is extended recursively. Note that although Post-Placement Protection and Structured Stacking share a similar recursive logical form, they differ in geometric instantiation: the former corresponds to general rearrangement with sequential protection, whereas the latter uses stacked target configurations and enforces bottom-to-top placement semantics.

For each template, task instances are generated by structured randomization. Cube initial positions, goal positions, temporal windows, and template-specific auxiliary regions such as forbidden, protection, and locked regions are randomly sampled under validity constraints to avoid contradictory or physically degenerate instances. For each of the eight templates, we generate 200 task instances and evaluate three success-rate metrics. The \emph{logic success rate} measures the fraction of executions whose end-effector trajectories achieve nonnegative STL robustness. The \emph{object success rate} measures the fraction of executions in which all relevant cubes reach their designated goal configurations at the end of execution. Finally, the \emph{task success rate} requires both conditions to hold simultaneously, and is therefore the main metric reported for overall task completion.

\begin{table*}[htbp]
\centering
\caption{Template-level success rates (\%) for randomly instantiated STL cube manipulation tasks.}
\label{tab:stl-cube-template-summary}
\small
\begin{tabular}{lcccc}
\toprule
\textbf{Task Family} & \textbf{\# Cubes} & \textbf{Logic Success} & \textbf{Object Success} & \textbf{Task Success} \\
\midrule
\multirow{2}{*}{Post-Placement Protection} & 2 & 100.0 & 85.5 & 85.5 \\
 & 3 & 98.5 & 76.5 & 76.5 \\
\midrule
\multirow{2}{*}{Forbidden Rearrangement} & 1 & 100.0 & 98.0 & 98.0 \\
 & 4 & 97.5 & 65.0 & 65.0 \\
\midrule
\multirow{2}{*}{Structured Stacking} & 2 & 100.0 & 92.5 & 92.5 \\
 & 3 & 100.0 & 83.5 & 83.5 \\
\midrule
\multirow{2}{*}{Unlock-Then-Place} & 2 & 100.0 & 93.0 & 93.0 \\
 & 4 & 95.0 & 80.5 & 80.5 \\
\bottomrule
\end{tabular}
\end{table*}

\paragraph{Results and Analysis}
The results in Table~\ref{tab:stl-cube-template-summary} show that DAG-STL remains effective on these extended STL-style manipulation tasks beyond the original OGBench benchmark. Across all eight templates, the logic success rate is consistently high, ranging from \(95.0\%\) to \(100.0\%\). It indicates that the robot is usually able to carry out the required temporally structured end-effector behavior in execution, even under the additional logical constraints introduced by the new templates.

At the same time, the object success rates are systematically lower than the corresponding STL success rates. This gap is informative: it shows that a large fraction of failures do not come from violating the STL task at the end-effector level, but from manipulation-specific execution errors after the logical motion has essentially been carried out correctly. In particular, object failure is mainly caused by unintended contacts during execution, such as accidentally pushing nearby cubes, disturbing already placed objects, or destabilizing partial stacks. This effect becomes more pronounced in the harder templates, especially those involving more cubes, where physical interactions are more frequent and errors can propagate more easily.

A clear complexity trend is also visible. For every task family, the larger-cube variant is more difficult than the smaller one, with the drop appearing mainly in object success rather than STL success. For example, in Forbidden Rearrangement, the STL success decreases only slightly from \(100.0\%\) to \(97.5\%\) when moving from 1 cube to 4 cubes, whereas the object success drops much more substantially from \(98.0\%\) to \(65.0\%\). A similar pattern appears in Post-Placement Protection and Unlock-Then-Place. This suggests that the main challenge in these extended tasks is not generating an end-effector trajectory that satisfies the logical constraints, but ensuring that the resulting manipulation remains robust to contact-rich interactions as the number of objects increases.

Overall, these results provide two complementary conclusions. First, DAG-STL can successfully generalize to richer STL-style manipulation templates beyond the original benchmark, and the executed end-effector trajectories satisfy the required temporal logic in the vast majority of cases. Second, the remaining performance gap is largely attributable to object-level manipulation failures caused by unmodeled contact effects, rather than to a failure of the logic-guided planning mechanism itself.

\subsection{Comparative Experiment with an Optimization-Based Baseline}\label{sec:exp_custom}

We finally consider the custom-built environment from Section~\ref{sec:env} to obtain a controlled model-aware reference for the otherwise fully offline setting. This section contains two complementary controlled studies. The first evaluates how much of a model-solvable reference task set DAG-STL can recover using offline data alone, and how this recovery compares with direct optimization. The second constructs a tighter combinatorial stress test to isolate allocation behavior when success depends on identifying a near-optimal visiting order under a shared deadline.

\subsubsection{Comparison on a Model-Solvable Reference Set}

\paragraph{Baseline}  
We use the optimization-based STL planner provided in \texttt{stlpy}~\cite{kurtz2022mixed}, following the smooth robustness formulation of~\cite{gilpin2020smooth}. Unlike DAG-STL, this baseline plans directly in a full model-aware setting: it optimizes a state-action trajectory with \emph{complete access to the analytical double-integrator dynamics and obstacle geometry}. Therefore, it serves as a reference for the best possible performance under the given dynamics and task templates. By contrast, DAG-STL must plan from offline trajectories alone without direct access to the system model or environment layout. The comparison provides a controlled reference for how much of the model-solvable reference task set DAG-STL can recover using offline data alone, and how this recovery compares with direct optimization.

\paragraph{Evaluation Protocol and Metrics}  
The environment, offline dataset, and training protocol follow the common setup in Sections~\ref{sec:env} and~\ref{sec:implementation}. The test tasks are sampled from the same STL template family introduced in Section~\ref{sec:templates}. For this controlled study, we generate 200 test tasks per template and retain only those for which the optimization-based baseline returns a feasible solution. Since the optimization-based solver becomes prohibitively slow on the two deepest nested templates (Types~4 and~5), we report comparisons only on Types~1, 2, 3, 6, 7, 8, and 9.

For DAG-STL, planning, trajectory generation, and robustness evaluation follow the common protocol of Sections~\ref{sec:implementation} and~\ref{sec:robustness}. For the optimization-based baseline, by contrast, a successful solve itself certifies feasibility under the known dynamics and constraints, since the baseline directly solves the full model-based control problem. This difference is inherent to the two information regimes: the baseline plans with complete access to the analytical system model and obstacle geometry, whereas DAG-STL must plan from offline trajectories alone.

On this retained reference set, the optimization-based baseline can be regarded as having \(100\%\) feasibility by construction. Accordingly, we report SR0 and SR only for the DAG-STL variants, while using planning time to compare all methods.

We report the following metrics:
\begin{itemize}[leftmargin=*]
    \item \textbf{Allocation Success Rate (SR0):} for DAG-STL variants, the proportion of retained tasks for which the Progress Allocation module finds a solution.
    \item \textbf{Execution Success Rate (SR):} for DAG-STL variants, the proportion of retained tasks whose final executed trajectories achieve nonnegative STL robustness.
    \item \textbf{Planning Time (PT):} the average wall-clock planning time per case; this metric is reported for all DAG-STL variants and the optimization-based baseline.
\end{itemize}

\begin{table*}[tbp]
\centering
\caption{Results in the custom-built environment. SR0 (\%) and SR (\%) are reported for DAG-STL variants, and PT (s) denotes planning time. For each template, results are aggregated over 200 test tasks.}
\begin{tabular}{c r r r | r r r | r r r r}
\toprule
\multirow{2}{*}{\centering \textbf{Type}} & \multicolumn{3}{c}{\textbf{SR0(\%)$\uparrow$}} & \multicolumn{3}{c}{\textbf{SR(\%)$\uparrow$}} & \multicolumn{4}{c}{\textbf{Planning Time (s)$\downarrow$}} \\
\cmidrule(r){2-11}
 & \textbf{B} & \textbf{ARS-FS} & \textbf{ARS} & \textbf{B} & \textbf{ARS-FS} & \textbf{ARS} & \textbf{B} & \textbf{ARS-FS} & \textbf{ARS} & \textbf{Baseline} \\
\midrule
1 & 96.0 & 100 & 100 & 94.5 & 99.0 & 99.5 & 0.42$\pm$0.01 & 0.76$\pm$0.03 & 1.55$\pm$0.03 & 0.45$\pm$0.20 \\
2 & 99.0 & 100 & 100 & 99.0 & 100 & 100 & 0.82$\pm$0.02 & 1.48$\pm$0.05 & 2.53$\pm$0.45 & 2.08$\pm$0.68 \\
3 & 97.5 & 100 & 100 & 95.5 & 98.5 & 98.5 & 0.88$\pm$0.02 & 2.80$\pm$0.06 & 3.73$\pm$0.42 & 10.60$\pm$6.87 \\
\midrule
6 & 96.5 & 100 & 100 & 95.0 & 96.5 & 97.0 & 1.20$\pm$0.03 & 2.16$\pm$0.06 & 3.23$\pm$0.55 & 4.12$\pm$1.23 \\
7 & 97.5 & 100 & 100 & 96.5 & 100 & 100 & 0.81$\pm$0.02 & 1.45$\pm$0.04 & 2.44$\pm$0.42 & 3.30$\pm$1.30 \\
8 & 96.0 & 100 & 100 & 96.0 & 100 & 100 & 0.82$\pm$0.02 & 1.46$\pm$0.04 & 2.51$\pm$0.47 & 29.12$\pm$13.84 \\
9 & 99.0 & 100 & 100 & 99.0 & 100 & 100 & 1.70$\pm$0.05 & 3.82$\pm$0.98 & 6.01$\pm$1.32 & 214.66$\pm$89.94 \\
\midrule
Avg. & 97.4 & 100 & 100 & 96.5 & 99.1 & 99.3 & 0.95$\pm$0.03 & 1.99$\pm$0.18 & 3.14$\pm$0.52 & 37.76$\pm$16.29 \\
\bottomrule
\end{tabular}
\label{tab:custom_env_comparison}
\end{table*}

\paragraph{Results and Analysis}
Table~\ref{tab:custom_env_comparison} shows that DAG-STL remains highly effective in this controlled setting. Although it only has access to offline trajectories and relies on finite-budget heuristic search, it still recovers most tasks on the model-solvable reference set. Averaged over the seven templates, Basic Planning (B) achieves \(97.4\%\) SR0 and \(96.5\%\) SR, while ARS First Solution (ARS-FS) already reaches \(100\%\) SR0 and \(99.1\%\) SR. Full ARS further attains \(100\%\) SR0 and \(99.3\%\) SR. In particular, the fact that ARS-FS and ARS both reach \(100\%\) SR0 shows that, in this relatively simple controlled environment, increasing the search budget and branching width is already sufficient to recover essentially the entire model-solvable set.

The same refinement gain observed in the earlier experiments still holds here: compared with ARS-FS, full ARS still yields a small improvement in realized success. However, this gain is much less pronounced in the present setting, where the dynamics and execution interface are simple and SR already remains very close to SR0.

The timing results reveal a similarly clear efficiency picture. On average, B, ARS-FS, and ARS require \(0.95\), \(1.99\), and \(3.14\) seconds per task, respectively, whereas the optimization-based baseline requires \(37.76\) seconds. Even the larger-budget ARS variant therefore remains substantially faster than direct optimization, and the gap widens on the harder templates: on Type~8, ARS requires \(2.51\) seconds compared with \(29.12\) seconds for the optimization-based baseline, and on Type~9 it requires \(6.01\) seconds compared with \(214.66\) seconds. Taken together, these results show that DAG-STL can recover almost all model-solvable tasks in this controlled setting while retaining a large computational advantage over full model-based optimization.

\subsubsection{Combinatorial Stress Test}
\paragraph{Task Family and Goal}
The comparison above shows that DAG-STL recovers most tasks on the retained reference set in this controlled environment. A complementary question is how well it can still \emph{find any valid solution} in deliberately near-extremal regimes, and how that solution-finding ability changes as the search budget of ARS is varied. To probe both aspects more directly, we additionally construct a hard-feasible combinatorial stress test in the same custom environment. We choose a shared-deadline multi-goal task family because it is especially challenging for DAG-STL: the planner must identify a suitable visiting order from many possibilities, while the common deadline weakens much of the temporal heuristic guidance that would otherwise help order the subgoals. At the same time, its difficulty remains easy to control by tightening or relaxing the shared deadline. The task family takes the form
\[
\varphi_n=\bigwedge_{i=1}^{n}\F_{[0,D]} p_i,
\]
where all \(n\) circular goal regions share the same deadline \(D\). The formula therefore specifies \emph{what} regions must be visited, but not \emph{in what order}. The resulting tasks are feasible by construction but deliberately hard, since only a limited subset of visiting orders remains valid under a tight deadline.

\begin{figure}[t]
    \centering
    \begin{subfigure}[t]{0.48\columnwidth}
        \centering
        \includegraphics[width=\columnwidth]{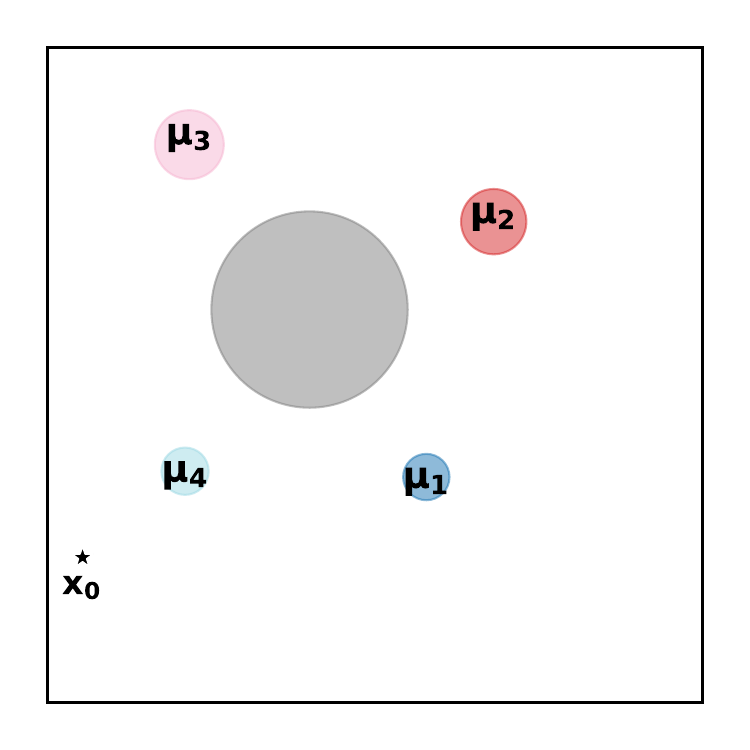}
    \end{subfigure}\hfill
    \begin{subfigure}[t]{0.48\columnwidth}
        \centering
        \includegraphics[width=\columnwidth]{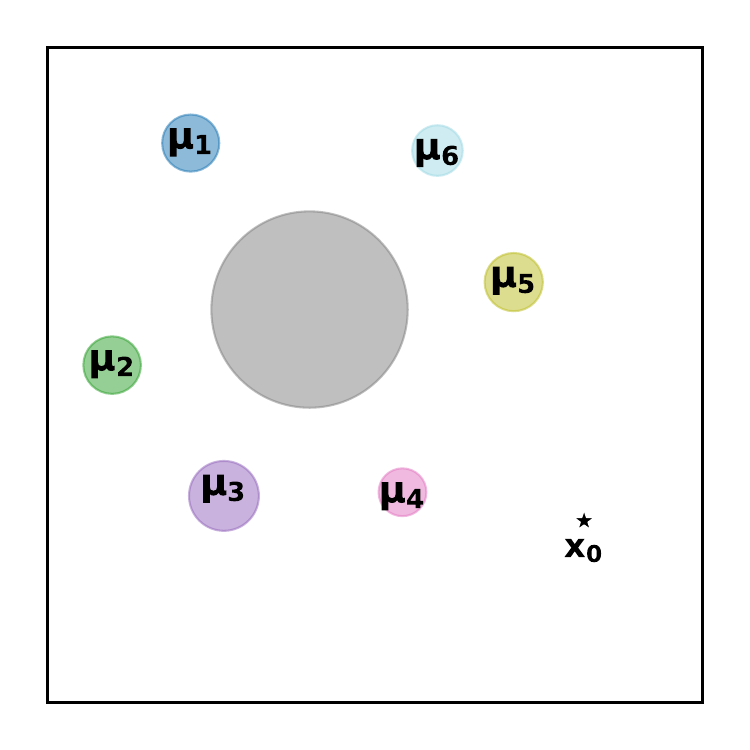}
    \end{subfigure}
    \caption{Representative sampled layouts for the combinatorial stress test in the custom-built environment, shown here for \(n=4\) and \(n=6\). Under this obstacle-centered layout, a natural feasible strategy is to start from the initial state and visit the regions in either clockwise or counterclockwise order, while the STL task itself still leaves the order unspecified.}
    \label{fig:custom_stress_layouts}
\end{figure}

\begin{figure*}[t]
    \centering
    \begin{subfigure}[t]{0.98\textwidth}
        \centering
        \includegraphics[width=\textwidth]{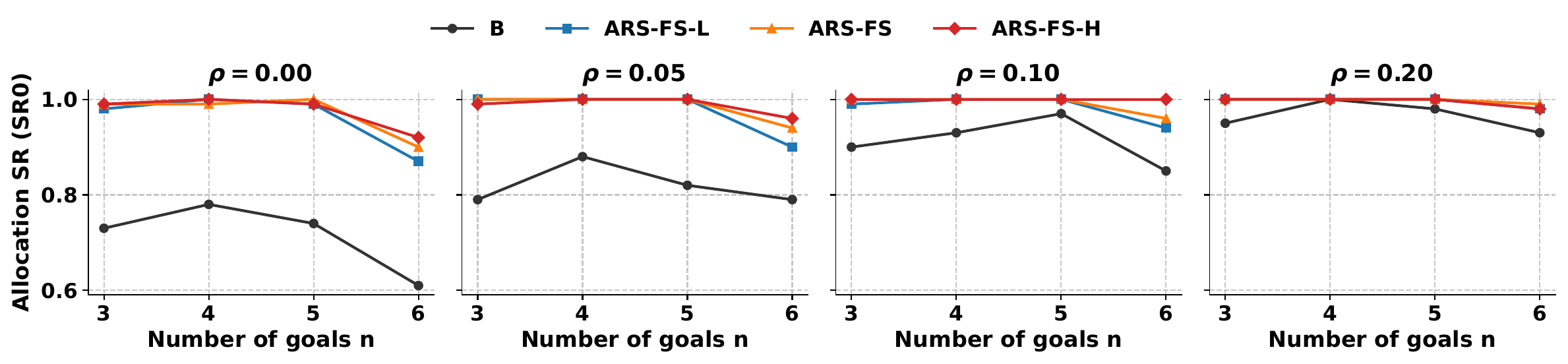}
        % \caption{Allocation success rate SR0 versus the number of unordered goals \(n\).}
    \end{subfigure}

    \vspace{0.5em}

    \begin{subfigure}[t]{0.98\textwidth}
        \centering
        \includegraphics[width=\textwidth]{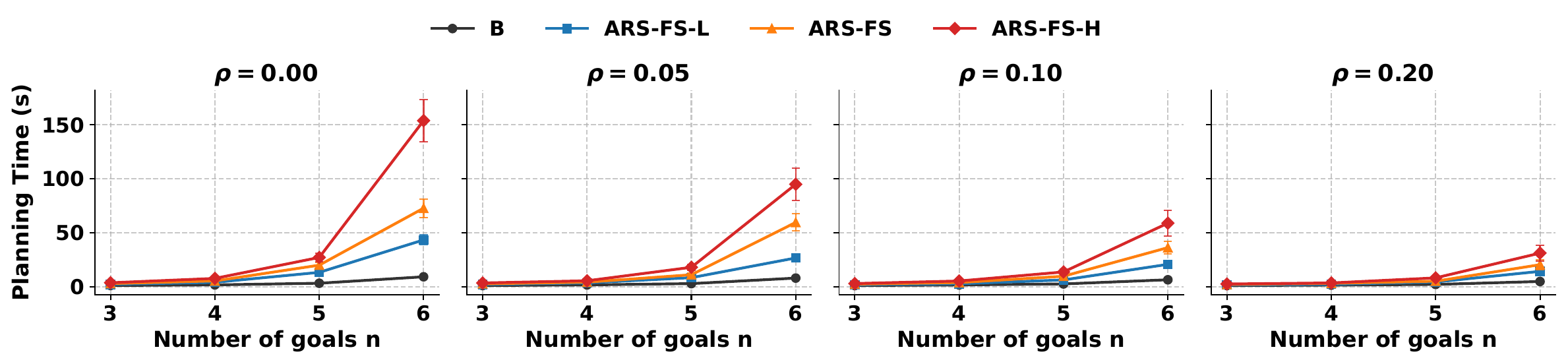}
        % \caption{Mean total internal planning time versus the number of unordered goals \(n\).}
    \end{subfigure}
    \caption{Results on the combinatorial stress test in the custom-built environment. Each column fixes a slack ratio \(\rho\) and varies the number of unordered goals \(n\). The upper panel reports SR0, and the lower panel reports planning time.}
    \label{fig:custom_stress_results}
\end{figure*}

\paragraph{Controlled Task Construction}
The controlled setup is designed so that each sampled task is both feasible and deliberately tight. We first randomly generate obstacle-centered layouts of goal regions, as illustrated in Figure~\ref{fig:custom_stress_layouts}. Under such layouts, completing the task within a short shared deadline naturally requires visiting the regions in a coherent clockwise or counterclockwise order starting from the sampled initial state. We therefore use this natural order only as a hidden witness for task construction, without providing it to DAG-STL. Along this witness order, we solve each local segment with the model-based planner to estimate the shortest feasible system horizon, and then calibrate that horizon against the empirical segment-length distribution seen during DAG-STL training by taking the larger of the minimal feasible horizon and a distance-matched empirical target horizon from the training segments. The calibrated segment horizons are summed to obtain a reference duration \(T^\star\), from which we define the overall deadline as \(D=\lceil(1+\rho)T^\star\rceil\). In this way, \(\rho\) directly controls how tight the resulting task is. We then apply the optimization-based sanity check to the full task and keep only cases that remain feasible under the true model. In the reported experiment, we use \(n\in\{3,4,5,6\}\) and \(\rho\in\{0,0.05,0.10,0.20\}\). For each \(n\), we repeatedly sample and filter tasks in this manner until 100 accepted base cases are retained. Each accepted base case is then instantiated with the four \(\rho\) values, yielding 100 tasks for every \((n,\rho)\) setting and 1600 tasks in total. For each accepted base case, the layout and hidden witness are fixed, and only \(D\) in the STL task is varied through \(\rho\) to instantiate the four difficulty levels. This procedure yields a large collection of tasks that are guaranteed to have solutions while remaining deliberately tight, so failures primarily reflect the difficulty of finding a good allocation under limited search budget rather than arbitrary task infeasibility or severe local distribution mismatch.

\paragraph{Compared Variants and Evaluation}
On the accepted task set, we compare \textbf{Basic Planning (B)} with three budget levels of \textbf{ARS First Solution (ARS-FS)}, denoted ARS-FS-L, ARS-FS, and ARS-FS-H. The three ARS-FS variants use \((K,N_{\mathrm{iter}},t_{\max})=(3,400,120\text{s})\), \((5,600,180\text{s})\), and \((8,800,240\text{s})\), respectively, with \(N_{\mathrm{seq}}=1\) in all cases, while Basic Planning (B) uses its default search procedure under the same wall-clock cap \(t_{\max}=180\text{s}\). All variants are evaluated on the same accepted tasks. As in the controlled comparison above, we report SR0 and planning time, with SR0 serving as the primary metric in this stress test because our focus here is solution-finding ability under increasing combinatorial difficulty. The final rollout success rate remains very close to SR0 overall in this environment, so it is not shown separately.

\paragraph{Results and Analysis}
Figure~\ref{fig:custom_stress_results} shows a clear hardness scaling trend: SR0 decreases as the number of goals \(n\) increases and the slack ratio \(\rho\) decreases. Even under this trend, the ARS-FS variants retain strong solution-finding ability throughout the grid. In the hardest setting \((\rho=0,n=6)\), Basic Planning (B) achieves only \(61.0\%\) SR0, whereas ARS-FS-L, ARS-FS, and ARS-FS-H reach \(87.0\%\), \(90.0\%\), and \(92.0\%\), respectively. The same ordering appears over the full grid, where the mean SR0 values are \(85.3\%\) for B, \(97.8\%\) for ARS-FS-L, \(98.6\%\) for ARS-FS, and \(98.9\%\) for ARS-FS-H. These results show that expanding the search branches and increasing the first-solution budget consistently improve the probability of recovering a complete allocation, with the highest-budget variant still maintaining a strong solve rate in the most difficult regime.

The lower panel shows the corresponding computational cost. Planning time grows with \(n\) and increases more sharply as \(\rho\) becomes smaller, indicating that tighter shared deadlines require substantially more search effort. The three ARS-FS variants also follow the expected budget ordering, with ARS-FS-L consistently fastest among them and ARS-FS-H slowest. In the hardest setting \((\rho=0,n=6)\), the mean planning times are \(9.2\) s for B, \(43.3\) s for ARS-FS-L, \(72.7\) s for ARS-FS, and \(153.8\) s for ARS-FS-H. Averaged over the full grid, B requires \(3.15\) s per task, compared with \(9.84\), \(16.35\), and \(27.53\) s for ARS-FS-L, ARS-FS, and ARS-FS-H.

Overall, the results show a clean budget-performance trade-off: DAG-STL already maintains good solution-finding ability at relatively low search budgets, and this ability can be further strengthened by investing a moderate amount of additional search computation.

\section{Limitations}\label{sec:limitations}
\noindent Despite the strong empirical performance across both navigation and manipulation domains, several aspects remain outside the current scope of DAG-STL.

\subsection{Coverage and Formula Scope}
DAG-STL gains tractability by introducing a structured intermediate representation together with finite-budget search. The core decomposition is developed for the STL fragment described in Section~\ref{sec:STLdecomposition}, and disjunctive formulas are handled in practice by solving disjunction-free branches separately after preprocessing, a tractability strategy that is common in decomposition- and rewrite-based STL synthesis~\cite{leahy2023rewrite,yu2024continuous,kapoor2024safe}. In addition, both the basic allocator and ARS explore only finitely many sampled waypoint and timing hypotheses under a finite computation budget. These choices are precisely what make the framework scalable in the offline unknown-dynamics setting, but they do not provide a completeness guarantee over the full STL language or over all dynamically feasible trajectories compatible with the underlying system.

At the same time, the experiments suggest that the resulting conservativeness is often moderate rather than prohibitive in the target regime. In the controlled comparison of Section~\ref{sec:exp_custom}, DAG-STL recovers most tasks on the model-solvable reference set, and in the supplementary combinatorial stress test it retains strong solution-finding ability even under deliberately tight deadlines, with higher-budget ARS variants further improving success. Future work may reduce the remaining gap through native handling of disjunctive choices, richer proposal mechanisms for timed waypoints, and more systematic anytime search over allocation hypotheses.

\subsection{Computational Complexity and Recurrent Obligations}
As discussed in Section~\ref{sec:complexity}, the theoretical complexity of the proposed planning pipeline can grow rapidly with the number of decomposed progress conditions, time variables, and sampled allocation hypotheses. This is an expected consequence of converting a temporally structured STL formula into a coupled decomposition--allocation problem and then exploring the resulting timed waypoint space under finite budgets. In practice, the anytime and heuristic components make this complexity manageable for the milestone-like long-horizon tasks targeted in this paper, where the formula structure is dominated by a finite set of temporally ordered reachability and invariance requirements.

However, the same representation is less naturally suited to specifications dominated by high-frequency recurrent obligations, especially dense outer-\(\G\F\)-type patterns. The recurrent reachability case studies in Section~\ref{sec:casestudy} show that DAG-STL can handle such obligations to a finite extent, but they also illustrate how recurrent structure can induce many progress conditions and time variables. For more complex or higher-frequency recurrent specifications, the resulting allocation problem may become computationally demanding, and more specialized symbolic compression, recurrence-aware decomposition, or receding-horizon treatment would likely be needed.

\subsection{Task-Space Abstraction and Execution Interface}
Like many hierarchical planners, DAG-STL reasons primarily over state trajectories in a task-relevant planning space and delegates low-level action realization to a PD controller or a learned inverse dynamics model, as described in Section~\ref{sec:action_control}; similar task-space abstractions are also common in high-dimensional generative planning and control~\cite{luo2025generative}. This abstraction works well in Maze2D, AntMaze, and Cube, but it also means that fine-grained actuator limits, contact mode changes, and full-configuration feasibility are captured only indirectly through offline data support and execution-time correction. Similarly, although the experiments span both navigation and manipulation, the predicates considered here are still predominantly geometric region predicates defined in relatively low-dimensional task spaces.

A natural next step is therefore to tighten the coupling between logic-level planning and low-level realization, for example through stronger action-recovery models, receding-horizon execution, or richer predicate classes involving contact, force, or hybrid symbolic state.

\subsection{Plan Quality and Learned Timing Models}
The present framework is designed primarily to find executable STL-satisfying plans rather than to optimize a global secondary objective. ARS already improves nominal plan quality by refining multiple allocation hypotheses, but the method does not provide formal optimality with respect to makespan, temporal slack, energy, or robustness margin. In addition, the time predictor remains a learned approximation derived from offline trajectories. When the dataset under-represents rare detours, long-horizon transitions, or evaluation-time distribution shift, the predicted transition lengths may still become biased, which in turn can affect allocation quality.

Our experiments indicate that this dependence is manageable in the target regime. The analysis in Section~\ref{exp:time_predictor} shows that moderate timing error is often absorbed by segment-wise generation and closed-loop execution, and the gains of ARS indicate that search can compensate for part of the residual predictor bias. Still, there is clear room for improvement through uncertainty-calibrated time prediction, slack-aware or multi-objective allocation scores, and data augmentation that improves coverage of rare but feasible transitions~\cite{li2024diffstitch,lee2025state,wang2023optimal,myers2025offline,liu2025vh}.

\section{Conclusion}
This paper presented DAG-STL, a hierarchical framework for offline trajectory planning under Signal Temporal Logic (STL) constraints when the underlying dynamics are unknown. The central idea is to separate logical reasoning from trajectory realization. Following this principle, DAG-STL first decomposes an STL formula into progress conditions with shared timing constraints, then allocates timed waypoints, and finally completes the plan with diffusion-based trajectory generation. We further enhance this pipeline with a rollout-free dynamic consistency metric, an anytime refinement search procedure, and hierarchical online replanning, yielding a planning framework that is both logically structured and more aware of execution feasibility.

The experiments support this design from several complementary perspectives. In Maze2D, DAG-STL substantially outperforms direct robustness-guided diffusion on complex long-horizon tasks. In both Maze2D and AntMaze, ARS and online replanning improve reliability on difficult templates by refining or repairing timed allocations instead of relying on a single nominal plan. In the OGBench Cube domain, the framework generalizes beyond navigation and supports temporally structured manipulation using only task-agnostic offline data. In the custom-built environment, DAG-STL recovers most tasks in the model-solvable reference set while maintaining a clear computational advantage over direct optimization as temporal complexity increases.

Future work will focus on further reducing conservativeness and broadening scope, including uncertainty-calibrated time prediction, more systematic multi-objective allocation refinement, tighter coupling between task-level planning and low-level realization, and richer STL predicates beyond the current geometric task-space setting.
% \begin{acks}

% \end{acks}

\section*{Author Contributions}
Ruijia Liu: Conceptualization, methodology, software, investigation, and writing--original draft. Ancheng Hou: Software and investigation. Xiao Yu: Supervision and writing--review \& editing. Xiang Yin: Conceptualization, methodology, supervision, and writing--review \& editing.

\section*{Statements and Declarations}

\subsection*{Ethical considerations}
Not applicable. This study does not involve human participants, human data, human tissue, or animal subjects.

\subsection*{Consent to participate}
Not applicable. This article does not involve human participants or personal data requiring consent to participate.

\subsection*{Consent for publication}
Not applicable. This article does not contain any individual person's data in any form, including any individual details, images, or videos.

\subsection*{Declaration of conflicting interest}
The author(s) declared no potential conflicts of interest with respect to the research, authorship, and/or publication of this article.

\subsection*{Funding statement}
This work was supported by the National Natural Science Foundation of China (62173226,92367203,62061136004).

\subsection*{Data availability}
The code and experimental assets supporting the findings of this study are available at \url{https://cps-sjtu.github.io/DAG-STL}.

\subsection*{Use of AI-assisted language editing}
ChatGPT (OpenAI) was used to improve the language and readability of the manuscript. The authors reviewed and verified all AI-assisted text, take full responsibility for the final content and citations, and acknowledge that large language models may generate biased, erroneous, or incomplete outputs.

\bibliographystyle{SageH}
\bibliography{root}

@inproceedings{wang2025multi,
  title={Multi-agent reinforcement learning guided by signal temporal logic specifications},
  author={Wang, Jiangwei and Yang, Shuo and An, Ziyan and Han, Songyang and Zhang, Zhili and Mangharam, Rahul and Ma, Meiyi and Miao, Fei},
  booktitle={IEEE/RSJ International Conference on Intelligent Robots and Systems},
  pages={6048--6054},
  year={2025}
}

@inproceedings{sun2022dnn,
  title={Out-of-distribution detection with deep nearest neighbors},
  author={Sun, Yiyou and Ming, Yifei and Zhu, Xiaojin and Li, Yixuan},
  booktitle={International conference on machine learning},
  pages={20827--20840},
  year={2022},
  organization={PMLR}
}

@inproceedings{breunig2000lof,
  title={LOF: identifying density-based local outliers},
  author={Breunig, Markus M and Kriegel, Hans-Peter and Ng, Raymond T and Sander, J{\"o}rg},
  booktitle={Proceedings of the 2000 ACM SIGMOD international conference on Management of data},
  pages={93--104},
  year={2000}
}

@inproceedings{ramaswamy2000outliers,
  title={Efficient algorithms for mining outliers from large data sets},
  author={Ramaswamy, Sridhar and Rastogi, Rajeev and Shim, Kyuseok},
  booktitle={Proceedings of the 2000 ACM SIGMOD international conference on Management of data},
  pages={427--438},
  year={2000}
}

@inproceedings{xiao2023safediffuser,
  author       = {Wei Xiao and
                  Tsun{-}Hsuan Wang and
                  Chuang Gan and
                  Ramin M. Hasani and
                  Mathias Lechner and
                  Daniela Rus},
  title        = {SafeDiffuser: Safe Planning with Diffusion Probabilistic Models},
  booktitle    = {The Thirteenth International Conference on Learning Representations},
  year         = {2025},
}

@article{kapoor2020model,
  title={Model-based reinforcement learning from signal temporal logic specifications},
  author={Kapoor, Parv and Balakrishnan, Anand and Deshmukh, Jyotirmoy V},
  journal={arXiv preprint arXiv:2011.04950},
  year={2020}
}

@article{belta2019formal,
  title={Formal methods for control synthesis: An optimization perspective},
  author={Belta, Calin and Sadraddini, Sadra},
  journal={Annual Review of Control, Robotics, and Autonomous Systems},
  volume={2},
  number={1},
  pages={115--140},
  year={2019},
  publisher={Annual Reviews}
}

@inproceedings{pant2017smooth,
  title={Smooth operator: Control using the smooth robustness of temporal logic},
  author={Pant, Yash Vardhan and Abbas, Houssam and Mangharam, Rahul},
  booktitle={IEEE Conference on Control Technology and Applications},
  pages={1235--1240},
  year={2017},
  organization={IEEE}
}

@article{cardona2025stl,
  title={STL and wSTL control synthesis: A disjunction-centric mixed-integer linear programming approach},
  author={Cardona, Gustavo A and Kamale, Disha and Vasile, Cristian-Ioan},
  journal={Nonlinear Analysis: Hybrid Systems},
  volume={56},
  pages={101576},
  year={2025},
  publisher={Elsevier}
}

@article{armanini2023soft,
  title={Soft robots modeling: A structured overview},
  author={Armanini, Costanza and Boyer, Fr{\'e}d{\'e}ric and Mathew, Anup Teejo and Duriez, Christian and Renda, Federico},
  journal={IEEE Transactions on Robotics},
  volume={39},
  number={3},
  pages={1728--1748},
  year={2023},
  publisher={IEEE}
}

@inproceedings{venkataraman2020tractable,
  title={Tractable reinforcement learning of signal temporal logic objectives},
  author={Venkataraman, Harish and Aksaray, Derya and Seiler, Peter},
  booktitle={Learning for dynamics and control},
  pages={308--317},
  year={2020},
  organization={PMLR}
}

@article{salunkhe2025kinematic,
  title={Kinematic issues in 6R cuspidal robots, guidelines for path planning and deciding cuspidality},
  author={Salunkhe, Durgesh Haribhau and Marauli, Tobias and M{\"u}ller, Andreas and Chablat, Damien and Wenger, Philippe},
  journal={The International Journal of Robotics Research},
  volume={44},
  number={6},
  pages={1035--1054},
  year={2025},
  publisher={SAGE Publications Sage UK: London, England}
}

@article{kurtz2020trajectory,
  title={Trajectory optimization for high-dimensional nonlinear systems under STL specifications},
  author={Kurtz, Vince and Lin, Hai},
  journal={IEEE Control Systems Letters},
  volume={5},
  number={4},
  pages={1429--1434},
  year={2020},
  publisher={IEEE}
}

@article{kress2018synthesis,
  title={Synthesis for robots: Guarantees and feedback for robot behavior},
  author={Kress-Gazit, Hadas and Lahijanian, Morteza and Raman, Vasumathi},
  journal={Annual Review of Control, Robotics, and Autonomous Systems},
  volume={1},
  number={1},
  pages={211--236},
  year={2018},
  publisher={Annual Reviews}
}

@book{lavalle2006planning,
  title={Planning algorithms},
  author={LaValle, Steven M},
  year={2006},
  publisher={Cambridge University Press}
}

@inproceedings{puranic2021learning,
  title={Learning from demonstrations using signal temporal logic},
  author={Puranic, Aniruddh and Deshmukh, Jyotirmoy and Nikolaidis, Stefanos},
  booktitle={Conference on Robot Learning},
  pages={2228--2242},
  year={2021} 
}

@article{sewlia2022cooperative,
  title={Cooperative object manipulation under signal temporal logic tasks and uncertain dynamics},
  author={Sewlia, Mayank and Verginis, Christos K and Dimarogonas, Dimos V},
  journal={IEEE Robotics and Automation Letters},
  volume={7},
  number={4},
  pages={11561--11568},
  year={2022},
  publisher={IEEE}
}

@article{gu2025robust,
  title={Robust-locomotion-by-logic: Perturbation-resilient bipedal locomotion via signal temporal logic guided model predictive control},
  author={Gu, Zhaoyuan and Zhao, Yuntian and Chen, Yipu and Guo, Rongming and Leestma, Jennifer K and Sawicki, Gregory S and Zhao, Ye},
  journal={IEEE Transactions on Robotics},
  year={2025},
  publisher={IEEE}
}

@inproceedings{yang2025stlgame,
  title={STLGame: Signal Temporal Logic Games in Adversarial Multi-Agent Systems},
  author={Yang, Shuo and Zheng, Hongrui and Vasile, Cristian-Ioan and Pappas, George and Mangharam, Rahul},
  booktitle={7th Annual Learning for Dynamics \& Control Conference},
  pages={181--196},
  year={2025},
  organization={PMLR}
}

@article{yin2024formal,
  title={Formal synthesis of controllers for safety-critical autonomous systems: Developments and challenges},
  author={Yin, Xiang and Gao, Bingzhao and Yu, Xiao},
  journal={Annual Reviews in Control},
  volume={57},
  pages={100940},
  year={2024},
  publisher={Elsevier}
}

@article{kapoor2025stlcg++,
  title={{STLCG}++: A masking approach for differentiable signal temporal logic specification},
  author={Kapoor, Parv and Mizuta, Kazuki and Kang, Eunsuk and Leung, Karen},
  journal={IEEE Robotics and Automation Letters},
  year={2025}
}

@article{silano2021power,
  title={Power line inspection tasks with multi-aerial robot systems via signal temporal logic specifications},
  author={Silano, Giuseppe and Baca, Tomas and Penicka, Robert and Liuzza, Davide and Saska, Martin},
  journal={IEEE Robotics and Automation Letters},
  volume={6},
  number={2},
  pages={4169--4176},
  year={2021},
  publisher={IEEE}
}

@article{luckcuck2019formal,
  title={Formal specification and verification of autonomous robotic systems: A survey},
  author={Luckcuck, Matt and Farrell, Marie and Dennis, Louise A and Dixon, Clare and Fisher, Michael},
  journal={ACM Computing Surveys},
  volume={52},
  number={5},
  pages={1--41},
  year={2019},
  publisher={ACM New York, NY, USA}
}

@inproceedings{ogbench_park2025,
  title={OGBench: Benchmarking Offline Goal-Conditioned RL},
  author={Park, Seohong and Frans, Kevin and Eysenbach, Benjamin and Levine, Sergey},
  booktitle={International Conference on Learning Representations},
  year={2025},
}

@article{li2024derivative,
  title={Derivative-Free Guidance in Continuous and Discrete Diffusion Models with Soft Value-Based Decoding},
  author={Li, Xiner and Zhao, Yulai and Wang, Chenyu and Scalia, Gabriele and Eraslan, Gokcen and Nair, Surag and Biancalani, Tommaso and Regev, Aviv and Levine, Sergey and Uehara, Masatoshi},
  journal={arXiv preprint arXiv:2408.08252},
  year={2024}
}

@article{yu2024continuous,
  title={Continuous-time control synthesis under nested signal temporal logic specifications},
  author={Yu, Pian and Tan, Xiao and Dimarogonas, Dimos V},
  journal={IEEE Transactions on Robotics},
  volume={40},
  pages={2272--2286},
  year={2024},
  publisher={IEEE}
}

@inproceedings{huang2025diffusion,
  title={Diffusion Models as Optimizers for Efficient Planning in Offline RL},
  author={Huang, Renming and Pei, Yunqiang and Wang, Guoqing and Zhang, Yangming and Yang, Yang and Wang, Peng and Shen, Hengtao},
  booktitle={European Conference on Computer Vision},
  pages={1--17},
  year={2025},
  organization={Springer}
}

@article{bartocci2018specification,
  title={Specification-based monitoring of cyber-physical systems: a survey on theory, tools and applications},
  author={Bartocci, Ezio and Deshmukh, Jyotirmoy and Donz{\'e}, Alexandre and Fainekos, Georgios and Maler, Oded and Ni{\v{c}}kovi{\'c}, Dejan and Sankaranarayanan, Sriram},
  journal={Lectures on Runtime Verification: Introductory and Advanced Topics},
  pages={135--175},
  year={2018},
  publisher={Springer}
}

@inproceedings{kapoor2024safe,
  title={Safe Planning Through Incremental Decomposition of Signal Temporal Logic Specifications},
  author={Kapoor, Parv and Kang, Eunsuk and Meira-G{\'o}es, R{\^o}mulo},
  booktitle={NASA Formal Methods Symposium},
  pages={377--396},
  year={2024},
  organization={Springer}
}

@article{fu2020d4rl,
  title={D4rl: Datasets for deep data-driven reinforcement learning},
  author={Fu, Justin and Kumar, Aviral and Nachum, Ofir and Tucker, George and Levine, Sergey},
  journal={arXiv preprint arXiv:2004.07219},
  year={2020}
}

@inproceedings{mizuta2024cobl,
  title={Cobl-diffusion: Diffusion-based conditional robot planning in dynamic environments using control barrier and lyapunov functions},
  author={Mizuta, Kazuki and Leung, Karen},
  booktitle={2024 IEEE/RSJ International Conference on Intelligent Robots and Systems},
  pages={13801--13808},
  year={2024},
  organization={IEEE}
}

@article{christopher2024constrained,
  title={Constrained synthesis with projected diffusion models},
  author={Christopher, Jacob K and Baek, Stephen and Fioretto, Ferdinando},
  journal={Advances in Neural Information Processing Systems},
  volume={37},
  pages={89307--89333},
  year={2024}
}

@inproceedings{leahy2023rewrite,
  title={Rewrite-based decomposition of signal temporal logic specifications},
  author={Leahy, Kevin and Mann, Makai and Vasile, Cristian-Ioan},
  booktitle={NASA Formal Methods Symposium},
  pages={224--240},
  year={2023},
  organization={Springer}
}

@inproceedings{yu2023model,
  title={Model predictive control for signal temporal logic specifications with time interval decomposition},
  author={Yu, Xinyi and Wang, Chuwei and Yuan, Dingran and Li, Shaoyuan and Yin, Xiang},
  booktitle={2023 62nd IEEE Conference on Decision and Control},
  pages={7849--7855},
  year={2023},
  organization={IEEE}
}

@inproceedings{maler2004monitoring,
  title={Monitoring temporal properties of continuous signals},
  author={Maler, Oded and Nickovic, Dejan},
  booktitle={International Symposium on Formal Techniques in Real-Time and Fault-Tolerant Systems},
  pages={152--166},
  year={2004},
  organization={Springer}
}

@inproceedings{sadraddini2015robust,
  title={Robust temporal logic model predictive control},
  author={Sadraddini, Sadra and Belta, Calin},
  booktitle={53rd Annual Allerton Conference on Communication, Control, and Computing},
  pages={772--779},
  year={2015},
  organization={IEEE}
}

@inproceedings{janner2022planning,
  title={Planning with Diffusion for Flexible Behavior Synthesis},
  author={Janner, Michael and Du, Yilun and Tenenbaum, Joshua and Levine, Sergey},
  booktitle={International Conference on Machine Learning},
  pages={9902--9915},
  year={2022},
  organization={PMLR}
}

@inproceedings{raman2014model,
  title={Model predictive control from signal temporal logic specifications: A case study},
  author={Raman, Vasumathi and Maasoumy, Mehdi and Donz{\'e}, Alexandre},
  booktitle={Proceedings of the 4th ACM SIGBED International Workshop on Design, Modeling, and Evaluation of Cyber-Physical Systems},
  pages={52--55},
  year={2014}
}

@article{sun2022multi,
  title={Multi-agent motion planning from signal temporal logic specifications},
  author={Sun, Dawei and Chen, Jingkai and Mitra, Sayan and Fan, Chuchu},
  journal={IEEE Robotics and Automation Letters},
  volume={7},
  number={2},
  pages={3451--3458},
  year={2022},
  publisher={IEEE}
}

@article{kurtz2022mixed,
  title={Mixed-integer programming for signal temporal logic with fewer binary variables},
  author={Kurtz, Vincent and Lin, Hai},
  journal={IEEE Control Systems Letters},
  volume={6},
  pages={2635--2640},
  year={2022},
  publisher={IEEE}
}

@article{he2024scalable,
  title={Scalable Signal Temporal Logic Guided Reinforcement Learning via Value Function Space Optimization},
  author={He, Yiting and Liu, Peiran and Ji, Yiding},
  journal={arXiv preprint arXiv:2408.01923},
  year={2024}
}

@inproceedings{wang2024synthesis,
  title={Synthesis of temporally-robust policies for signal temporal logic tasks using reinforcement learning},
  author={Wang, Siqi and Li, Shaoyuan and Yin, Li and Yin, Xiang},
  booktitle={IEEE International Conference on Robotics and Automation},
  pages={10503--10509},
  year={2024},
  organization={IEEE}
}

@inproceedings{kalagarla2021model,
  title={Model-free reinforcement learning for optimal control of Markov decision processes under signal temporal logic specifications},
  author={Kalagarla, Krishna C and Jain, Rahul and Nuzzo, Pierluigi},
  booktitle={60th IEEE Conference on Decision and Control},
  pages={2252--2257},
  year={2021},
  organization={IEEE}
}

@inproceedings{zhong2023guided,
  title={Guided conditional diffusion for controllable traffic simulation},
  author={Zhong, Ziyuan and Rempe, Davis and Xu, Danfei and Chen, Yuxiao and Veer, Sushant and Che, Tong and Ray, Baishakhi and Pavone, Marco},
  booktitle={IEEE International Conference on Robotics and Automation},
  pages={3560--3566},
  year={2023},
  organization={IEEE}
}

@article{gilpin2020smooth,
  title={A smooth robustness measure of signal temporal logic for symbolic control},
  author={Gilpin, Yann and Kurtz, Vince and Lin, Hai},
  journal={IEEE Control Systems Letters},
  volume={5},
  number={1},
  pages={241--246},
  year={2020},
  publisher={IEEE}
}

@inproceedings{
ajay2022conditional,
title={Is Conditional Generative Modeling all you need for Decision Making?},
author={Anurag Ajay and Yilun Du and Abhi Gupta and Joshua B. Tenenbaum and Tommi S. Jaakkola and Pulkit Agrawal},
booktitle={The Eleventh International Conference on Learning Representations },
year={2023},
}

@article{ikemoto2022deep,
  title={Deep reinforcement learning under signal temporal logic constraints using Lagrangian relaxation},
  author={Ikemoto, Junya and Ushio, Toshimitsu},
  journal={IEEE Access},
  volume={10},
  pages={114814--114828},
  year={2022},
  publisher={IEEE}
}

@article{chi2023diffusion,
  title={Diffusion policy: Visuomotor policy learning via action diffusion},
  author={Chi, Cheng and Xu, Zhenjia and Feng, Siyuan and Cousineau, Eric and Du, Yilun and Burchfiel, Benjamin and Tedrake, Russ and Song, Shuran},
  journal={The International Journal of Robotics Research},
  pages={02783649241273668},
  year={2023},
  publisher={SAGE Publications Sage UK: London, England}
}

@inproceedings{agrawal2016learning,
  title={Learning to poke by poking: experiential learning of intuitive physics},
  author={Agrawal, Pulkit and Nair, Ashvin and Abbeel, Pieter and Malik, Jitendra and Levine, Sergey},
  booktitle={Proceedings of the 30th International Conference on Neural Information Processing Systems},
  pages={5092--5100},
  year={2016}
}

@article{leung2023backpropagation,
  title={Backpropagation through signal temporal logic specifications: Infusing logical structure into gradient-based methods},
  author={Leung, Karen and Ar{\'e}chiga, Nikos and Pavone, Marco},
  journal={The International Journal of Robotics Research},
  volume={42},
  number={6},
  pages={356--370},
  year={2023},
  publisher={SAGE Publications Sage UK: London, England}
}

@article{zhang2025decomposition,
  title={Decomposition-Based MPC for Uncertain Systems With Nested Signal Temporal Logic Specifications},
  author={Zhang, Jiarui and Lu, Penghong and Chen, Gang},
  journal={IEEE Control Systems Letters},
  year={2025},
  publisher={IEEE}
}

@inproceedings{balakrishnan2019structured,
  title={Structured reward functions using STL},
  author={Balakrishnan, Anand and Deshmukh, Jyotirmoy V},
  booktitle={Proceedings of the 22nd ACM International Conference on Hybrid Systems: Computation and Control},
  pages={270--271},
  year={2019}
}

@article{charitidou2021signal,
  title={Signal temporal logic task decomposition via convex optimization},
  author={Charitidou, Maria and Dimarogonas, Dimos V},
  journal={IEEE Control Systems Letters},
  volume={6},
  pages={1238--1243},
  year={2021},
  publisher={IEEE}
}

@inproceedings{ho2023sampling,
  title={Sampling-based Approach to Robust STL Synthesis for Complex Systems under Uncertainty},
  author={Ho, Qi Heng and Ilyes, Roland and Sunberg, Zachary and Lahijanian, Morteza},
  booktitle={Proceedings of the 26th ACM International Conference on Hybrid Systems: Computation and Control},
  pages={1--2},
  year={2023}
}

@inproceedings{aksaray2016q,
  title={Q-learning for robust satisfaction of signal temporal logic specifications},
  author={Aksaray, Derya and Jones, Austin and Kong, Zhaodan and Schwager, Mac and Belta, Calin},
  booktitle={IEEE 55th Conference on Decision and Control},
  pages={6565--6570},
  year={2016},
  organization={IEEE}
}

@inproceedings{ilyes2023stochastic,
  title={Stochastic robustness interval for motion planning with signal temporal logic},
  author={Ilyes, Roland B and Ho, Qi Heng and Lahijanian, Morteza},
  booktitle={IEEE International Conference on Robotics and Automation},
  pages={5716--5722},
  year={2023},
  organization={IEEE}
}

@inproceedings{dawson2022robust,
  title={Robust counterexample-guided optimization for planning from differentiable temporal logic},
  author={Dawson, Charles and Fan, Chuchu},
  booktitle={IEEE/RSJ International Conference on Intelligent Robots and Systems},
  pages={7205--7212},
  year={2022},
  organization={IEEE}
}

@inproceedings{li2024diffstitch,
author = {Li, Guanghe and Shan, Yixiang and Zhu, Zhengbang and Long, Ting and Zhang, Weinan},
title = {DiffStitch: boosting offline reinforcement learning with diffusion-based trajectory stitching},
year = {2024},
booktitle = {Proceedings of the 41st International Conference on Machine Learning},
articleno = {1148},
numpages = {13},
}

@article{rockafellar2000optimization,
  title={Optimization of conditional value-at-risk},
  author={Rockafellar, R Tyrrell and Uryasev, Stanislav and others},
  journal={Journal of risk},
  volume={2},
  pages={21--42},
  year={2000}
}

@article{cover1967nearest,
  title={Nearest neighbor pattern classification},
  author={Cover, Thomas and Hart, Peter},
  journal={IEEE transactions on information theory},
  volume={13},
  number={1},
  pages={21--27},
  year={1967},
  publisher={IEEE}
}

@inproceedings{luo2025generative,
  title={Generative Trajectory Stitching through Diffusion Composition},
  author={Luo, Yunhao and Mishra, Utkarsh Aashu and Du, Yilun and Xu, Danfei},
  booktitle={The Thirty-ninth Annual Conference on Neural Information Processing Systems},
  year={2025}
}

@article{meng2024diverse,
  title={Diverse controllable diffusion policy with signal temporal logic},
  author={Meng, Yue and Fan, Chuchu},
  journal={IEEE Robotics and Automation Letters},
  volume={9},
  number={10},
  pages={8354--8361},
  year={2024},
  publisher={IEEE}
}

@article{feng2024ltldog,
  title={LTLDoG: Satisfying temporally-extended symbolic constraints for safe diffusion-based planning},
  author={Feng, Zeyu and Luan, Hao and Goyal, Pranav and Soh, Harold},
  journal={IEEE Robotics and Automation Letters},
  year={2024},
  publisher={IEEE}
}

@inproceedings{feng2025diffusion,
  title={Diffusion meets options: Hierarchical generative skill composition for temporally-extended tasks},
  author={Feng, Zeyu and Luan, Hao and Ma, Kevin Yuchen and Soh, Harold},
  booktitle={2025 IEEE International Conference on Robotics and Automation},
  pages={10854--10860},
  year={2025},
  organization={IEEE}
}

@article{liu2025vh,
  title={VH-Diffuser: Variable Horizon Diffusion Planner for Time-Aware Goal-Conditioned Trajectory Planning},
  author={Liu, Ruijia and Hou, Ancheng and Li, Shaoyuan and Yin, Xiang},
  journal={arXiv preprint arXiv:2509.11930},
  year={2025}
}

@inproceedings{myers2025offline,
  title={Offline Goal-conditioned Reinforcement Learning with Quasimetric Representations},
  author={Myers, Vivek and Zheng, Bill and Eysenbach, Benjamin and Levine, Sergey},
  booktitle={The Thirty-ninth Annual Conference on Neural Information Processing Systems},
  year={2025}
}

@inproceedings{wang2023optimal,
  title={Optimal goal-reaching reinforcement learning via quasimetric learning},
  author={Wang, Tongzhou and Torralba, Antonio and Isola, Phillip and Zhang, Amy},
  booktitle={International Conference on Machine Learning},
  pages={36411--36430},
  year={2023},
  organization={PMLR}
}

@inproceedings{carvalho2023motion,
  title={Motion planning diffusion: Learning and planning of robot motions with diffusion models},
  author={Carvalho, Joao and Le, An T and Baierl, Mark and Koert, Dorothea and Peters, Jan},
  booktitle={IEEE/RSJ International Conference on Intelligent Robots and Systems},
  pages={1916--1923},
  year={2023},
  organization={IEEE}
}

@article{botteghi2023trajectory,
  title={Trajectory Generation, Control, and Safety with Denoising Diffusion Probabilistic Models},
  author={Botteghi, Nicol{\`o} and Califano, Federico and Poel, Mannes and Brune, Christoph},
  journal={arXiv preprint arXiv:2306.15512},
  year={2023}
}

@inproceedings{donze2010robust,
  title={Robust satisfaction of temporal logic over real-valued signals},
  author={Donz{\'e}, Alexandre and Maler, Oded},
  booktitle={International conference on formal modeling and analysis of timed systems},
  pages={92--106},
  year={2010},
  organization={Springer}
}

@article{suh2025dexterous,
  title={Dexterous contact-rich manipulation via the contact trust region},
  author={Suh, HJ Terry and Pang, Tao and Zhao, Tong and Tedrake, Russ},
  journal={The International Journal of Robotics Research},
  pages={02783649251398875},
  year={2025} 
}

@article{yuasa2026neuro,
  title={Neuro-Symbolic Generation of Explanations for Robot Policies With Weighted Signal Temporal Logic},
  author={Yuasa, Mikihisa and Sreenivas, Ramavarapu S and Tran, Huy T},
  journal={IEEE Robotics and Automation Letters},
  year={2026},
  publisher={IEEE}
}

@inproceedings{liu2025zeroshot,
title={Zero-Shot Trajectory Planning for Signal Temporal Logic Tasks},
author={Ruijia Liu and Ancheng Hou and Xiao Yu and Xiang Yin},
booktitle={The Thirty-ninth Annual Conference on Neural Information Processing Systems},
year={2025},
}

@inproceedings{lee2025state,
  title={State-Covering Trajectory Stitching for Diffusion Planners},
  author={Lee, Kyowoon and Choi, Jaesik},
  booktitle={Advances in Neural Information Processing Systems},
  year={2025},
}

@inproceedings{
zheng2024safe,
title={Safe Offline Reinforcement Learning with Feasibility-Guided Diffusion Model},
author={Yinan Zheng and Jianxiong Li and Dongjie Yu and Yujie Yang and Shengbo Eben Li and Xianyuan Zhan and Jingjing Liu},
booktitle={The Twelfth International Conference on Learning Representations},
year={2024},
}

@inproceedings{ho2020denoising,
  title={Denoising diffusion probabilistic models},
  author={Ho, Jonathan and Jain, Ajay and Abbeel, Pieter},
  booktitle={Proceedings of the 34th International Conference on Neural Information Processing Systems},
  pages={6840--6851},
  year={2020}
}

\appendix

\renewcommand{\theequation}{\Alph{section}.\arabic{equation}}
\renewcommand{\thefigure}{\Alph{section}.\arabic{figure}}
\renewcommand{\thetable}{\Alph{section}.\arabic{table}}
\counterwithin{equation}{section}
\counterwithin{figure}{section}
\counterwithin{table}{section}
\vspace{20pt}
{\large\textbf{Appendices}}
\section{Implementation Details}\label{apx:implementation}
This appendix summarizes the remaining implementation choices used in the experiments. Shared architectural and training settings have already been described in Sections~\ref{sec:implementation}--\ref{sec:robustness}; here we record the environment-specific hyperparameters and the additional details most relevant to reproduction.

\subsection{Environment-Specific Hyperparameter Summary}
Table~\ref{tab:impl_hparam} summarizes the main planning, refinement, and execution hyperparameters for each environment or environment family.
Here, \(\eta\) is the trajectory resolution factor introduced in Section~\ref{sec:robustness}; \(k\) is the synchronization ratio between one planning step and low-level control updates from Section~\ref{sec:action_control}; \(N_{\max}\) is the maximum number of waypoint-sampling attempts; \(K\) is the per-branch candidate capacity used in \texttt{MultiHypothesisAssign}; \(N_{\mathrm{iter}}\) and \(N_{\mathrm{seq}}\) are the ARS refinement budgets; and \((\varepsilon_{\mathrm{loc}},\varepsilon_{\mathrm{glob}})\) are the execution-error thresholds that determine whether the replanner performs local repair or global replanning.
The same hyperparameter setting is used across the three Maze2D layouts, so they are reported in a single row.

\begin{table}[htb]
\centering
\setlength{\tabcolsep}{4pt}
\caption{Environment-specific hyperparameters for DAG-STL.}
\label{tab:impl_hparam}
\begin{tabular}{lcccc}
\toprule
\textbf{Parameter} & \textbf{Maze2D} & \textbf{Ant} & \textbf{Cube} & \textbf{Custom} \\
\midrule
\(\eta\) & 8 & 8 & 4 & 4 \\
\(k\) & 1 & 2 & 1 & 1 \\
\(N_{\max}\) & 1 & 1 & 1 & 1 \\
\(K\) & 5 & 5 & 5 & 5 \\
\(N_{\mathrm{iter}}\) & 100 & 100 & - & 100 \\
\(N_{\mathrm{seq}}\) & 6 & 6 & - & 3 \\
\((\varepsilon_{\mathrm{loc}},\varepsilon_{\mathrm{glob}})\) & (0.4, 1.0) & (1.5, 4.0) & - & - \\
\bottomrule
\end{tabular}
\end{table}

\subsection{Learned Models}
\paragraph{Trajectory Generator}
The trajectory generator follows the DDPM-based \texttt{Diffuser} framework of~\cite{janner2022planning} with a U-Net backbone. During training, we adopt the variable-horizon strategy of~\cite{liu2025vh}: for each environment, we specify a maximum segment length \(H\) shown in Table~\ref{tab:impl_generator}, and each training sample is formed by randomly cropping a subtrajectory of length \(h\le H\) such that \(h\) is a multiple of the environment-dependent resolution factor \(\eta\). This matches the segment lengths used at test time and improves length generalization. At inference time, the segment length is controlled by changing the shape of the initial noise tensor so that the generated trajectory has the required number of states.

\begin{table}[t]
\centering
\small
\setlength{\tabcolsep}{5pt}
\caption{Environment-dependent settings of the trajectory generator.}
\label{tab:impl_generator}
\begin{tabular}{lcc}
\toprule
\textbf{Environment} & \(\boldsymbol{H}\) & \textbf{Diffusion Steps} \\
\midrule
Maze2D (U/M/L) & 128 / 384 / 512 & 64 / 128 / 192 \\
AntMaze & 384 & 192 \\
Cube & 128 & 64 \\
Custom & 64 & 64 \\
\bottomrule
\end{tabular}
\end{table}

At test time, invariance progress conditions are enforced by projection-based constrained denoising inspired by~\cite{christopher2024constrained}. After each selected denoising step, the current sample is replaced by its nearest point in the constraint set,
\[
\Pi_{\mathcal C}(\bs{\tau}) \;=\; \arg\min_{\hat{\bs{\tau}}\in\mathcal C}\|\hat{\bs{\tau}}-\bs{\tau}\|_2^2,
\]
where \(\mathcal C\) denotes the set of trajectories satisfying the active invariance conditions for the current segment. In the experiments of this paper, these constraints reduce to circular-region predicates and therefore admit closed-form projection. We also implemented a solver-based fallback for more general constraint sets. Following common practice, projection is applied only in the last few denoising steps.

\paragraph{Time Predictor}
The time predictor is a conditional generative model that takes a start state and a goal state as conditioning inputs, and generates the trajectory length between them as its sample. It is implemented as a diffusion model with a lightweight MLP backbone. Training uses randomly cropped trajectory segments with length \(h\le H\), where the maximum length \(H\) is chosen to match the training horizon of the trajectory generator in Table~\ref{tab:impl_generator}; each training sample consists of the start state, the end state, and the corresponding normalized segment length. The number of denoising steps is 128 in Maze2D, AntMaze, and Cube, and 64 in the custom environment. At test time, unguided sampling returns the nominal prediction \(l_{\mathrm{norm}}\), which is also the default mode used in our experiments. To obtain shorter and longer duration hypotheses, denoted by \(l_{\min}\) and \(l_{\max}\), respectively, we further apply the derivative-guidance strategy of~\cite{li2024derivative} during sampling.

\paragraph{Inverse Dynamics Model}
AntMaze is the only environment that uses a learned inverse-dynamics module for action recovery. This module is implemented as an MLP that takes the 29-dimensional current observation together with a future task-space reference state in \((x,y)\) as input, and predicts the corresponding 8-dimensional action. It is trained in a supervised manner on the offline AntMaze dataset using state-action trajectories collected under the same environment configuration as the planning data. At test time, the model is queried inside the time-synchronous execution loop to convert the planned task-space reference trajectory into executable low-level actions.

\subsection{Dynamic Consistency Metric}
The Dynamic Consistency Metric is instantiated in Maze2D, AntMaze, and the custom environment. It is not used in the reported Cube experiments because that environment only evaluates the basic planner and does not include Anytime Refinement Search, which is the main module that consumes this scorer. In all three scored environments, the metric is evaluated in a low-dimensional task space rather than the full raw state: \(\zeta(\x)\) extracts planar position, and the state-support and transition-support terms are computed from offline state samples and local displacement features using the same kNN-based support proxy. To make the state-support term more sensitive to unsupported traversals between endpoints, the implementation also checks a small number of interpolated interior points along each local segment. The final step cost then combines these data-support terms with a lightweight geometric regularizer. In Maze2D and the custom environment, this regularizer only penalizes overly long local displacements, which is sufficient for these planar tracking tasks. In AntMaze, we additionally include turning and smoothness penalties, since abrupt heading changes and nonsmooth local motion are more indicative of poor executability in the locomotion setting. Trajectory-level scores are finally obtained using the same CVaR-style tail aggregation described in Section~\ref{sec:dyn_metric}. Across the scored environments, we use the same default support-estimation and aggregation design, with only a minor adjustment of the step-length threshold in the custom environment.

\subsection{Other Implementation Details}
\paragraph{Custom Environment}
The custom environment is a 2D double integrator with state \([x,y,v_x,v_y]\), workspace \([0,10]\times[0,10]\), and a single circular obstacle centered at \((4,6)\) with radius \(1.5\). Controls are bounded in \([-0.5,0.5]^2\). The offline dataset is generated by sampling single reach-avoid tasks in this environment and solving them with \texttt{stlpy}'s \texttt{DrakeSmoothSolver}. The optimization-based reference in the controlled comparison uses the same backend on the corresponding coarse macro system.

\paragraph{Baselines and Metrics}
The robustness-guided diffusion baseline (RGD) in Maze2D reuses the same pretrained diffusion model as DAG-STL, while adding robustness-based guidance during sampling. Since this baseline must generate the complete trajectory in a single shot, we set its planning horizon directly to the horizon of the target STL task. We also use exactly the same trajectory resolution factor, rollout procedure, and robustness-evaluation protocol as those used for DAG-STL.

\section{Proofs Omitted}\label{apx:proof}
\subsection{Proof of Lemma~\ref{lemma:task_dec}}
\begin{proof}
We prove the statement by structural induction on $\varphi$. Throughout, we work in discrete time $t\in\mathbb{Z}_{\ge 0}$ and use the shorthand
$\s_0 \vDash \RC(a,b,\mu)
\ \Leftrightarrow\
\exists t\in[a,b]\ \text{such that}\ \x_t \vDash \mu,$
and
$\s_0 \vDash \IC(a,b,\mu)
\ \Leftrightarrow\
\forall t\in[a,b],\ \x_t \vDash \mu,$
as defined in~\eqref{formula:reachability}--\eqref{formula:invariance}.

\smallskip
\noindent\textbf{Base cases.}

\emph{(i) $\varphi=\mu$.}
The decomposition yields
$\PP_\varphi=\{\RC(0,0,\mu)\}$,
$\TT_\varphi=\varnothing$,
and $\Lambda_\varphi=\varnothing$.
By the semantics of atomic predicates,
$\s_0 \vDash \mu
\ \Leftrightarrow\
\x_0 \vDash \mu
\ \Leftrightarrow\
\s_0 \vDash \RC(0,0,\mu),
$
which is exactly the desired form.

\emph{(ii) $\varphi=\F_{[a,b]}\mu$.}
The decomposition yields $\Lambda_\varphi=\{\lambda\}$ and
$\PP_\varphi=\{\RC(\lambda,\lambda,\mu)\},
\TT_\varphi=\{\lambda\in[a,b]\}$.
By the Boolean semantics of $\F$ in~\eqref{formula:sem_F},
\[
\begin{aligned}
\s_0 \vDash \F_{[a,b]}\mu
\ &\Leftrightarrow\
\exists \lambda\in[a,b]\ \text{such that}\ \x_\lambda \vDash \mu\\
&\Leftrightarrow\
\exists \lambda\in[a,b]\ \text{such that}\ \s_0 \vDash \RC(\lambda,\lambda,\mu),
\end{aligned}
\]
which is exactly the target form with $\bs{\lambda}=[\lambda]\in\mathcal{F}_\varphi$.

\emph{(iii) $\varphi=\G_{[a,b]}\mu$.}
The decomposition yields $\Lambda_\varphi=\varnothing$ and
$\PP_\varphi=\{\IC(a,b,\mu)\}$,
$\TT_\varphi=\varnothing$.
By~\eqref{formula:sem_G},
\[
\s_0 \vDash \G_{[a,b]}\mu
\ \Leftrightarrow\
\forall t\in[a,b],\ \x_t \vDash \mu
\ \Leftrightarrow\
\s_0 \vDash \IC(a,b,\mu),
\]
which matches the desired equivalence (no time variables).

\emph{(iv) $\varphi=\mu_1 \U_{[a,b]} \mu_2$.}
The decomposition yields
$\Lambda_\varphi=\{\lambda\},\PP_\varphi=\{\IC(0,\lambda,\mu_1),\RC(\lambda,\lambda,\mu_2)\},
\TT_\varphi=\{\lambda\in[a,b]\}.$
By~\eqref{formula:sem_U},
\[
\begin{aligned}
\s_0 \vDash \mu_1 \U_{[a,b]} \mu_2
&\Leftrightarrow
\exists \lambda\in[a,b]\ \text{such that}\ \x_\lambda \vDash \mu_2
\\ & \qquad \text{and}\ 
\forall t\in[0,\lambda],\ \x_t \vDash \mu_1 \\
&\Leftrightarrow
\exists \lambda\in[a,b]\ \text{such that}\ 
\s_0 \vDash \RC(\lambda,\lambda,\mu_2)
\\ &\qquad \text{and}\ 
\s_0 \vDash \IC(0,\lambda,\mu_1),
\end{aligned}
\]
again yielding the target form with $\bs{\lambda}=[\lambda]\in\mathcal{F}_\varphi$.

\smallskip
\noindent\textbf{Inductive steps.}

For each constructor below, assume that the lemma holds for the immediate subformula(e).

\emph{(1) Conjunction: $\varphi=\varphi_1\wedge\varphi_2$.}
By STL semantics,
\[
\s_0 \vDash \varphi
\ \Leftrightarrow\
(\s_0 \vDash \varphi_1)\ \text{and}\ (\s_0 \vDash \varphi_2).
\]
The decomposition uses disjoint unions
\[
\PP_\varphi=\PP_{\varphi_1}\uplus\PP_{\varphi_2},
\TT_\varphi=\TT_{\varphi_1}\uplus\TT_{\varphi_2},
\Lambda_\varphi=\Lambda_{\varphi_1}\uplus\Lambda_{\varphi_2},
\]
with implicit $\alpha$-renaming to avoid variable clashes. Since the two variable sets are disjoint and the constraints are merged componentwise, a time assignment is feasible for $\varphi$ iff its two components are feasible for $\varphi_1$ and $\varphi_2$, respectively. Hence
\[
\mathcal{F}_\varphi=\mathcal{F}_{\varphi_1}\times \mathcal{F}_{\varphi_2}.
\]
Applying the induction hypothesis to $\varphi_1$ and $\varphi_2$, and observing that each progress condition in $\PP_{\varphi_i}$ depends only on the corresponding component of
$\bs{\lambda}=[\bs{\lambda}_1,\bs{\lambda}_2]$, we obtain
\[
\s_0\vDash\varphi
\ \Leftrightarrow\
\exists \bs{\lambda}\in\mathcal{F}_\varphi
\ \text{such that}\
\forall \PC\in\PP_\varphi,\ 
\s_0 \vDash \PC(\bs{\lambda}).
\]

\emph{(2) Outer eventually: $\varphi'=\F_{[a,b]}\varphi$.}
By~\eqref{formula:sem_F},
\[
\s_0\vDash\varphi'
\ \Leftrightarrow\
\exists \lambda\in[a,b]\ \text{such that}\ \s_\lambda \vDash \varphi.
\]
We first note the following time-shift identity: for any progress condition $\PC(c,d,\mu)$ and any $k\in\mathbb{Z}_{\ge 0}$,
\begin{equation}\label{eq:shift_appendix}
\s_k \vDash \PC(c,d,\mu)
\ \Leftrightarrow\
\s_0 \vDash \PC(c+k,d+k,\mu).
\end{equation}
This follows directly from the definitions of \(\RC\) and \(\IC\).

The decomposition of $\F_{[a,b]}\varphi$ introduces a fresh variable $\lambda\in[a,b]$ and shifts every progress condition of $\varphi$ by $+\lambda$. Therefore,
\[
\Lambda_{\varphi'}=\Lambda_\varphi\uplus\{\lambda\},
\TT_{\varphi'}=\TT_\varphi\uplus\{\lambda\in[a,b]\},
\]
and
\[
\mathcal{F}_{\varphi'}
=
\left\{
[\bs{\lambda},\lambda]
\;\middle|\;
\bs{\lambda}\in\mathcal{F}_\varphi,\ \lambda\in[a,b]
\right\}.
\]
Applying the induction hypothesis to $\varphi$ over the shifted signal $\s_\lambda$, and then using~\eqref{eq:shift_appendix}, yields
\[
\s_0\vDash \varphi'
\ \Leftrightarrow\
\exists \bs{\lambda}'\in\mathcal{F}_{\varphi'}
\ \text{such that}\
\forall \PC'\in\PP_{\varphi'},\ 
\s_0 \vDash \PC'(\bs{\lambda}').
\]
The special case $\F_{[a,a]}\varphi$ is identical, except that the shift is by the constant \(a\) and no new time variable is introduced.

\emph{(3) Outer always: $\varphi'=\G_{[a,b]}\varphi$.}
In discrete time,
\begin{equation}\label{eq:G_as_bigwedge_appendix}
\G_{[a,b]}\varphi
\equiv
\bigwedge_{k=a}^{b}\F_{[k,k]}\varphi,
\end{equation}
since $\s_0\vDash\G_{[a,b]}\varphi$ iff $\s_k\vDash\varphi$ for every integer $k\in[a,b]$.
The decomposition therefore creates an independent copy
\[
(\PP_\varphi^{(k)},\TT_\varphi^{(k)},\Lambda_\varphi^{(k)})
\]
for each \(k\in[a,b]\), shifts every copied progress condition by \(+k\), and then merges all copies. Accordingly, the feasible assignment set is the Cartesian product of the feasible sets of all copies:
\[
\mathcal{F}_{\varphi'}
=
\prod_{k=a}^{b}\mathcal{F}_{\varphi}^{(k)}.
\]
Applying the previous case to each subformula \(\F_{[k,k]}\varphi\) yields the desired equivalence.

It remains to justify the merging of constant invariance progress conditions. If
\(\IC(c,d,\mu)\in\PP_\varphi\) has constant endpoints, then its shifted copies
$\IC(c+a,d+a,\mu),\ \IC(c+a+1,d+a+1,\mu),\ \ldots,\ \IC(c+b,d+b,\mu)$
are jointly satisfied iff \(\mu\) holds over the union of the corresponding intervals. Since these intervals overlap consecutively in discrete time, their union is exactly \([c+a,d+b]\). Hence
\[
\bigwedge_{k=a}^{b}\IC(c+k,d+k,\mu)
\Leftrightarrow
\IC(c+a,d+b,\mu),
\]
so the merge preserves semantics.

\emph{(4) Outer until: $\varphi'=\phi \U_{[a,b]} \varphi$.}
By the restriction introduced in Section~\ref{sec:STL}, the prefix formula \(\phi\) contains no \(\F\) or \(\U\) operators. Therefore, its decomposition introduces no time variables and yields only invariance progress conditions with constant endpoints; in particular, $\Lambda_\phi=\varnothing,
\TT_\phi=\varnothing,$
and every element of \(\PP_\phi\) is of the form \(\IC(c,d,\mu)\).

The decomposition of \(\phi\U_{[a,b]}\varphi\) introduces a new variable \(\lambda\in[a,b]\), shifts every progress condition of \(\varphi\) by \(+\lambda\), and extends every invariance progress condition of \(\phi\) to end at the chosen until time. Thus
\[
\mathcal{F}_{\varphi'}
=
\left\{
[\bs{\lambda},\lambda]
\;\middle|\;
\bs{\lambda}\in\mathcal{F}_{\varphi},\ \lambda\in[a,b]
\right\}.
\]
By~\eqref{formula:sem_U},
\begin{equation}\label{eq:U_semantics_appendix}
\begin{aligned}
\s_0\vDash \phi \U_{[a,b]} \varphi
\ \Leftrightarrow\ &
\exists \lambda\in[a,b]\ \text{such that}\\
&\Big(\s_\lambda\vDash \varphi\Big)
\ \text{and}\
\Big(\forall \tau\in[0,\lambda],\ \s_\tau\vDash \phi\Big).
\end{aligned}
\end{equation}
The first conjunct is handled exactly as in the outer-eventually case using the shift identity~\eqref{eq:shift_appendix}. For the second conjunct, let \(\IC(c,d,\mu)\in\PP_\phi\). Then
\[
\forall \tau\in[0,\lambda],\ \s_\tau\vDash \IC(c,d,\mu)
\]
holds iff \(\mu\) is satisfied over every interval \([c+\tau,d+\tau]\) for \(\tau=0,\dots,\lambda\), which is equivalent to requiring \(\mu\) over the union of these intervals. Since
\[
\bigcup_{\tau=0}^{\lambda}[c+\tau,d+\tau]=[c,d+\lambda],
\]
this is equivalent to $\s_0 \vDash \IC(c,d+\lambda,\mu)$.
Applying this argument to every invariance progress condition in \(\PP_\phi\), and combining it with the shifted decomposition of \(\varphi\), proves that
\[
\s_0\vDash \varphi'
\ \Leftrightarrow\
\exists \bs{\lambda}'\in\mathcal{F}_{\varphi'}
\ \text{such that}\
\forall \PC'\in\PP_{\varphi'},\ 
\s_0 \vDash \PC'(\bs{\lambda}').
\]

\smallskip
\noindent\textbf{Conclusion.}
All base cases and constructors preserve the claimed equivalence. Therefore, by structural induction, for every disjunction-free STL formula \(\varphi\) in PNF satisfying the restriction of Section~\ref{sec:STL},
\[
\s_0 \vDash \varphi
\ \Longleftrightarrow\
\exists\,\bs{\lambda}\in\mathcal{F}_\varphi
\ \text{such that}\
\forall \PC\in\PP_\varphi,\ 
\s_0 \vDash \PC(\bs{\lambda}).
\]
If \(\Lambda_\varphi=\varnothing\), then \(\mathcal{F}_\varphi=\{\emptyset\}\), and the statement reduces to
\[
\s_0 \vDash \varphi
\ \Longleftrightarrow\
\forall \PC\in\PP_\varphi,\ 
\s_0 \vDash \PC.
\qedhere
\]
\end{proof}

\subsection{Proof of Lemma~\ref{lemma:progress_allo}}
\begin{proof}
We prove the lemma by maintaining inductive invariants along the depth-first search.
At each search node, the algorithm maintains a tuple
$(\x,t,\PP^{\RC}_{\mathrm{rem}},\TT,\tilde{\s})$,
where \(\tilde{\s}\) is the timed waypoint sequence constructed so far, and \(\PP^{\RC}_{\mathrm{rem}}\) is the set of remaining reachability progress conditions.

\paragraph{Inductive invariants}
After initialization, and after every successful extension returned by \texttt{TimeAssign} followed by \texttt{UpdateConstraint}, the following statements hold:
\begin{itemize}
\item[(I1)] \emph{Constraint feasibility:}
the current time-constraint set \(\TT\) admits at least one feasible time-variable assignment.

\item[(I2)] \emph{Committed reachability certification:}
for every reachability progress condition
\(\RC(a_\Lambda,b_\Lambda,\mu)\)
that has already been assigned a waypoint \((\x',t')\) and removed from \(\PP^{\RC}_{\mathrm{rem}}\),
the current constraint set \(\TT\) contains the inequalities
\[
a_\Lambda \le t',
\qquad
b_\Lambda \ge t'.
\]
Hence, for every feasible assignment \(\bs\lambda\) of \(\TT\), one has
\[
t' \in [a_\Lambda(\bs\lambda), b_\Lambda(\bs\lambda)],
\]
and therefore the assigned waypoint time certifies this reachability progress condition.

\item[(I3)] \emph{Waypoint-level invariance consistency:}
consider any invariance progress condition
\(\IC(c,d_\Lambda,\mu')\)
whose start time has already been determined through its triggering reachability progress condition.
If an accepted waypoint \((\x',t')\) satisfies \(\x' \nvDash \mu'\), then the current constraint set \(\TT\) contains the truncation constraint
\[
d_\Lambda < t'.
\]
Consequently, no feasible assignment of \(\TT\) can make this invariance progress condition active at time \(t'\).
\end{itemize}

\paragraph{Initialization}
At the root node, \(\tilde{\s}=[(\x_0,0)]\), and the initial time-constraint set is the one produced by the decomposition and preprocessing steps. By assumption, this initial constraint set is feasible, so (I1) holds. Since no reachability progress condition has yet been assigned, (I2) and (I3) hold vacuously.

\paragraph{Preservation under one successful extension}
Suppose the algorithm selects a remaining reachability progress condition
\[
\RC(a_\Lambda,b_\Lambda,\mu)\in \PP^{\RC}_{\mathrm{rem}}
\]
at the current node \((\x,t,\PP^{\RC}_{\mathrm{rem}},\TT,\tilde{\s})\), and \texttt{TimeAssign} returns a new waypoint \((\x',t^*)\) with \(\x'\vDash\mu\).

By construction, \texttt{TimeAssign} returns a \emph{locally admissible} candidate pair \((\x',t^*)\) based on the current timing window, the learned transition-time prediction, and the currently determined invariance conflicts. It does not by itself certify global consistency of the updated constraint system. Instead, \texttt{UpdateConstraint} augments \(\TT\) by adding
\[
a_\Lambda \le t^*,
\qquad
b_\Lambda \ge t^*,
\]
together with every required truncation constraint
\[
d_\Lambda < t^*
\]
for determined invariance progress conditions violated by \(\x'\).
The search branch is extended only if the resulting augmented constraint set remains feasible; otherwise, the branch is discarded and the algorithm backtracks. Therefore, (I1) is preserved. Moreover, the newly selected reachability progress condition is certified at time \(t^*\) by the inserted inequalities, so (I2) is preserved. Finally, every waypoint-level violation of an already activated invariance progress condition is explicitly ruled out by the corresponding truncation constraint \(d_\Lambda<t^*\), so (I3) is preserved as well.

\paragraph{Termination and witness construction}
Assume the algorithm terminates successfully. Then
\(\PP^{\RC}_{\mathrm{rem}}=\varnothing\),
and the final time-constraint set is \(\TT_f\) with feasible assignment set
\[
\mathcal F
=
\left\{
\bs{\lambda}\in\mathbb{Z}_{+}^{|\Lambda|}
\;\middle|\;
\bs{\lambda}\ \text{satisfies all constraints in}\ \TT_f
\right\}.
\]
By (I1), \(\mathcal F\neq\varnothing\). Fix any \(\bs{\lambda}\in\mathcal F\).

We first verify all reachability progress conditions.
Since every reachability progress condition in \(\PP^\RC\) has been assigned before termination, (I2) implies that for each
\(\RC(a_\Lambda,b_\Lambda,\mu)\in\PP^\RC\),
the corresponding waypoint time \(t_i\) satisfies
\[
t_i\in[a_\Lambda(\bs\lambda),b_\Lambda(\bs\lambda)].
\]
Because the assigned waypoint state satisfies \(\tilde{\x}_i \vDash \mu\), and \(\x_{t_i}=\tilde{\x}_i\) by assumption on the signal \(\s_0\), we obtain
\[
\s_0 \vDash \RC(a_\Lambda(\bs{\lambda}),b_\Lambda(\bs{\lambda}),\mu).
\]

We next verify waypoint-level consistency for all invariance progress conditions.
Let
\(\IC(a_\Lambda,b_\Lambda,\mu')\in\PP^\IC\),
and let \(t_i\in [a_\Lambda(\bs\lambda),b_\Lambda(\bs\lambda)]\cap\{t_0,\dots,t_n\}\).
We must show that \(\x_{t_i}\vDash\mu'\).

If \(\x_{t_i}\vDash\mu'\), then there is nothing to prove.
Otherwise, suppose \(\x_{t_i}\nvDash\mu'\).
Then this invariance progress condition must already have had its start time determined by the time \(t_i\) was accepted as a waypoint; indeed, after preprocessing every invariance progress condition is associated with a unique triggering reachability progress condition, and once the invariance becomes active its start time is fixed.
Therefore, by (I3), the final constraint set must contain the truncation constraint
\[
b_\Lambda < t_i
\]
for this invariance progress condition, which contradicts the assumption that
\[
t_i\in[a_\Lambda(\bs\lambda),b_\Lambda(\bs\lambda)].
\]
Hence \(\x_{t_i}\vDash\mu'\) must hold.

Therefore, under the chosen feasible assignment \(\bs\lambda\in\mathcal F\), every reachability progress condition is witnessed at the returned waypoint times, and no invariance progress condition is violated at any selected waypoint during its active interval. This proves the lemma.
\end{proof}

\subsection{Proof of Theorem~\ref{theorem:diffusionSTL}}
\begin{proof}
Let
\(
(\PP_\varphi,\TT_\varphi,\Lambda_\varphi)
\)
be the decomposition of the STL specification \(\varphi\), and let
\(
\tilde{\s}=(\tilde{\x}_0,t_0)(\tilde{\x}_1,t_1)\dots(\tilde{\x}_n,t_n),
0=t_0\le \dots \le t_n,
\)
be the timed waypoint sequence returned by the Progress Allocation algorithm.
By Lemma~\ref{lemma:progress_allo}, there exists a feasible time-variable assignment
\(
\bs\lambda^\star \in \mathcal{F}
\)
such that all reachability progress conditions are witnessed at the allocated waypoints, and no invariance progress condition is violated at any selected waypoint during its active interval.
Let
\(
\s_0=\x_0\x_1\dots\x_T
\)
with $T\ge t_n$
be the state trajectory returned by the Trajectory Generation module, which by assumption satisfies:
\begin{enumerate}[label=(\roman*),leftmargin=*]
\item \(\x_{t_i}=\tilde{\x}_i\) for all \(i\in\{0,\dots,n\}\);
\item for every invariance progress condition
\(\IC(a_\Lambda,b_\Lambda,\mu)\in\PP^\IC\),
the signal \(\s_0\) satisfies
\(
\IC(a_\Lambda(\bs\lambda^\star),b_\Lambda(\bs\lambda^\star),\mu)
\)
over its full active interval.
\end{enumerate}

\paragraph{Reachability progress conditions}
Consider any
\(
\RC(a_\Lambda,b_\Lambda,\mu)\in\PP^\RC.
\)
By Lemma~\ref{lemma:progress_allo}, under the chosen feasible assignment \(\bs\lambda^\star\), this reachability progress condition is witnessed at some allocated waypoint time \(t_i\), i.e.,
\(
t_i\in[a_\Lambda(\bs\lambda^\star),\,b_\Lambda(\bs\lambda^\star)]
\ \text{and}\ 
\tilde{\x}_i \vDash \mu.
\)
By item (i), the generated state trajectory visits all allocated waypoints, so \(\x_{t_i}=\tilde{\x}_i\). Hence
\(
\s_0 \vDash \RC(a_\Lambda(\bs\lambda^\star),b_\Lambda(\bs\lambda^\star),\mu).
\)

\paragraph{Invariance progress conditions}
Consider any
\(
\IC(a_\Lambda,b_\Lambda,\mu)\in\PP^\IC.
\)
By item (ii), the generated state trajectory satisfies
\(
\forall t\in[a_\Lambda(\bs\lambda^\star),\,b_\Lambda(\bs\lambda^\star)],\quad \x_t \vDash \mu,
\)
which is exactly
\(
\s_0 \vDash \IC(a_\Lambda(\bs\lambda^\star),b_\Lambda(\bs\lambda^\star),\mu).
\)

\paragraph{Conclusion via decomposition soundness}
We have shown that, for the same feasible assignment \(\bs\lambda^\star\),
\(
\forall \PC\in\PP_\varphi,
\s_0 \vDash \PC(\bs\lambda^\star).
\)
Therefore, by Lemma~\ref{lemma:task_dec}, it follows that
\(
\s_0 \vDash \varphi.
\)
This proves the theorem.
\end{proof}

\section{Additional Experimental Tables}\label{apx:additional_tables}

\subsection{Detailed Timing Results for Maze2D Variants}\label{apx:additional_timing}

Table~\ref{tab:mode_time_comparison_full} reports the per-template total decision time of the five DAG-STL variants in the three Maze2D layouts. These are the detailed results corresponding to the averaged summary shown in Table~\ref{tab:mode_time_comparison} in the main text.

\begin{table*}[t]
\centering
\caption{Detailed total decision time (s) of different DAG-STL variants across environments and task templates. U: Umaze; M: Medium; L: Large. Dark blue cells denote the best value, while light blue cells denote the second-best value.}
\label{tab:mode_time_comparison_full}
\begin{tabular}{l l r r r r r }
\toprule
Env & Type & B (s) & B+OR (s) & ARS-FS (s) & ARS (s) & ARS+OR (s) \\
\midrule
U & 1 & \bestcell{0.47±0.02} & \secondbestcell{0.48±0.02} & 0.97±0.09 & 3.66±0.17 & 3.65±0.17 \\
U & 2 & \bestcell{0.96±0.02} & \secondbestcell{0.97±0.03} & 1.90±0.13 & 7.04±1.95 & 6.73±1.85 \\
U & 3 & \bestcell{1.06±0.02} & \secondbestcell{1.13±0.19} & 8.14±0.20 & 11.89±1.91 & 13.22±2.35 \\
U & 4 & \bestcell{1.61±0.06} & \secondbestcell{1.73±0.24} & 3.75±0.23 & 12.50±6.24 & 12.86±6.20 \\
U & 5 & \bestcell{1.45±0.03} & \secondbestcell{1.51±0.19} & 2.95±0.26 & 10.53±4.55 & 10.91±4.65 \\
U & 6 & \bestcell{1.46±0.04} & \secondbestcell{1.52±0.21} & 3.12±0.85 & 10.17±4.39 & 10.36±4.50 \\
U & 7 & \bestcell{0.96±0.02} & \secondbestcell{0.97±0.07} & 1.87±0.13 & 6.49±1.63 & 6.25±1.54 \\
U & 8 & \secondbestcell{0.96±0.02} & \bestcell{0.96±0.07} & 1.90±0.14 & 8.63±4.00 & 8.07±3.76 \\
U & 9 & \bestcell{1.93±0.06} & \secondbestcell{1.99±0.20} & 3.92±0.28 & 14.54±6.28 & 14.23±6.17 \\
\midrule
M & 1 & \bestcell{0.79±0.02} & \secondbestcell{0.84±0.18} & 1.18±0.09 & 5.31±0.16 & 5.35±0.19 \\
M & 2 & \bestcell{1.81±0.04} & \secondbestcell{2.00±0.43} & 2.50±0.15 & 9.69±2.23 & 10.14±2.44 \\
M & 3 & \bestcell{1.90±0.03} & \secondbestcell{2.18±0.58} & 8.61±0.21 & 15.63±2.19 & 16.66±2.61 \\
M & 4 & \bestcell{3.44±0.05} & \secondbestcell{3.90±0.90} & 5.10±0.22 & 16.66±6.31 & 18.31±6.52 \\
M & 5 & \bestcell{2.66±0.16} & \secondbestcell{3.13±0.84} & 3.80±0.18 & 13.82±4.59 & 15.23±4.91 \\
M & 6 & \bestcell{2.66±0.08} & \secondbestcell{3.18±0.79} & 5.09±2.72 & 14.65±4.73 & 16.66±5.75 \\
M & 7 & \bestcell{1.79±0.05} & \secondbestcell{1.96±0.40} & 2.61±0.25 & 10.10±2.55 & 10.60±2.84 \\
M & 8 & \bestcell{1.80±0.03} & \secondbestcell{2.01±0.45} & 2.68±0.16 & 9.91±3.07 & 11.47±3.73 \\
M & 9 & \bestcell{3.64±0.09} & \secondbestcell{3.90±0.66} & 5.62±0.57 & 17.54±6.80 & 20.80±10.19 \\
\midrule
L & 1 & \bestcell{1.24±0.03} & \secondbestcell{1.64±0.87} & 1.66±0.11 & 8.14±0.27 & 8.22±0.39 \\
L & 2 & \bestcell{2.52±0.05} & 3.53±1.31 & \secondbestcell{3.31±0.19} & 14.03±2.93 & 14.60±3.10 \\
L & 3 & \bestcell{2.60±0.04} & \secondbestcell{3.82±1.39} & 9.37±0.23 & 19.08±2.24 & 20.93±3.00 \\
L & 4 & \bestcell{4.65±0.14} & 6.75±2.34 & \secondbestcell{6.57±0.32} & 23.09±7.25 & 33.37±10.48 \\
L & 5 & \bestcell{3.85±0.29} & 5.69±2.13 & \secondbestcell{4.82±0.26} & 17.45±4.81 & 21.54±8.95 \\
L & 6 & \bestcell{3.80±0.09} & \secondbestcell{5.45±1.91} & 5.81±2.10 & 17.71±4.40 & 19.63±5.11 \\
L & 7 & \bestcell{2.50±0.04} & \secondbestcell{3.27±1.11} & 3.60±1.05 & 13.64±2.96 & 14.65±3.40 \\
L & 8 & \bestcell{2.51±0.04} & 3.31±1.09 & \secondbestcell{3.25±0.14} & 14.16±3.43 & 14.73±3.53 \\
L & 9 & \bestcell{5.07±0.11} & \secondbestcell{6.53±1.87} & 7.14±0.90 & 24.14±6.74 & 31.86±18.40 \\
\bottomrule
\end{tabular}
\end{table*}

\end{document}